\documentclass{article}

\usepackage[utf8]{inputenc} % allow utf-8 input
\usepackage[T1]{fontenc}    % use 8-bit T1 fonts
\usepackage{hyperref}       % hyperlinks
\usepackage{url}            % simple URL typesetting
\usepackage{booktabs}       % professional-quality tables
\usepackage{amsfonts}       % blackboard math symbols
\usepackage{nicefrac}       % compact symbols for 1/2, etc.
\usepackage{microtype}      % microtypography
\usepackage{xcolor}         % colors
\usepackage{wrapfig}
\usepackage{multirow}
\usepackage{minitoc}
\usepackage[toc,page,header]{appendix}

\usepackage{amsthm}
\newtheorem{theorem}{Theorem}
\usepackage{CJKutf8}

\usepackage{inconsolata}

\usepackage{amsmath}
\usepackage{siunitx} % for better number formatting
\usepackage{booktabs}
\usepackage{graphicx}
\usepackage{subcaption}
\usepackage{xcolor}
\definecolor{customblue}{RGB}{31, 119, 180}
\definecolor{customgreen}{RGB}{44, 160, 44}
\definecolor{customyellow}{RGB}{255, 193, 7}
\definecolor{custompurple}{RGB}{148, 103, 189}
\definecolor{customred}{RGB}{255,99,71}
\usepackage{fp} % for arithmetic calculations
\usepackage{listings} % Add this to your preamble

\newcommand{\meanpmstd}[2]{%
\FPeval{\result}{pow(0.5, #2)}%
\FPround{\result}{\result}{1}%
#1  \scalebox{0.9}{$\pm$\result}
}

\usepackage[final]{neurips_2025}

\usepackage[utf8]{inputenc} % allow utf-8 input
\usepackage[T1]{fontenc}    % use 8-bit T1 fonts
\usepackage{hyperref}       % hyperlinks
\usepackage{url}            % simple URL typesetting
\usepackage{booktabs}       % professional-quality tables
\usepackage{amsfonts}       % blackboard math symbols
\usepackage{nicefrac}       % compact symbols for 1/2, etc.
\usepackage{microtype}      % microtypography
\usepackage{xcolor}         % colors

\title{On Extending Direct Preference Optimization to Accommodate Ties}

\author{%
  Jinghong Chen, Guangyu Yang,  Weizhe Lin, Jingbiao Mei, Chenxu Lyu, Bill Byrne\\
  Department of Engineering\\
  University of Cambridge\\
  Cambridge, United Kingdom CB2 1PZ \\
  \texttt{\{jc2124, gy266, wl356, jm2245, cl927, wjb31\}@cam.ac.uk} \\
}

\begin{document}

\maketitle

\begin{abstract}

We derive and investigate two DPO variants that explicitly model the possibility of declaring a tie in pair-wise comparisons. We replace the Bradley-Terry model in DPO with two well-known modeling extensions, by Rao and Kupper and by Davidson, that assign probability to ties as alternatives to clear preferences. Our experiments in neural machine translation and summarization show that explicitly labeled ties can be added to the datasets for these DPO variants without  the degradation in task performance that is observed when the same tied pairs are presented to DPO. We find empirically that the inclusion of ties leads to stronger regularization with respect to the reference policy as measured by KL divergence, and we see this even for DPO in its original form. We provide a theoretical explanation for this regularization effect using ideal DPO policy theory. We further show performance improvements over DPO in translation and mathematical reasoning using our DPO variants. We find it can be beneficial to include ties in preference optimization rather than simply discard them, as is done in common practice.

\end{abstract}

\section{Introduction}

The original formulation of DPO~\citep{DPO-Rafael} does not allow for ties.  DPO requires training data consisting of paired options, $y_w \succ y_l$, and each of these pairs should represent a clear preference in judgment with no ambiguity as to which is the winner and which is the loser.  From this data, the DPO learning procedure  encourages the underlying policy to prefer $y_w$ over $y_l$.  This formulation does not allow for any ambiguity or uncertainty in the comparison of the paired examples in the training data. 

    This certainty is not easy to achieve in practice. A common approach is simply to discard data.  
    \citet[Sec. 4.2.1]{dubey2024llama3herdmodels} apply DPO in post-training of Llama 3 models and note that for  
    ``DPO, we use samples that are labeled as the chosen response being significantly better or better than the rejected counterpart for training and discard samples with similar responses.''  Similarly, Qwen2 developers~\citep[Sec. 4.3]{yang2024qwen2technicalreport} ``sample multiple responses from the current policy model, and the reward model selects the most and the least preferred responses, forming preference pairs that are used for DPO.'' 
    Over-generation followed by aggressive selection is  effective in producing the strongly ordered judgments needed for DPO.  However the process appears wasteful: many potentially useful, and expensively collected,  preference judgments are discarded simply because they are ties. As \citet{rao1967ties} note: ``any model which does not allow for the possibility of ties is not making full use of the information contained in the no-preference class.'' 

Motivated by this, we investigate DPO variants that can incorporate ties.  
We replace the Bradley-Terry preference model at the heart of DPO by two well-known extensions  by \citet{rao1967ties} and by \citet{Davidson-On-Extending-Bradley-Terry} that explicitly assign probability to tied judgments alongside winners and losers. Since  these models are generalizations of the Bradley-Terry model, we find that they are readily incorporated into the DPO modeling framework. 
% We note that recent work ~\citep{RankingChatbots} makes use of these extensions in evaluating chatbots.
In experiments in neural machine translation and summarization, we find that ties can be added to the datasets for these DPO variants without  the degradation in task performance that results from adding ties to the original DPO.  We also observe improved regularization, in reduced KL-divergence to the reference policy, by adding ties. We provide a theoretical explanation based on ideal DPO policy theory by \citet{PL-paper-chen}.
We further show DPO-RK and DPO-D improves performance over strong DPO baselines in translation and mathematical reasoning by including previously discarded data as tied pairs. These findings suggest it can be beneficial to incorporate ties in preference optimization rather than simply discard them, as is done in common practice.

\section{Methodology}

% In this section, we first summarize the derivation of Direct Preference Optimization (DPO) by \citet{DPO-Rafael} which is based on the Bradley-Terry model. Then, we review the Rao-Kupper~\citep{rao1967ties} and the Davidson~\citep{Davidson-On-Extending-Bradley-Terry} comparison models which account for ties as alternatives to BT. We proceed to derive their respective preference learning objectives, DPO-RK and DPO-D, by following the same derivation recipe as DPO.

\subsection{DPO and the Bradley-Terry Preference Distribution}

The Bradley-Terry model assigns probability that an item $y_i$ will be preferred to item $y_j$ 
in terms of their `strength' parameters $\lambda$.   
In the RLHF setting, strengths are expressed as  rewards $r$, $\lambda=e^r$  [Eq. 1]~\cite{DPO-Rafael},
so that the preference distribution for item $i$ over item $j$  depends on the  difference in their rewards, $d_{ij} = r_i - r_j$  
\begin{align}
    p^{BT}(y_i\succ y_j) &= \frac{\lambda_i}{\lambda_i + \lambda_j} = \frac{e^{r_i}}{e^{r_i} + e^{r_j}} = \sigma(r_i - r_j) = \sigma(d_{ij})
\label{eq:bt}
\end{align}

One of the enabling observations made by \citet{DPO-Rafael} is that when a policy $\pi_\theta$ is sought to maximize the KL-regularized objective 
%\begin{equation}
$    \max_{\pi_\theta}  \mathbb{E} ~ [ r(x,y) ] - 
    \beta ~ D(\pi_\theta(y|x) ~ || ~ \pi_{ref}(y|x))
    \label{eq:dpo_obj}
    $, 
%\end{equation}
the reward associated with the policy has the form  
%\begin{equation}
$
    r_\theta(x,y) = \beta \log \frac{\pi_\theta(y|x)}{\pi_{ref}(y|x)} + \beta\log Z_\theta(x)
    \label{eq:optrisk}
$.
%\end{equation}
This allows expressing the difference in rewards between hypotheses $y_w$ and $y_l$ under a parameterized policy $\pi_\theta$ as the {reward margin}
\begin{align}
                    d_\theta(x, y_w, y_l) &= r_\theta(x,y_w) - r_\theta(x,y_l)   =
                    \beta \log \frac{\pi_\theta(y_w|x)}{\pi_{ref}(y_w|x)} - \beta \log \frac{\pi_\theta(y_l|x)}{\pi_{ref}(y_l|x)}
                \label{eq:rewarddiff}
            \end{align}
so that  the corresponding Bradley-Terry probability that item $y_w$ beats item $y_l$ is 
\begin{align}
 p^{BT}_\theta&(y_w \succ_x y_l) = \sigma(d_\theta(x, y_w, y_l)) = 
 \sigma\big( \beta \log \frac{\pi_\theta(y_w|x)}{\pi_{ref}(y_w|x)} - \beta \log \frac{\pi_\theta(y_l|x)}{\pi_{ref}(y_l|x)}  \big)
\end{align}  
The DPO policy objective~[Eq. 7]\cite{DPO-Rafael} follows by incorporating the parameterized form of the preference distribution into a maximum likelihood objective
\begin{align}
    \mathcal L_{DPO}(\pi_\theta; \pi_{ref}) &=  - \mathbb{E}_{x,y_w,y_l} \log p_\theta(y_w \succ_x y_l)
    \label{eq:dponll} \\ &= - \mathbb{E}_{x,y_w,y_l} \log \sigma( \beta \log \frac{\pi_\theta(y_w|x)}{\pi_{ref}(y_w|x)} - \beta \log \frac{\pi_\theta(y_l|x)}{\pi_{ref}(y_l|x)})
    \label{eq:dpoloss}
\end{align}
We note that Eq.~\ref{eq:rewarddiff} follows from the regularized risk optimization~\citep[A.1]{DPO-Rafael}.  It does  not rely on any assumption that  limits its use to the Bradley-Terry model. 

\subsection{Bradley-Terry Extensions that Accommodate Ties}
\label{ssec:bte}
An observed weakness of the Bradley-Terry model is that it does not allow for ties.    Unless two items have exactly the same strengths (so that $d_{ij}=0$), the model always assigns a higher probability of winning to the stronger item.  This may be reasonable if one item is much stronger than the other, but when items are relatively comparable it may be desirable to allow some probability for tied outcomes.   

The Rao-Kupper~\citep{rao1967ties} model assigns win and tie probabilities as:
\begin{align}
    p^{RK}(y_i\succ y_j) & =  \frac{\lambda_i}{\lambda_i + \nu_{RK} \lambda_j} 
    && \mbox{$y_i$ beats $y_j$}     \label{eq:rkwin} \\
    p^{RK}(y_i\sim y_j) & =  \frac{(\nu_{RK}^2-1) \lambda_i \lambda_j}{(\lambda_i + \nu_{RK} \lambda_j)(\lambda_j + \nu_{RK} \lambda_i)}
    && \mbox{$y_i$ ties $y_j$}
    \label{eq:rktie}
\end{align}
% Suppose $\lambda_i = \lambda_j$:  
% \begin{align}
%     P(\mbox{no tie}) & =  p_{ij} + p_{ji} = \frac{2}{1+\nu} 
%     && \mbox{no tie: either $i$ or $j$ wins} \\
%     P(\mbox{tie}) & =  \frac{\nu^2-1}{(1 + \nu)^2} = \frac{\nu-1}{\nu+1}   && \mbox{a tie}
% \end{align}
while the 
Davidson~\citep{Davidson-On-Extending-Bradley-Terry} model assigns win and tie probabilities as:
\begin{align}
    p^{D}(y_i\succ y_j) & =  \frac{\lambda_i}{\lambda_i + \lambda_j + 2 \nu_D \sqrt{\lambda_i\lambda_j}}  
    && \mbox{ $y_i$ beats $y_j$} \label{eq:davidsonwin}\\
    p^{D}(y_i\sim y_j) & =  \frac{ 2 \nu_D \sqrt{\lambda_i\lambda_j}}{\lambda_i + \lambda_j + 2 \nu_D  \sqrt{\lambda_i\lambda_j}}  
    && \mbox{ $y_i$ and $y_j$ tie}
\label{eq:davidsontie}
\end{align}
% Note that the probabilities of the three options sum to one: $p_{ij}$ + $p_{ji}$ + $q_{ij}$ = 1.  
% The parameter $\nu$ can be interpreted as the odds of a tie between evenly matched players.

% To see this, set $\lambda_i = \lambda_j$:
% \begin{align}
%     P(\mbox{no tie}) & =  p_{ij} + p_{ji} = \frac{1}{1+\nu} 
%     && \mbox{no tie: either $i$ or $j$ wins} \\
%     P(\mbox{tie}) & =  \frac{\nu}{1 + \nu}    && \mbox{a tie}
% \end{align}

The probabilities of the three outcomes sum to one for both of these Bradley-Terry extensions: $p(y_i \succ y_j) + p(y_j \succ y_i) + p(y_i \sim y_j)  = 1$.
For both models, $p(y_i \sim y_j) = p(y_j \sim y_i)$ and $p(y_i \sim y_j)$ 
tends towards 0 if $\lambda_j \gg \lambda_i$.
Both variants have parameters $\nu$ that control how much probability is allocated to ties.  
Apart from $\nu_{RK} = 1$ or $\nu_D=0$, when both variants agree with Bradley-Terry, some probability is reserved for tied outcomes.

The Rao-Kupper and Davidson models arise from different considerations.   
\citet{rao1967ties} begin with the formulation 
$p^{BT}(y_i \succ y_j) = \frac{1}{4} \int_{-(r_i-r_j)}^\infty \mbox{sech}^2(y/2) dy$ \citep{bradley_statistical_1953} and note its sensitivity to the  difference in values $r_i-r_j$. 
They  note that  some judges ``may not be able to express any real preference''  in paired-comparisons   if their ``sense of perception is not sharp enough'' to detect small differences.  They  
reason that a ``threshold of sensory perception'' is needed such that if the observed difference is less than the threshold, a  judge declares a tie.  They introduce the sensitivity threshold $\alpha_{RK}$ as follows,    
$ p^{RK}(y_i \succ y_j) = \frac{1}{4} \int_{-(r_i-r_j)+\alpha_{RK}}^\infty \mbox{sech}^2(y/2) dy$, and 
Eqs.~\ref{eq:rkwin} and~\ref{eq:rktie} follow for 
$\nu_{RK} = e^{\alpha_{RK}}$.

\citet{Davidson-On-Extending-Bradley-Terry} starts from Luce's ``choice axiom'' \citep{luce1959individual} which states  that a complete system of choice probabilities should satisfy  ${p(y_i\succ y_j)}/{p(y_j\succ y_i)} = {\lambda_i}/{\lambda_j}$, which the Rao-Kupper model fails to do.  \citet{Davidson-On-Extending-Bradley-Terry} observes that it is desirable for the probability of a tie to ``be proportional to the geometric mean of the probabilities of preference''.  Adding this requirement $p(y_i\sim y_j) = 2\nu_D \sqrt{p(y_i \succ y_y)p(y_j\succ y_i)}$ to the choice axioms yields   Eqs.~\ref{eq:davidsonwin} and~\ref{eq:davidsontie} as a preference model that allows for ties and also satisfies the choice axiom.  

% \newpage
The Rao-Kupper win and tie probabilities can be written in a form more useful for DPO (Appendix \ref{app:Rao-Kupper and Davidson Preference/Tie Probabilities}),  with $\nu_{RK} = e^{\alpha_{RK}}$, as
\begin{align}
    &p_\theta^{RK}(y_w \succ_x y_l)  =  \sigma(d_\theta(x,y_w,y_l) - \alpha_{RK}) 
\label{eq:rkwinpprob}\\
    &p_\theta^{RK}(y_w \sim_x y_l) 
    % & = (\nu_{RK}^2-1) \, \sigma(-d_\theta(x,y_w,y_l) - \alpha_{RK}) \, \sigma(d_\theta(x,y_w,y_l) - \alpha_{RK})  \notag    \\
    =  (\nu_{RK}^2-1) \, \sigma(-d_\theta(x,y_w,y_l) - \alpha_{RK}) \, p_\theta^{RK}(y_w \succ_x y_l)
    \label{eq:rktiepprob}
\end{align}
and the Davidson win and tie probabilities can be written as
\begin{align}
    p^{D}_\theta(y_w\succ_x y_l) &= \frac{ 1 }{ 1 + e^{-d_\theta(x,y_w,y_l)} + 2 \nu_D e^{-d_\theta(x,y_w,y_l)/2}} \label{eq:dwinpprob} \\ 
    p^{D}_\theta(y_w\sim_x y_l) 
    &= 2\, \nu_D\, e^{-d_\theta(x,y_w,y_l)/2}\, p_\theta^{D}(y_w\succ_x y_l) 
    \label{eq:dtiepprob}
\end{align}
Although their parametric forms are different, their  treatments of wins and ties are similar (Appendix~\ref{app:Rao-Kupper and Davidson Preference/Tie Probabilities}, Fig.~\ref{fig:probability analytical}).   For pairs $(x,y_w, y_l)$ treated as wins, higher likelihood is assigned for  higher values of the reward margin $d_\theta(x,y_w,y_l)$.  For the Rao-Kupper this is particularly clear, in that the Bradley-Terry preference distribution is simply shifted by $\alpha_{RK}$.  Conversely, 
for pairs $(x,y_w, y_l)$ treated as ties, the probability of declaring a tie is high for small reward margins $d_\theta(x,y_w,y_l)$.  

\paragraph{Balancing Wins and Ties.}
In the special case of two evenly matched players ($\lambda_i = \lambda_j$), we are interested in the probability of a tie $p(y_i\sim y_j)$ versus a clear win by either player, $p(y_i\succ y_j)+p(y_j\succ y_i)$. It follows that $P_{RK}(\mbox{tie}) = \frac{\nu_{RK}-1}{2}P_{RK}(\mbox{no tie}) $ and $P_D(\mbox{tie}) = \nu_{D} P_D(\mbox{no tie})$. This shows that the parameters $\nu$ determine the probability that equally-matched items are judged as tied or not. 
% If data including observed ties is available, this can be estimated~\citep{Davidson-On-Extending-Bradley-Terry}. 
$\nu$ can be tuned, but in our work, we assume that equally-matched items will tie with a probability of $1/2$ and so we set $\nu_{RK}=3$ and $\nu_{D}=1$.

\subsection{Incorporating Rao-Kupper and Davidson Models into DPO}

% The dataset of preferences used in DPO is augmented with a binary flag $t$ to indicate a tie. 
We extend the DPO policy objective (Eq. \ref{eq:dponll})  to include a binary flag $t$ to indicate a tie:
% \citep[Eq. 7]{DPO-Rafael}
\begin{align}
    \mathcal L(\pi_\theta; \pi_{ref}) =  
    - \mathbb{E}_{x,y_w,y_l,t=0} \log p_\theta(y_w \succ_x y_l)
    - \mathbb{E}_{x,y_w,y_l,t=1} \log p_\theta(y_w \sim_x y_l) 
    \label{eq:extended_obj_with_tie}
\end{align}
where  $p_\theta(y_w \succ y_l) $  and $p_\theta(y_w \sim y_l)$ are taken  from either the Rao-Kupper model (Eqs.~\ref{eq:rkwinpprob},~\ref{eq:rktiepprob} or the Davidson model (Eqs.~\ref{eq:dwinpprob}, ~\ref{eq:dtiepprob}).  Note that in  Eq.~\ref{eq:extended_obj_with_tie} preference pairs in the dataset are unambiguously either wins ($t=0$) or ties ($t=1$).
The policy objectives for these two DPO variants are:
\begin{align}
&\mathcal{L}_{RK}(\pi_\theta; \pi_{ref}) = - 
\mathbb{E}_{x,y_w,y_l,t=0} \Bigl[\log \sigma(d_\theta(x, y_w, y_l)- \alpha_{RK}) \Bigr] \notag \\ &- 
\mathbb{E}_{x,y_w,y_l,t=1} \Bigl[ \log \sigma(-d_\theta(x, y_w,y_l) - \alpha_{RK}) +  \log \sigma(d_\theta(x, y_w, y_l) - \alpha_{RK}) - \log(\nu_{RK}^2-1)    \Bigr]
\label{eq:dpork_loss}
\end{align}

and
\begin{align}
\mathcal{L}_{D}(\pi_\theta; \pi_{ref}) =  &- 
\mathbb{E}_{x,y_w,y_l,t=0} \Bigl[ \log \frac{ 1 }{ 1 + e^{-d_\theta(x,y_w,y_l)} + 2 \nu_D e^{-d_\theta(x,y_w,y_l)/2}} \Bigr] \notag \\
&- 
\mathbb{E}_{x,y_w,y_l,t=1} \Bigl[ \log \frac{ 2 \nu_D e^{-d_\theta(x,y_w,y_l)/2} }{ 1 + e^{-d_\theta(x,y_w,y_l)} + 2 \nu_D e^{-d_\theta(x,y_w,y_l)/2}} \Bigr] 
\label{eq:dpod_loss}
\end{align}
We refer to these DPO variants as DPO-RK and DPO-D.  Like DPO, these objectives depend on the policy $\pi_\theta$ through the reward margin $d_\theta(x, y_w, y_l)$ (Eq.~\ref{eq:rewarddiff}). 
Unlike DPO, the training objective Eq.~\ref{eq:extended_obj_with_tie} consists of two competing terms.
For pairs $(x, y_w, y_l)$ labeled as wins the objective is to find $\pi_\theta$ to increase the reward margin $d_\theta(x,y_w,y_l)$.  However, for pairs labeled as ties
the objective is to find $\pi_\theta$ to minimize  $|d_\theta(x,y_w,y_l)|$.  To simultaneously achieve both these objectives, the underlying policy should learn to model both wins and ties.

\subsubsection{DPO-RK and DPO-D Updates}
\label{sec:DPO-RK and DPO-D Updates}

\citet{DPO-Rafael} show that DPO dynamically adjusts the gradient according to how well the preference objective is optimized for each sample 
\begin{align}
    \nabla_\theta& \log p_\theta^{BT}(y_w \succ_x y_l) = \underbrace{\sigma(-d_\theta(x,y_w,y_l))}_{\substack{\text{higher weight when reward} \\ \text{estimate is wrong}}} \,
    \beta\nabla_\theta \log  \frac{\pi_\theta(y_w|x)}{\pi_\theta(y_l|x)}
%    \Bigl[   \nabla_\theta \log \pi_\theta(y_w|x) - \nabla_\theta \log \pi_\theta(y_l|x)   \Bigr] 
\label{eq:dpo_grad}
\end{align}
DPO-RK and DPO-D also adjust their gradients dynamically (Appendix \ref{app:Gradients for DPO-RK and DPO-D}). We define the gradient scale factors $\Delta_{win}$ and $\Delta_{tie}$ to illustrate the DPO-RK and DPO-D gradient updates on wins and ties:
% DPO-RK gradients:
\begin{align}
    \nabla \log p_\theta^{RK}(y_w \succ_x y_l) & =  \underbrace{\sigma( \alpha - d_\theta(x, y_w, y_l) )}_{  \Delta^{RK}_{win}(d_\theta)} \,
    \beta\nabla_\theta \log  \frac{\pi_\theta(y_w|x)}{\pi_\theta(y_l|x)}
    \label{eq:rkgrad_preference}
\end{align}
\begin{align}
   &\nabla_\theta \log p_\theta^{RK}(y_w \sim_x y_l) = \bigl[ \underbrace{\sigma(\alpha -d_\theta(x,y_w,y_l)) - \sigma(\alpha+d_\theta(x,y_w,y_l))}_{\Delta^{RK}_{tie}(d_\theta)} \bigr]
\beta\nabla_\theta \log  \frac{\pi_\theta(y_w|x)}{\pi_\theta(y_l|x)}
\label{eq:rkgrad_tie}
\end{align}
% Davidson gradients:
\begin{align}
    \nabla_\theta \log p^{D}_\theta(y_w\succ_x y_l)  &= \underbrace{\frac{  e^{-d_\theta} + \nu e^{-d_\theta/2} }{  1+e^{-d_\theta}+ 2\nu e^{-d_\theta/2}}}_{  \Delta_{win}^{D}(d_\theta) } \, 
    \beta\nabla_\theta \log  \frac{\pi_\theta(y_w|x)}{\pi_\theta(y_l|x)}
    \label{eq:dgrad_preference}
    \\
    \nabla_\theta \log p^{D}_\theta(y_w\sim_x y_l)  &= \underbrace{\Bigl[  \Delta_{win}^{D}(d_\theta) - \frac{1}{2} \bigr]}_{ \Delta_{tie}^{D}(d_\theta)} \, 
    \beta\nabla_\theta \log  \frac{\pi_\theta(y_w|x)}{\pi_\theta(y_l|x)}
    %\Bigl[   \nabla_\theta \log \pi_\theta(y_w | x) - \nabla_\theta \log\pi_\theta(y_l | x)     \Bigr]
    \label{eq:dgrad_tie}
\end{align}

\paragraph{$\nabla \log p_\theta(y_w \succ_x y_l)$:} For data labeled as wins, the DPO-RK gradient scale factor has the same form as DPO, but shifted by $\alpha_{RK}$ (Fig.~\ref{fig:graident scale factor analytical}).  DPO-D has a symmetric scale factor that is not as steep as DPO-RK.  All three methods work to increase the reward margin  $d_\theta(x,y_w,y_l)$.
\paragraph{$\nabla \log p_\theta(y_w \sim_x y_l)$:} For data labeled as ties, the DPO-D and DPO-RK gradient scale factors are odd and work to drive $d_\theta(x,y_w,y_l)$ towards zero, although the DPO-RK scale factor is more aggressive. This is a mechanism not present in DPO. 

\paragraph{\textbf{Intuition for hyper-parameter $\alpha_{RK}$  and $\nu_{D}$} :} We note that DPO-RK and DPO-D each introduces one additional hyper-parameter relative to DPO. In Appendix \ref{app:Mathematical Derivations} we discuss how the Rao-Kupper and the Davidson variants use these hyper-parameters to control the likelihood of a tie. The updates on ties above (Eq.\ref{eq:rkgrad_tie} and \ref{eq:dgrad_tie}) offer another intuition for these hyper-parameters: $\alpha_{RK}$ and $\nu_D$ control the width of the band in reward margin where there is little gradient contributions from tied pairs whose difference in reward falls within the band. However, for tied pairs whose difference in reward fall outside the band, the gradient updates work to reduce the margin. We find experimentally that performance is not sensitive to the choice of $\alpha_{RK}$ and $\nu_D$ so long as they are in sensible range (Appendix \ref{app:Ablation Study on the Value of alpha for DPO-RK and DPO-D}) and so keep to $\alpha_{RK}=\log 3$ and $\nu_D=1$ as motivated in Sec.\ref{ssec:bte}. 

% Given large $\alpha$, the gradient on TPs approach zero regardless of the reward margin of the pair under the current model.

% On data labeled as clear preferences, DPO-RK has the shape as DPO in the gradient scale factor but shifted by $\alpha_{RK}$ whereas DPO-D has a symmetrical gradient scale factor and not as steep. On data labeled as ties, both DPO-RK and DPO-D have symmetrical gradient scale factors that work to keep the margin close to zero. 

% (!) On data labeled as non-ties, DPO-RK and DPO-D have very similar gradients as DPO despite having different parametric forms of $\Delta_{win}$. On data labeled as ties, DPO-RK and DPO-D may have negative gradient scale factors $\Delta_{tie}$ resulting in updates that reduce the reward margins to keep the tie margin within the operating range. We analyze the training behavior empirically in Appendix \ref{TODO}.

% \subsection{Analyzing Preference Optimization}

% In this section we introduce theories and definitions useful for analyzing the behavior of DPO and our proposed  DPO-RK/DPO-D preference optimization algorithms. 

% \subsubsection{The Ideal DPO Policy}

% \citet{PL-paper-chen} shows that the ideal DPO policy $\pi^*$ should follow:
% \begin{equation}
%     \frac{\pi^*(y_w|x)}{\pi^*(y_l|x)} = \frac{\pi_{\text{ref}}(y_w|x)}{\pi_{\text{ref}}(y_l|x)}(\frac{\gamma(x,y_w,y_l)}{1-\gamma(x, y_w, y_l)})^{1/\beta}
%     \label{eq:ideal DPO}
% \end{equation}

% where $\gamma(x, y_w, y_l)$ is the true preference probability of $y_w \succ y_l$ under prompt $x$. 

%\newpage
\section{Experiments in Adding Ties to DPO, DPO-RK, DPO-D}

DPO in its original formulation relies on a static dataset of comparisons $\mathcal{D} = \{x^{(i)}, y_w^{(i)}, y_l^{(i)}\}_{i=1}^N$ where $y_w^{(i)}$ and $y_l^{(i)}$ are  preferred and dispreferred responses to a prompt $x^{(i)}$~\citep{DPO-Rafael}.  These preferences are assumed to be sampled from some latent reward model and we refer to this dataset as {\bf Clear Preference Pairs} ({\bf CPs}, for short) because they are typically selected to reflect a clear preference between winner and loser as assessed either by human judges or by some trusted automatic metric.   We distinguish CPs from 
{\bf Tied Pairs} ({\bf TPs)}.  TPs also consist of a winner and a loser,  but are very similar in quality.   Human judges might be less consistent, or have less confidence,  in selecting the winner in a tied pair, and automatic metrics will assign more similar or even conflicting quality scores to TPs than to CPs. We study conflicting preferences in Sec.\ref{sec:Performance Improvement through DPO-RK and DPO-D}.

% We now investigate the effects of adding ties to DPO and the DPO variants  DPO-D and DPO-RK.
\subsection{Extending Preference Datasets to Include Ties}
\label{sec:adding ties}

As noted,  DPO datasets typically are constructed to include only CPs.  We develop data selection procedures to generate TPs along with CPs so that we can investigate how DPO changes when Tied Pairs are included in the training data. We follow previous work~\citep{DPO-SelfTraining-GSM8K, IterativeDPOBuildingMA, StatisticalRejectionSamplingForDPO-RSO} to collect preference pairs by sampling and then ranking model responses. We pair the top- and bottom-ranked responses as CPs, and select TPs as pairs that are close in the ranking. For Neural Machine Translation (NMT) on WMT-21 ZH-EN~\citep{WMT21-dataset-paper} and IWSLT-17 FR-EN~\citep{IWSLT17-dataset-paper}, we rank responses by BLEURT, a widely-used reference-based quality metric~\citep{BLEURT-paper, WMT23_Metric_paper}. For Summarization on TL;DR~\citep{TLDR-dataset-OpenAI}, we rank responses using the implicit reward function learned by DPO itself, without an external reward model. Appendix \ref{app:experimental details and full results} gives experiment details. Studies of these selection strategies can be found in Appendix \ref{app:Verifying tied pair selection strategy} and \ref{app:Using GPT-4 to validate ties on NMT}.

% We report experiments on Neural Machine Translation (NMT) and Summarization. For NMT, we 
% Appendix \ref{app:experimental details and full results} gives experiment details.

% BB:  not clear we need this signposting.
% But if we decide to include, we 
%In this section, we characterize the effect of incorporating tied pairs to DPO from three aspects: (1) the KL-Performance tradeoff; (2) the training characteristics and (3) reward margins and reward accuracy on held-out sets. We first introduce our experimental setups and then present results and analyses for each aspect. Our findings are summarized at the end of this section.

%\subsection{Experiments Setup}
%\label{sec:adding ties - experiments setup}

% For NMT, we DPO fine-tune BLOOMZ-mt-7b~\citep{Bloomz-model-paper} on self-generated data following \citet{bruce-dpo-mbr} and evaluate on WMT21 ZH-EN and IWSLT17 FR-EN with BLEURT scores~\citep{BLEURT-paper}. 
% For summarization, we follow \citet{ODPO-Direct-Preference-Optimization-with-an-offset-Amini} to DPO fine-tune Pythia-2.8B~\citep{Pythia-model-paper} on the TL;DR dataset~\citep{TLDR-dataset-OpenAI} and evaluate with win-rate against human-written summaries. We use PairRM win-rate which correlates closely with GPT-4 win-rate on TL;DR (Appendix \ref{app:full experiment details}). 
% Following previous works, we train all systems for one epoch only. Full experimental details are in Appendix \ref{app:full experiment details}. 

% \subsection{Task Performance vs. KL to the Reference Policy - The Regularizing Effects of Ties}
\subsection{Adding Ties to DPO - The Regularizing Effects of Ties}
\label{ss:tpvkl}

Following prior work \citep{DPO-Rafael, ODPO-DirectPreferenceOptimizationWithAnOffset, DisentanglingLengthDPO}, we evaluate DPO in terms of task performance versus KL divergence to the reference policy. For each of the three tasks we form two training sets: CP, which contains the Clear Preference Pairs; and CP+TP, which contains both the Clear Preference Pairs and the Tied Pairs.  We refer to DPO training on these sets as DPO(CP) and DPO(CP+TP). We note IPO~\citep{IPO-AGeneralTheoreticalParadigmtoUnderstandLearningfromHumanPreferences} yields similar results (Appendix \ref{app:IPO Baselines and KIWI-22 Evaluation on WMT}).

% We evaluate two conditions: DPO(CP) and DPO(CP+TP) to dis
% the frontiers of task performance versus KL divergence to the reference policy (Figure~\ref{fig:DPO(CP) vs DPO(CP+TP) KL-Performance frontier}). 
% report DPO task performance versus KL to the reference policy for both conditions in Figure~\ref{fig:DPO(CP) vs DPO(CP+TP) KL-Performance frontier}. 

% Figure~\ref{fig:DPO(CP) vs DPO(CP+TP) KL-Performance frontier} shows Performance vs KL to the reference policy for DPO over the two datasets: Clear Preference Pairs alone (DPO(CP)) and Clear Preference Pairs with Tied Pairs (DPO(CP+TP)).   

% Figure \ref{fig:DPO(CP) vs DPO(CP+TP) KL-Performance frontier} shows the KL-Performance frontier on the three datasets for DPO with Clear Preferences only, DPO(CP), and DPO with Clear Preferences and Ties, DPO(CP+TP). The optimal beta values are annotated. We find that further increasing beta values on WMT21 and IWSLT17 causes the model to generate repetitive, low-quality translations which we discuss in Appendix \ref{app:full experiment details}.
\begin{figure*}[!htbp]
    \centering
    \begin{subfigure}{0.32\textwidth}
        \includegraphics[width=1.0\textwidth]{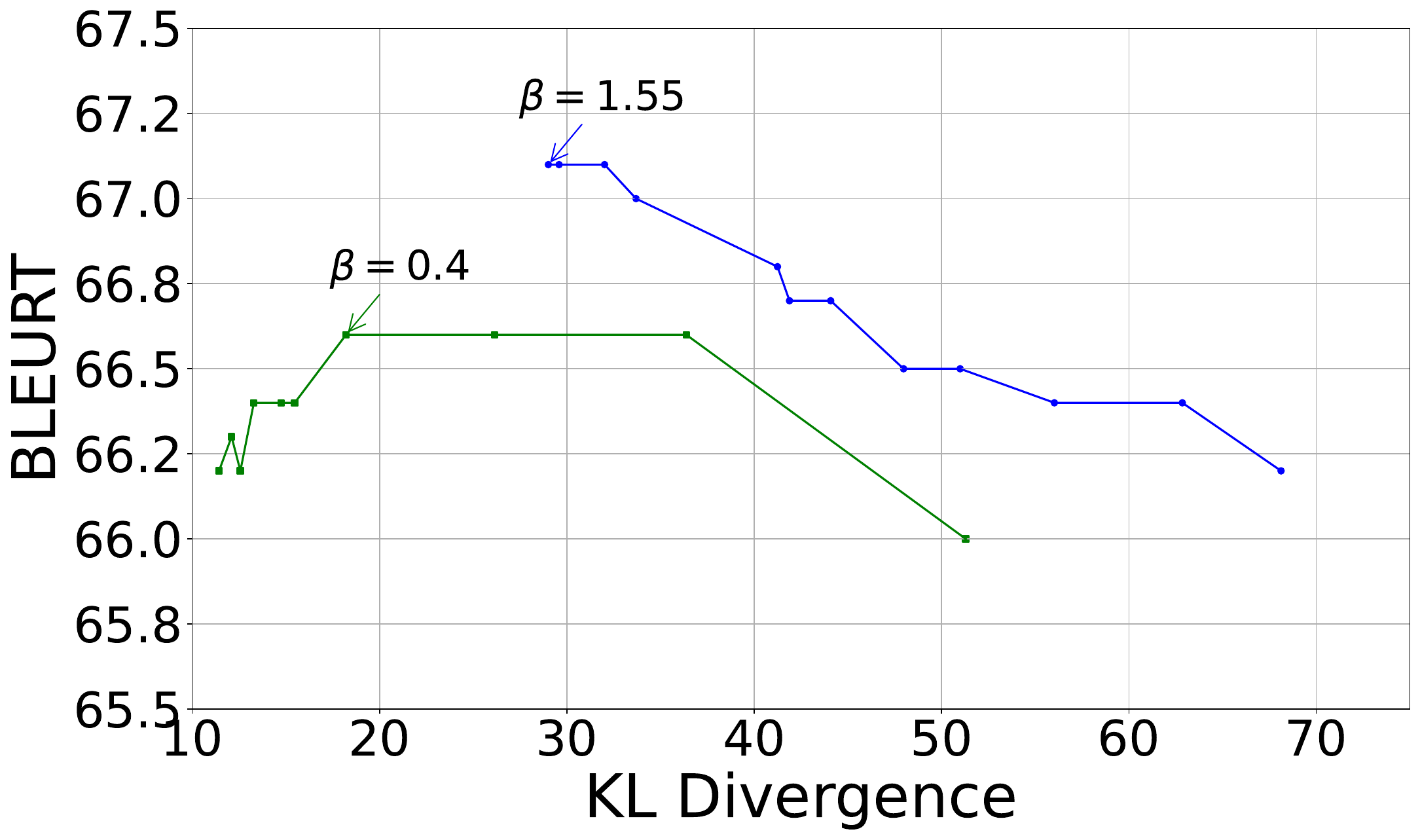}
        \caption{WMT21 ZH-EN. \\ Baseline  = 66.0}
        \label{fig:WMT21 - DPO(CP) vs DPO(CP+TP) KL-Performance frontier}
    \end{subfigure}
    \hfill
    \begin{subfigure}{0.32\textwidth}
        \centering
        \includegraphics[width=1.0\textwidth]{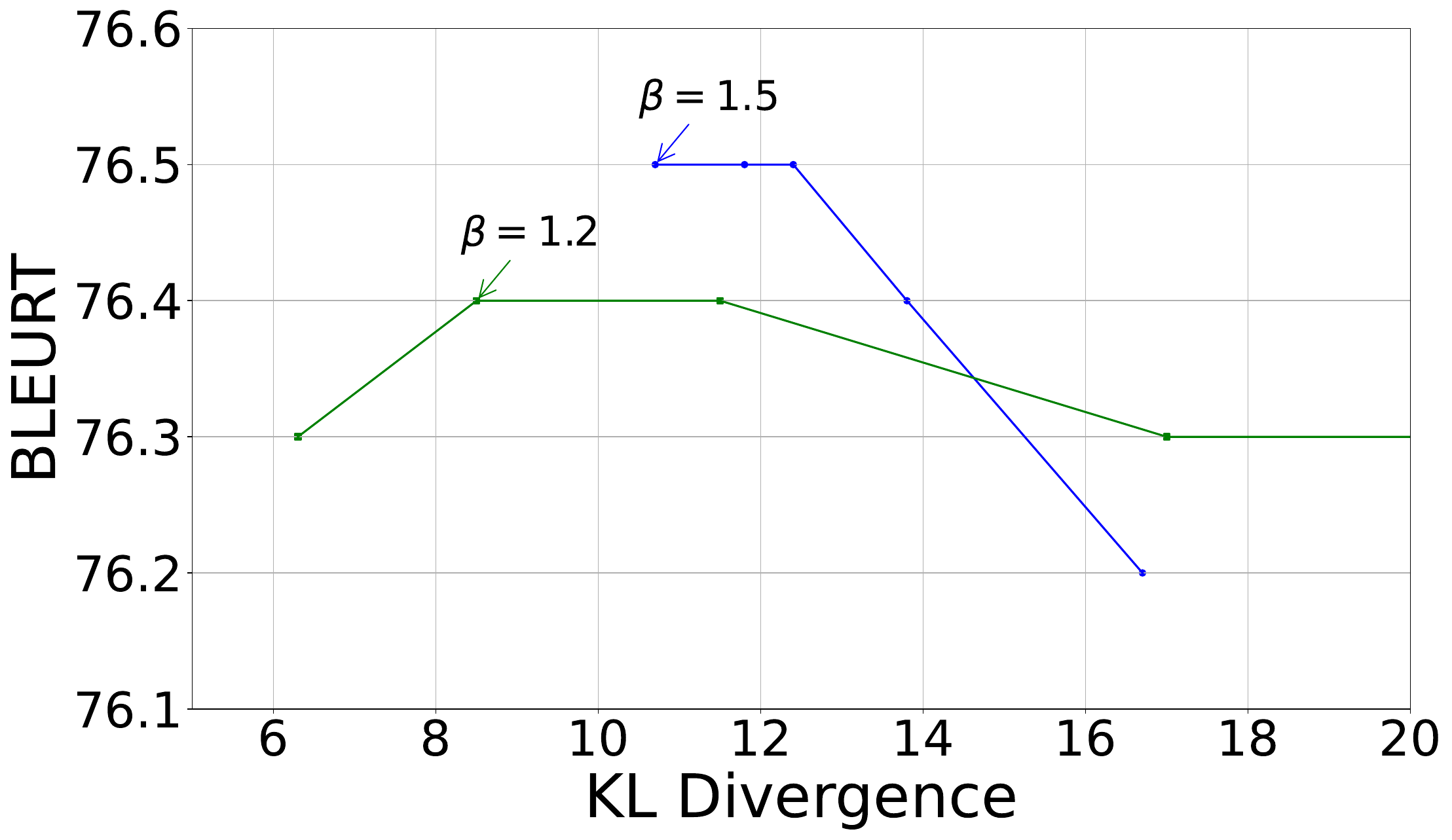}
        \caption{IWSLT17 FR-EN.\\Baseline $=75.7$}
        \label{fig:IWSLT17 - DPO(CP) vs DPO(CP+TP) KL-Performance frontier}
    \end{subfigure}
    \hfill
    \begin{subfigure}{0.32\textwidth}
        \centering
        \includegraphics[width=1.0\textwidth]{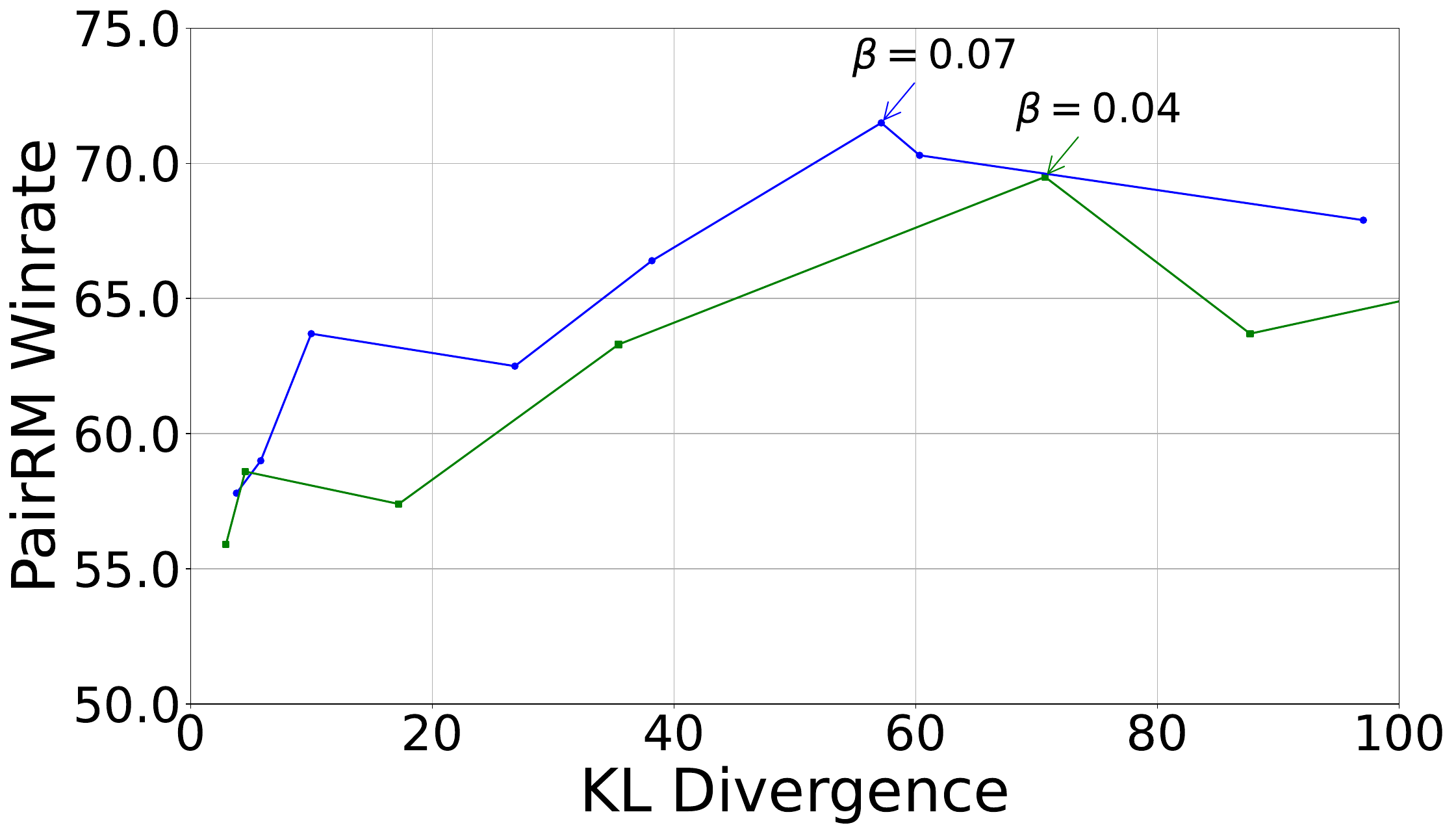}
        \caption{TL;DR PairRM Win-Rate;\\ Baseline = 37.5\%}
        \label{fig:TLDR - DPO(CP) vs DPO(CP+TP) KL-Performance frontier}
    \end{subfigure}
    \caption{Task Performance vs. KL to the reference policy for DPO systems trained on Clear Preference Pairs  (DPO(CP), \textcolor{customblue}{blue}) and on Clear Preference Pairs and Tied Pairs (DPO(CP+TP),  \textcolor{customgreen}{green}).  
    Task Performance is reported in BLEURT for translation tasks on WMT21 ZH-EN and IWSLT17 FR-EN. Summarization performance is reported on TL;DR in terms of PairRM win-rate against human-written summaries.
    KL is estimated over 256 test set policy samples; $\beta$ is noted for best performing systems. Full details are in Appendix \ref{app:experimental details and full results}. }
    \label{fig:DPO(CP) vs DPO(CP+TP) KL-Performance frontier}
    % \vspace{-1em}
\end{figure*}

The obvious conclusion from these experiments (Figure~\ref{fig:DPO(CP) vs DPO(CP+TP) KL-Performance frontier}) is that including tied pairs in  DPO  hurts task performance.  All best performing systems are obtained by DPO(CP), with DPO(CP+TP) underperforming for nearly all values of KL relative to the reference policy.  This performance degradation from including ties is consistent with common practice in the DPO literature which only keeps pairs with clear preference, filtering others to obtain the best-performing system~\citep{yang2024qwen2technicalreport, dubey2024llama3herdmodels}.  Consistent with this,  the TL;DR results show that removing tied pairs from the DPO dataset leads to improved summarization performance, even when ties are identified by a DPO model in an unsupervised manner. 
However these results also suggest that tied pairs in the DPO datasets can enhance regularization.  By this we mean that including TPs causes DPO to find models that are closer to the reference policy as measured by KL divergence.   The overall effect of the reduced task performance and more regularization is to shift the frontier `down and to the left'.

% \textbf{Ties enhance regularization but hurt performance.}
% Overall, the frontiers shift to the bottom-left. This suggests that DPO(CP+TP) can yield system with smaller KL at the end of training than  DPO(CP), indicating stronger regularization with respect to the reference model. However, under comparable KL, DPO(CP+TP) models generally under-perform DPO(CP) models. All best performing systems are obtained by DPO(CP). This performance degradation from including ties is consistent with practice in the DPO literature which only keeps pairs with clear preference, filtering others to obtain the best-performing system~\citep{Zephyr-paper-huggingface}

\subsection{Explaining the Regularization Effect of Ties via Ideal DPO Policy Theory}
\label{sec:Explaining the regularization effect of ties via ideal DPO policy theory}

Theorem 3.1 of~\cite{PL-paper-chen} suggests how these regularization effects might arise.  The ideal DPO policy $\pi^*$ should follow (Appendix \ref{app:Simulating Ideal DPO-Davidon Model}):
\begin{equation}
    \frac{\pi^*(y_w|x)}{\pi^*(y_l|x)} = \frac{\pi_{\text{ref}}(y_w|x)}{\pi_{\text{ref}}(y_l|x)}\Big(\frac{\gamma(x,y_w,y_l)}{1-\gamma(x, y_w, y_l)}\Big)^{1/\beta}
    \label{eq:ideal DPO}
\end{equation}
where $\gamma(x, y_w, y_l)$ is the true preference probability of $y_w \succ y_l$ under prompt $x$. 
If we assume that tied pairs have a true preference probability $\gamma(x, y_w, y_l)$ of 0.5,  from Equation~\ref{eq:ideal DPO} we have 
$\frac{\pi^*(y_w|x)}{\pi^*(y_l|x)} = \frac{\pi_{\text{ref}}(y_w|x)}{\pi_{\text{ref}}(y_l|x)} $, where $\pi^\ast$ is the ideal DPO policy\footnote{In Appendix \ref{app:Simulating Ideal DPO-Davidon Model}, we show that the ideal policy can also be derived for DPO-D which includes the ideal DPO policy as a special case.}.  By this analysis, the ideal DPO model should maintain the same chosen/rejected likelihood ratio as the reference model on tied pairs, and this constraint serves as a form of regularization. In our NMT experiments (Figures \ref{fig:WMT21 - DPO(CP) vs DPO(CP+TP) KL-Performance frontier}, \ref{fig:IWSLT17 - DPO(CP) vs DPO(CP+TP) KL-Performance frontier}), where half of the pairs are constructed to be ties, the regularization effect is especially pronounced 
as the DPO model should keep to the reference model likelihood ratio on 50\% of the training data. Regularization  is less pronounced on TL;DR (Figure \ref{fig:TLDR - DPO(CP) vs DPO(CP+TP) KL-Performance frontier}) where only 1/8 of the pairs are ties.  

Another way to view this is to rearrange Eq \ref{eq:ideal DPO} as follows: 
\begin{align}
    d^*_\theta(x, y_w, y_l) &=\beta\Big(\log\frac{\pi^*(y_w|x)}{\pi_{ref}(y_w|x)}-\log\frac{\pi^*(y_l|x)}{\pi_{ref}(y_l|x)}\Big) = \log \frac{\gamma(x,y_w,y_l)}{1-\gamma(x, y_w, y_l)}
    \label{eq:ideal DPO target margin}
\end{align}
From this it follows that the reward margin on tied pairs should ideally be close to zero, which is a form of regularization. We verify this experimentally in Table \ref{tab:reward_margin_mean_and_variance_combined}. 

\subsection{Adding Ties to DPO-RK and DPO-D - Regularization without Performance Degradation}
\label{sec:Adding Ties to DPO-RK and DPO-D - Regularization without Performance Degradation}
% In the previous section we investigated the effects of including TPs in DPO datasets.
Using the same data as in Sec.\ref{ss:tpvkl}, we now evaluate DPO-RK and DPO-D as DPO variants that  explicitly model both ties and clear preferences.  
% \subsection{Replacing BT with Rao-Kupper and Davidson}
% Methodology. TODO.
% \subsection{Experiments and Analyses}
We use the DPO datasets CP+TP (Sec. \ref{ss:tpvkl}) with the DPO-D and DPO-RK algorithms to produce models DPO-D(CP+TP) and DPO-RK(CP+TP).   
We follow the protocols of Sec.~\ref{sec:adding ties} so that results are directly comparable to earlier DPO(CP) and DPO(CP+TP) results.  We find that the choice of $\nu$ has only a minor effect on task performance (Appendix \ref{app:Ablation Study on the Value of alpha for DPO-RK and DPO-D}). For all experiments we set $\nu^{RK}=3$ and $\nu^D=1$ for DPO-RK and DPO-D as described in Sec.~\ref{ssec:bte}. Training dynamics are provided in Appendix \ref{app:convergence behaviour}.

\begin{figure*}[ht!]
    \centering
    \begin{subfigure}[b]{0.32\textwidth}
        \centering
        \includegraphics[width=1.0\textwidth]{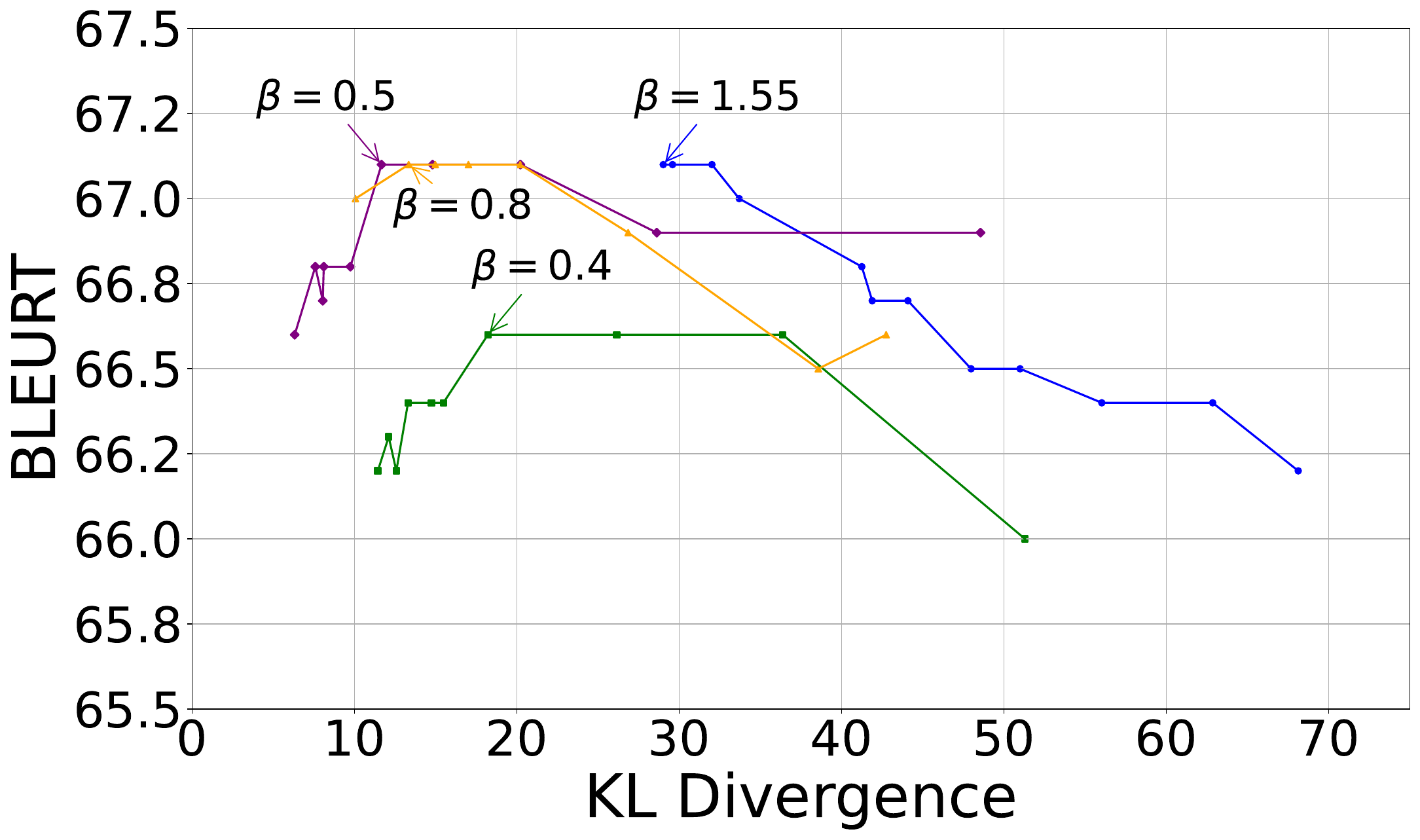}
        \caption{WMT21 ZH-EN frontiers. \\ Baseline BLEURT = 66.0}
        \label{fig:wmt21 zh-en kl-bleurt plot}
    \end{subfigure}
    \hfill
    \begin{subfigure}[b]{0.32\textwidth}
    \includegraphics[width=1.0\textwidth]{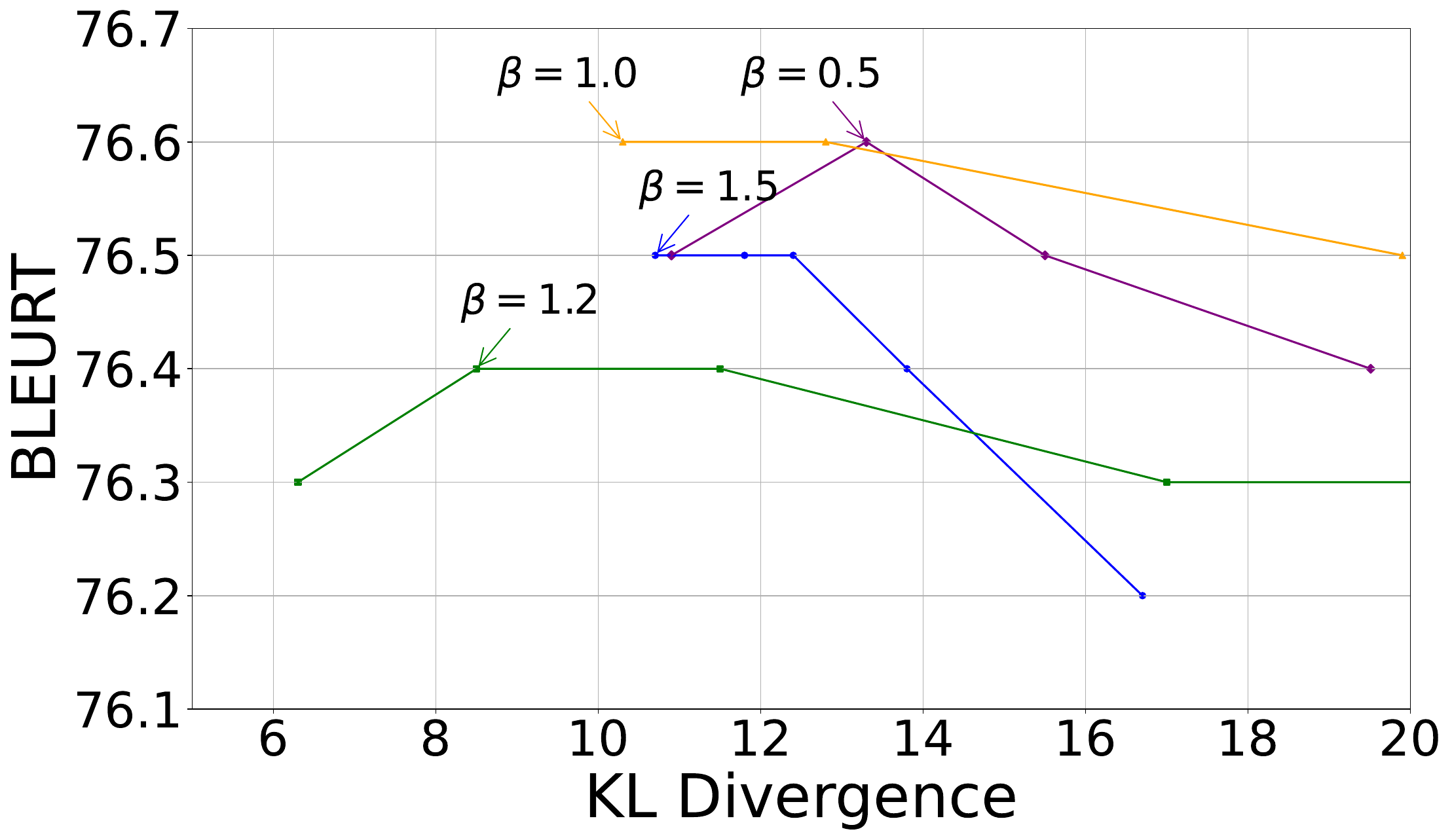}
         \caption{IWSLT FR-EN frontiers. \\ Baseline BLEURT = 75.7}
    \end{subfigure}
    \begin{subfigure}[b]{0.32\textwidth}
        \centering
        \includegraphics[width=1.0\textwidth]{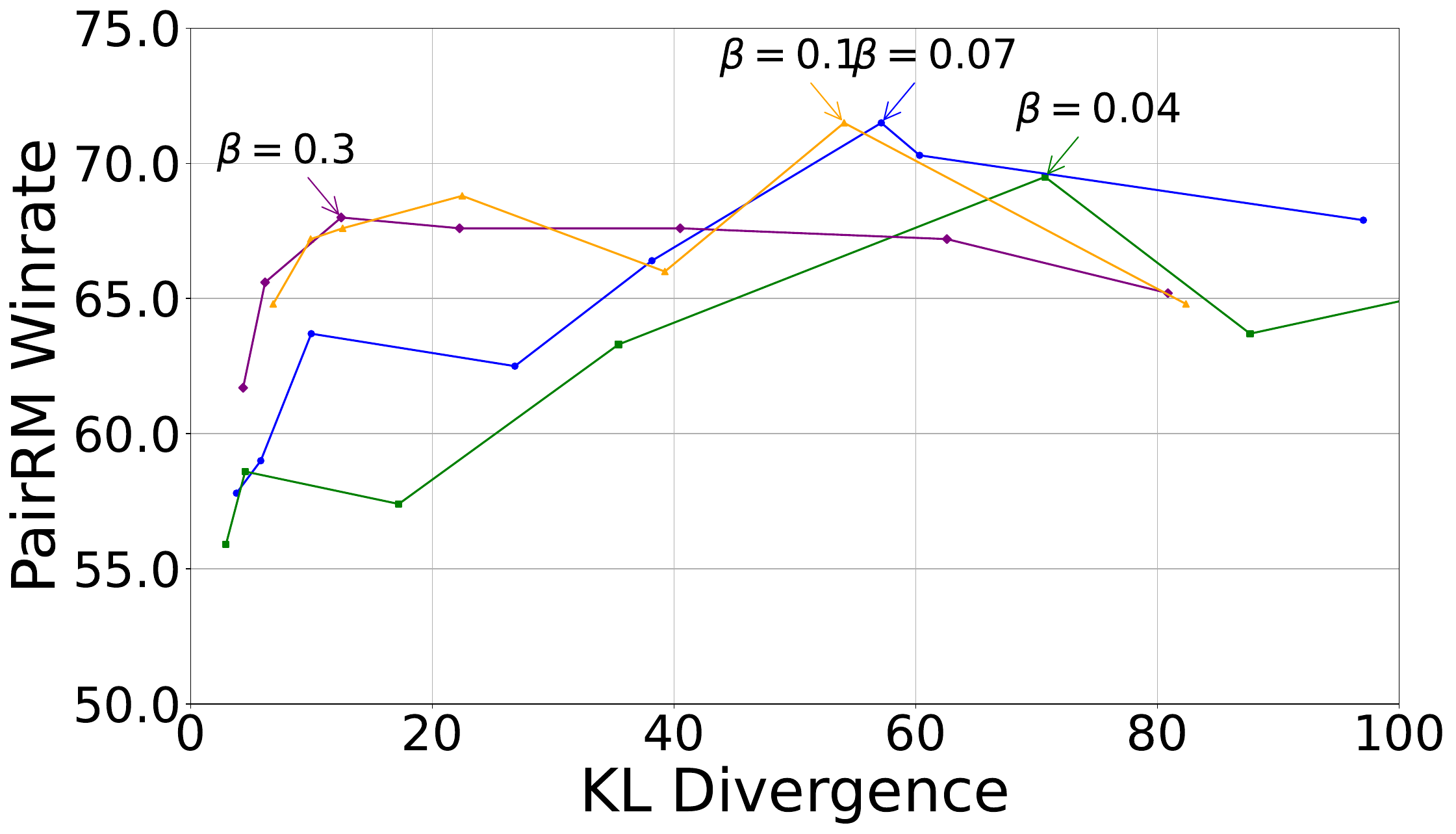}
        \caption{TL;DR frontiers. \\ Baseline win-rate = 37.5\%}
        \label{fig:iwslt-fr-en kl-bleurt plot}
    \end{subfigure}
    \caption{KL-Performance frontiers with DPO(CP) in \textcolor{customblue}{blue}, DPO(CP+TP) in \textcolor{customgreen}{green}, DPO-RK(CP+TP) in \textcolor{custompurple}{purple}, and DPO-D(CP+TP) in \textcolor{orange}{orange}. Full details in Appendix \ref{app:Tabulated KL-Performance Results on NMT and Summarization}. For TL;DR, we additionally report win-rate as judged by GPT-4 in Table \ref{tab:gpt4-evaluation-summarization-app}.}
    \label{fig:rkdavidon-pareto}
    % \vspace{-1.5em}
\end{figure*}

% Similar to Sec.\ref{sec:adding ties}, we characterise DPO-RK and DPO-D with (1) KL-Performance frontier, (2) training statistics, and (3) reward margins and reward accuracy. 

We observe consistent patterns across three datasets: DPO-D(CP+TP) and DPO-RK(CP+TP) do not suffer the same drops in task performance that DPO(CP+TP) exhibits (Fig.~\ref{fig:rkdavidon-pareto}, orange and purple vs. green). Both reach similar levels of task performance to DPO(CP), but do so at smaller KL values than DPO (Fig.~\ref{fig:rkdavidon-pareto}, orange and purple vs. blue).  
% These are the regularization effects of including tie pairs in the DPO datasets reported in Section~\ref{sec:adding ties}, but without decrease in task performance.  
For a given level of KL to reference policy, DPO-D(CP+TP) and DPO-RK(DP+TP) yield higher task performance than DPO(CP). 
Compared to DPO as it is usually done, DPO-RK and DPO-D frontiers are shifted leftwards, showing similar task performance but stronger regularization. Following the analysis in Sec.\ref{sec:Explaining the regularization effect of ties via ideal DPO policy theory}, we empirically show the percentage of TPs in the preference dataset is proportional to the regularization effect in Appendix \ref{app:varying the percentage of ties}.

\subsection{Performance Improvement through DPO-RK and DPO-D}
\label{sec:Performance Improvement through DPO-RK and DPO-D}

In this section, we show two realistic scenarios in machine translation and mathematical reasoning where DPO-RK and DPO-D makes use of otherwise discarded data to achieve performance improvement over strong DPO baselines. We also find DPO-RK and DPO-D leads to overall higher reward accuracy compared to DPO (Appendix~\ref{app:Using DPO-RK and DPO-D model as Preference Classifier}).

\subsubsection{Modeling Conflicting Preferences as Ties}
\label{sec:Modeling Conflicting Preferences}

In curating preference datasets, it is common to employ multiple human or machine annotators to rank candidate responses. Conflicts between annotators can easily arise. This is the case for the ALMA-R-Preference dataset~\citep{ALMA-R_paper}, which contains 3 translations sampled from human, GPT-4, and a base model for each source sentence. These translations are ranked by the average of their XCOMET and KIWI-XXL scores, which are reliable reference-less metrics shown to correlate well with human expert judgment~\citep{WMT23_Metric_paper}. In \citet{ALMA-R_paper}'s original work, the pair of translations with the largest average score difference is selected as a CP and the third, middle-scoring, translation is simply discarded. This yields 3070 CPs, one pair for each source sentence.

We propose a simple scheme to treat translations with conflicting preferences as ties. For pairs not selected as CPs, we identify pairs for which XCOMET and KIWI-XXL disagree about which translation is better as pairs of conflicting preferences. For each source sentence where such pair exists, we add one to DPO-RK/DPO-D training as TP, resulting in 2302 TPs. 
% In addition to the pairs used as Clear Preference in Xu et. al.'s work, we find pairs (if they exist) where there is a ranking contradiction in terms of XCOMET and KIWI-XXL for each source sentence. We find that in the ZH-EN direction, 75\% of the time this discarded translation form a pair of conflicting preference with one of the selected responses under the two metrics and we propose to include such pairs as TPs in DPO-D/DPO-RK training to make full use of the available data. 
Table \ref{tab:conflicting preference ALMA-R} shows DPO-RK(CP+TP) and DPO-D(CP+TP) yield better performance across all evaluation metrics compared to DPO(CP). This shows the ability of our DPO variants to learn from pairs of conflicting preferences.

\begin{table}[h]
\centering
\label{tab:results}
\begin{tabular}{lccccc}
\toprule
Model  & COMET & KIWI-22 & XCOMET & KIWI-XXL \\
\midrule

ALMA-7B-LoRA \cite{ALMA-R_paper} & 79.78 & 76.81 & 83.94 & 73.65 \\ \midrule
~+ DPO(CP) & 79.66 & 77.73 & 88.87 & 74.12 \\ 
~+ DPO-RK(CP+TP) & \textbf{80.63} & \textbf{78.91} & \textbf{90.40} & \textbf{75.77} \\
~+ DPO-D(CP+TP) & 80.38 & 78.27 & 90.09 & 75.54 \\
\bottomrule
\end{tabular}
\caption{ZH-EN translation performance on ALMA-R-Preference test set. The best result is reported for DPO(CP), DPO-RK(CP+TP) and DPO-D(CP+TP) over a beta sweep in [0.1, 0.3, 0.5, 0.7, 0.9].}
\label{tab:conflicting preference ALMA-R}
\vspace{-1.5em}
\end{table}

% Table \ref{tab:conflicting preference ALMA-R} shows DPO-RK(CP+TP) and DPO-D(CP+TP) yield better performance across all evaluation metrics compared to DPO(CP). This shows an effective use of pairs with conflicting preference directions by labeling them as ties, whereas standard DPO practice would typically discard them.

\subsubsection{Preserving good behaviour in the reference model through regularization}
\label{sec:Making use of No-Preference Data}

We have shown that including ties leads to a strong regularization effect with respect to the reference model (Sec.\ref{ss:tpvkl}). We now exploit this property to preserve good performance observed in the reference model while leading to stronger overall performance in mathematical reasoning. 

Following the DPO-augmented Self-Training  (DPO-ST) approach by \citet{DPO-SelfTraining-GSM8K}, we curate preference data by sampling model responses to questions in the GSM8K training set. 
For 2310 of 7470 questions, we find that the moderately-sized Qwen2.5-3B-Instruct model~\citep{QWen2.5-technical-report} answers correctly in all of its 8 sampled  responses (at temperature $0.7$). 
In DPO-ST, these 2310 questions would be excluded from preference data, even though they are a substantial portion ($30.9\%$) of the training set.  We show instead that correct responses to these instances can be treated as ties.

We label CPs and TPs as follows: for training instances where there is at least one correct and one incorrect model response, we randomly choose a correct-versus-incorrect pair as a CP; for examples where all model responses are correct, we randomly choose a pair as a TP. We then conduct DPO training on CP and DPO-RK and DPO-D training on CP+TP with a range of beta values. We also report SimPO and CPO training on the CP set for comparison.

\begin{table}[!ht]
\centering

\label{tab:beta_results}
\begin{tabular}{lccccc}
\toprule
$\beta$ & SimPO (CP) & CPO (CP) & DPO (CP) & DPO-RK (CP+TP) & DPO-D (CP+TP) \\
\midrule
0.1 & 82.5\% & 82.6\% & 76.4\% & \textbf{83.5\%} & 81.7\% \\
0.3 & 81.8\% & 83.1\% & 83.7\% & \textbf{84.4\%} & 83.2\% \\
0.5 & 81.8\% & 83.1\% & 83.6\% & 83.8\% & \textbf{84.5\%} \\
0.7 & 81.6\% & 82.6\% & 83.3\% & 83.7\% & \textbf{84.5\%} \\
1.0 & 82.2\% & 83.6\% & 83.5\% & \textbf{84.1\%} & 83.7\% \\
\bottomrule
\end{tabular}
\caption{GSM8K test set performance with greedy decoding after one-epoch of preference optimization for a range of beta values, evaluated by exact match after ``\#\#\#\#''. Tie-compatible variants yield better performance for every beta value. The base Qwen2.5-3B-Instruct model scores 70.9\%. Standard deviation is reported in Appendix \ref{app:GSM8K-variance}.}
\label{tab:GSM8K_DPORKD_results}
\vspace{-1.0em}
\end{table}

Table \ref{tab:GSM8K_DPORKD_results} shows that DPO-RK(CP+TP) and DPO-D(CP+TP) achieve better performance than the DPO(CP) model at every beta value. Intuitively, these TPs provide strong regularization signals for examples where the reference model already excels, thereby preserving good behaviour in DPO-RK/DPO-D training. Such a mechanism is not available in the original DPO formulation. To verify this intuition, we define Preservation Rate (PR) as the percentage of questions where the policy continues to answer the question correctly provided that the reference model gives the correct answer under greedy decoding. 
Lower PR indicates more instances where the answer flips from correct to incorrect after training, which is clearly undesirable.
Table \ref{tab:GSM8K_PR_CR_DR_of-DPO(CP)_and_DPORKD(CP+TP)} shows PR along with KL and overall accuracy on the GSM8K test set. We find that compared to DPO(CP), DPO-RK(CP+TP) and DPO-D(CP+TP) has lower KL and higher PR which leads to higher overall accuracy.

\begin{table}[!ht]
\centering
\label{tab:evaluation_metrics}
\begin{tabular}{lccc}
\toprule
Model (optimal $\beta$) & KL & Overall Acc\% & PR\%~$\uparrow$  \\
\midrule
Base model & 0 & 70.9\% & 100\% \\
DPO(CP), $\beta=0.3$ & 2.258 & 83.7\% & 95.19\% \\
DPO-RK(CP+TP), $\beta=0.3$ & 1.762 & 84.4\% & 97.11\%  \\
DPO-D(CP+TP), $\beta=0.7$ & 1.465 & \textbf{84.5\%} & \textbf{97.65\%} \\
\bottomrule
\end{tabular}
\caption{Preservation Rate (PR\%) of the best-performing models.}
\label{tab:GSM8K_PR_CR_DR_of-DPO(CP)_and_DPORKD(CP+TP)}
\vspace{-2em}
\end{table}

\section{Additional Supporting Experiments and Analyses}
\label{sec:Additional Supporting Experiments and Analyses}

Due to space constraints, we briefly highlight additional findings and provide pointers to supporting experiments in the appendices. 

\paragraph{\textbf{DPO-RK and DPO-D yields higher reward accuracy than DPO}} Like DPO, DPO-RK and DPO-D yields an implicit reward function formed by the trained policy and the reference model, which can be used to classify preference pairs. In Table \ref{tab:RK-based and D-based accuracy for DPO, DPO-RK, DPO-D}, we show that DPO-RK and DPO-D yields overall higher reward accuracy compared to DPO in classifying held-out clear preference pairs and ties. See Appendix \ref{app:Using DPO-RK and DPO-D model as Preference Classifier} for details.

\paragraph{\textbf{The proportion of TPs controls regularization}} We show empirically that the number of ties included in DPO-RK and DPO-D training is proportional to the strength of regularizaion (Table \ref{tab:tp-proportion-regularization}). This provides additional support for the regularization effect of ties shown in Sec.\ref{sec:Adding Ties to DPO-RK and DPO-D - Regularization without Performance Degradation}.

\paragraph{\textbf{Ideal DPO-D policy theory}} Following the ideal DPO policy theory by \citet{PL-paper-chen} used in explaining regularization (Sec.\ref{sec:Explaining the regularization effect of ties via ideal DPO policy theory}), we derive the ideal DPO-D policy theory starting from the ternary classification objective that accounts for ties in Appendix \ref{app:Simulating Ideal DPO-Davidon Model}. This can be view as an extension to \citet{PL-paper-chen}'s theory in that the ideal DPO policy emerges as special case. 

\textbf{Convergence behaviour} In Appendix \ref{app:convergence behaviour}, we investigate how reward margins and gradient scale factors (defined in Eqs.\ref{eq:dpo_grad} - \ref{eq:dgrad_tie}) evolve when tied pairs are included in training. We find that reward margins on TPs remain close to zero as desired, and that reward margins on CPs rise less sharply and to lower magnitude compared to DPO(CP). An interesting consequence is that, contrary to DPO(CP) training where gradient scale factor decreases to $<0.05$ after 1/4 epoch and stabilizes, the gradient scale factors remain non-negligible throughout the entirety of training for DPO(CP+TP), DPO-RK(CP+TP), and DPO-D(CP+TP), indicated prolonged learning.  

\paragraph{Qualitative Analyses} Example responses from DPO, DPO-RK, and DPO-D systems are provided and analyzed in Appendix \ref{app:Qualitative Cases from DPO, DPO-RK, and DPO-D Systems}.

\section{Related Work}
\label{sec:related work}
\paragraph{Related Variants of Direct Preference Optimization}  ODPO~\citep{ODPO-Direct-Preference-Optimization-with-an-offset-Amini} incorporates preference strength in the objective by introducing an offset parameter. In deriving ODPO, the offset parameter of \citet[Theorem 3]{ODPO-DirectPreferenceOptimizationWithAnOffset})  plays a role similar to 
 the sensitivity threshold of \citet{rao1967ties}. %To our knowledge, our work is the first to consider accommodating tied pairs in DPO. 
 We note that the ODPO objective with a fixed offset agrees with our proposed DPO-RK objective restricted to clear preference data, but does not extend to ties. We note DPO-RK is independently proposed by contemporaneous work~\cite{TODO-TernaryPreferences}. Compared to their work, we additionally introduce DPO-D, propose novel tie selection strategies, identify the regularization effects of ties and provide theoretical explanations. Our experimental setups are complementary: while they primarily focus on general chatbots, we conduct experiments in translation, summarization and mathematical reasoning.

\paragraph{Frameworks for Pair-wise Preference Optimization} Several works propose theoretical frameworks for understanding general Preference Optimization from which DPO can be obtained as a special case. \citet{DensityEstimationPerspectiveForPreferenceOptimization} formulate learning from pair-wise preference as learning the implicit preference generating distribution of the annotators. In this formalism, DPO is a well-specified model for the implicit preference distribution assuming that the human preference generative process follows the Bradley-Terry model. Our work can be viewed as assuming an annotator preference generating distribution that allows for the outcome of a tie (i.e. the Rao-Kupper or the Davidson model).
\citet{GPO-GeneralizedPreferenceOptimization-AUnifiedApproachtoOfflineAlignment} propose a generalized approach to deriving offline preference optimization losses through binary classification. In this work, we extend binary classification to ternary classification with the possibility of declaring a tie (Appendix \ref{app:Simulating Ideal DPO-Davidon Model}).

\paragraph{Pair-wise Comparison Models}  \citet{TheManyRoutesToTheUbiquitousbradleyterrymodel} review the  Bradley-Terry model, including its relation to the logistic distribution~\citep{BradleyTerryandLogisiticDistribution} and  Luce choice axiom~\cite{LuceChoice}. 
The Rao-Kupper~\citep{rao1967ties} and the Davidson model~\citep{David-TheMethodOfPairedComparisons} are two notable extensions to Bradley-Terry (Sec.~\ref{ssec:bte}). 
See the review by \citet{David-TheMethodOfPairedComparisons}  and  bibliography by \citet{A-Bibliography-on-the-Method-of-Paired-Comparisons}. Modeling ties remains an active research topic in fields such as sport team ranking~\citep{ImprovePageRankSport},  medical treatments~\citep{TieMedicalRecord}, and chatbots~\citep{RankingChatbots}.

% \begin{itemize}
%     \item \textbf{DPO and Variants}
%     \begin{itemize}
%         \item Loss variants: IPO. Non-reference model variants: SimPO, etc.
%         \item Incorporate Preference Strength in the Training Objective (PRO, ODPO)
%         \item Length Normalization (Disentangling Length, SimPO)
%     \end{itemize}
%     \item \textbf{Bradley-Terry Extensions}
%     \begin{itemize}
%         \item To incorporate ties
%         \item To incorporate order effect
%         \item To incorporate ranking (general Plackett-Luce model)
%     \end{itemize}
% \end{itemize}

\section{Conclusion}
We have derived and investigated two tie-compatible DPO variants, DPO-RK and DPO-D, by replacing the Bradley-Terry preference model with the Rao-Kupper and the Davidson models, respectively. 
Our experiments show that DPO-RK and DPO-D can accommodate tied pairs in preference data without the degradation in task performance that is observed when the same tied pairs are added to the original DPO. 
We find empirically that the inclusion of ties in preference learning leads to stronger regularization with respect to the reference model and provide theoretical explanations based on ideal DPO policy theory. We further show our DPO variants can improve model performance over DPO by making fuller use of the available data on translation and mathematical reasoning. These findings motivate the use of tied pairs in available preference data as opposed to wastefully discarding them.
 We discuss limitations in Appendix \ref{app:limitation}.

\newpage
\section{Acknowledgement}
Jinghong Chen is supported by the Warwick Postgraduate Studentship from Christ’s College and the Huawei Hisilicon Studentship for the undertaking of the PhD in Engineering at the University of Cambridge.

Guangyu Yang is supported by Cambridge Commonwealth, European and International Trust for the undertaking of the PhD in Engineering at the University of Cambridge.

Jingbiao Mei is supported by Cambridge Commonwealth, European and International Trust for the undertaking of the PhD in Engineering at the University of Cambridge.

Weizhe Lin is supported by a Research Studentship funded by Toyota Motor Europe (RG92562(24020)) for the undertaking of the PhD in Engineering at the University of Cambridge.

Prof. Bill Byrne holds concurrent appointments as a Professor of Information Engineering at Cambridge University and as an Amazon Scholar. This publication describes work performed at Cambridge University and is not associated with Amazon.

We would also like to thank all the reviewers
for their knowledgeable reviews that helped us
strengthen the contribution.

% \section*{Acknowledgments}

% \eric{TODO: Add Bruce's ack.}

% Jinghong Chen is supported by the Warwick Postgraduate Studentship from Christ's College and the Huawei Hisilicon Cambridge Studentship.

% Weizhe Lin was supported by a Research Studentship funded by Toyota Motor Europe (RG92562(24020)). 

% Prof. Bill Byrne holds concurrent appointments as a Professor of Information Engineering at Cambridge University and as an Amazon Scholar.  This publication describes work performed at Cambridge University and is not associated with Amazon.

% We would also like to thank all the reviewers for their knowledgeable reviews.

\bibliographystyle{unsrtnat}
\bibliography{reference}

\newpage

\appendix

% \section{Appendix Overview}

% \begingroup
% \setcounter{tocdepth}{2}  % Optional: Adjust depth (1 = sections, 2 = subsections)
% \tableofcontents
% \endgroup
% \addcontentsline{toc}{section}{Appendix} % Add the appendix text to the document TOC
% \part{Appendix} % Start the appendix part
% \parttoc % Insert the appendix TOC

\section{Limitations}
\label{app:limitation}
 The effect of accommodating ties in preference learning can be further investigated using human-annotated tied pairs. However, at the time of writing, there is no substantial preference dataset with annotated ties; notably, current annotation guidelines are typically written to explicitly exclude ties. We note that this enforcement of win/lose judgments has likely conditioned the generative process of human preference towards the Bradley-Terry model. A meaningful extension of this work would be to assess the effectiveness of DPO-RK and DPO-D on preference datasets where the annotators are asked to identify ties.
 As explained in Sec \ref{ssec:bte}, the hyper-parameter $\nu_{RK}$ and $\nu_D$ can be tuned which would require either grid search or estimation given ground-truth preference/tie probabilities. We find that the choice of $\nu_{RK}=3$ and $\nu_D=1$ as motivated in Sec \ref{ssec:bte} works well and we did not need to tune the parameter to obtain good performance. It is likely that better performance and more efficient frontiers can be obtained by tuning $\nu$ to better fit the underlying preference generative process for both DPO-RK and DPO-D. Given our focus on accommodating ties from a modeling perspective, we leave performance optimization to future works concerning applications.

\section{Mathematical Derivations}
\label{app:Mathematical Derivations}

\subsection{Rao-Kupper and Davidson Preference and Tie Probabilities}
\label{app:Rao-Kupper and Davidson Preference/Tie Probabilities}

We derive the win and tie probabilities as functions of the reward margin $d_\theta(x, y_w, y_l)=r_\theta(x,y_w) - r_\theta(x,y_l)$ (Eq \ref{eq:rewarddiff}) under the Rao-Kupper (Eq \ref{eq:rkwinpprob}, \ref{eq:rktiepprob}) and  Davidson formulations (Eq \ref{eq:dwinpprob}, \ref{eq:dtiepprob}).

The Rao-Kupper win and tie probabilities can be obtained by substituting $\lambda_w = e^{r_\theta(x,y_w)}$, $\lambda_l = e^{r_\theta(x,y_l)}$ and $\nu_{RK}=e^{\alpha_{RK}}$ into Eq \ref{eq:rkwin} and Eq \ref{eq:rktie}, respectively:
\begin{align*}
    &p_\theta^{RK}(y_w \succ y_l)  =  \frac{\lambda_w}{\lambda_w + \nu_{RK} \lambda_l} = \frac{e^{r_\theta(x,y_w)}}{e^{r_\theta(x,y_w)}+\nu_{RK}e^{r_\theta(x,y_l)}} 
    \\ &= \frac{1}{1+e^{r_\theta(x,y_l)-r_\theta(x,y_w)+\alpha_{RK}}} = \sigma(d_\theta(x, y_w, y_l)-\alpha_{RK})  
    \end{align*}
\begin{align*}
    &p_\theta^{RK}(y_w \sim y_l) =  \frac{(\nu_{RK}^2-1) \lambda_w \lambda_l}{(\lambda_w + \nu_{RK} \lambda_l)(\lambda_l + \nu_{RK} \lambda_w)} \\ &= \frac{(\nu_{RK}^2-1) e^{r_\theta(x,y_w)} e^{r_\theta(x,y_l)}}{(e^{r_\theta(x,y_w)} + \nu_{RK} e^{r_\theta(x,y_l)})(e^{r_\theta(x,y_l)} + \nu_{RK} e^{r_\theta(x,y_w)})} 
    \\ &=  (\nu_{RK}^2-1)\big( \frac{e^{r_\theta(x,y_l)}}{e^{r_\theta(x,y_l)} + \nu_{RK} e^{r_\theta(x,y_w)}}\big) \big(\frac{e^{r_\theta(x,y_w)}}{e^{r_\theta(x,y_w)} + \nu_{RK} e^{r_\theta(x,y_l)}}\big) 
    \\ &= (\nu_{RK}^2-1)\,   \sigma(- d_\theta(x, y_w, y_l) - \alpha_{RK})\, \sigma(d_\theta(x, y_w, y_l) - \alpha_{RK})  \\
    &=   (\nu_{RK}^2-1)\,  \sigma(- d_\theta(x, y_w, y_l) - \alpha_{RK})\,    p_\theta^{RK}(y_w \succ y_l) 
\end{align*}

The Davidson win and tie probabilities can be obtained with the same substitution into Eq \ref{eq:davidsonwin} and Eq \ref{eq:davidsontie}, respectively:
\begin{align*}
    &p_\theta^D(y_w \succ_x y_l) 
    = \frac{\lambda_w}{\lambda_w + \lambda_l + 2 \nu_D \sqrt{\lambda_w\lambda_l}}
    \\ &= \frac{e^{r_\theta(x,y_w)}}{e^{r_\theta(x,y_w)} + e^{r_\theta(x,y_l)} + 2 \nu_D \sqrt{e^{r_\theta(x,y_w)+r_\theta(x,y_l)}}} \\
    &= \frac{ 1 }{ 1 + e^{-d_\theta(x,y_w,y_l)} + 2 \nu_D e^{-d_\theta(x,y_w,y_l)/2}} 
    \\
    &p_\theta^D(y_w \sim_x y_l) =  \frac{ 2 \nu_D \sqrt{\lambda_w\lambda_l}}{\lambda_w + \lambda_l + 2 \nu_D  \sqrt{\lambda_w\lambda_l}} \\ &= (2\nu_D \lambda_w^{-\frac{1}{2}} \lambda_l^{\frac{1}{2}})   \frac{\lambda_w}{\lambda_w + \lambda_l + 2 \nu_D \sqrt{\lambda_w\lambda_l}} \\
    &= 2\, \nu_D\, e^{-\frac{1}{2}(r_\theta(x,y_w)-r_\theta(x,y_l))}\, p_\theta^{D}(y_w\succ_x y_l) \\
    &= 2\, \nu_D\, e^{-d_\theta(x,y_w,y_l)/2}\, p_\theta^{D}(y_w\succ_x y_l) 
\end{align*}
In Figure \ref{fig:probability analytical} we plot the preference and tie probabilities as a function of reward margin $d_\theta$ under Bradley-Terry (as used in DPO), Rao-Kupper (as used in DPO-RK), and Davidson (as used in DPO-D). 

% In both the Rao-Kupper and the Davidson models, the tie probability can be expressed as the preference probability multiplied by a correcting term.

\subsection{Gradients for DPO-RK and DPO-D}
\label{app:Gradients for DPO-RK and DPO-D}
The gradients of the Rao-Kupper log probabilities (Eq \ref{eq:rkgrad_preference}, \ref{eq:rkgrad_tie}) are as follows. For convenience, we use the short-hand $d_\theta$ for $d_\theta(x, y_w, y_l)$. 
\begin{align*}
\nabla_\theta & \log p_\theta^{RK}(y_w \succ_x y_l) \\ &=  \nabla_\theta \log \sigma(d_\theta-\alpha_{RK}) \\ 
&= \sigma( \alpha_{RK} - d_\theta ) \nabla_\theta d_\theta(x, y_w, y_l) \\
&= \underbrace{\sigma( \alpha_{RK} - d_\theta )}_{  \Delta^{RK}_{win}(d_\theta) }
   \Bigl[ \nabla_\theta \log \pi_\theta(y_w | x) - \nabla_\theta \log\pi_\theta(y_l | x) \Bigr]  \\
   & =  \Delta^{RK}_{win}(d_\theta) \nabla_\theta \log \frac{\pi_\theta(y_w | x)}{\pi_\theta(y_l | x)}
   \\
   \nabla_\theta & \log p_\theta^{RK}(y_w \sim_x y_l)  \\
   &= \nabla_\theta \Bigl[ \log \sigma(-d_\theta - \alpha_{RK}) + \log \sigma(d_\theta - \alpha_{RK})    \Bigr] 
   \\
& = - \sigma(d_\theta + \alpha_{RK})\, \nabla_\theta d_\theta + 
\sigma(-d_\theta + \alpha_{RK})\, \nabla_\theta d_\theta \\
& = \underbrace{\Big(\sigma(\alpha_{RK} -d_\theta) - \sigma(\alpha_{RK}+d_\theta)\Big)}_{\Delta^{RK}_{tie}(d_\theta)} 
\nabla_\theta d_\theta   \\
& = \Delta^{RK}_{tie}(d_\theta) \nabla_\theta \log \frac{\pi_\theta(y_w | x)}{\pi_\theta(y_l | x)}
\end{align*}
The gradients of the Davidson log-probabilities (Eq \ref{eq:dgrad_preference}, \ref{eq:dgrad_tie}) follow similarly. 
\begin{align*}
\nabla_\theta & \log p^{D}_\theta(y_w\succ_x y_l)   =  \frac{\nabla_\theta  p^{D}_\theta(y_w\succ_x y_l) }{p^{D}_\theta(y_w\succ_x y_l) } \\
&= \frac{\nabla_\theta  (1+e^{-d_\theta}+ 2\nu e^{-d_\theta/2})^{-1} }{p^{D}_\theta(y_w\succ_x y_l) } \\
&  =
(-1) \frac{(1+e^{-d_\theta}+ 2\nu e^{-d_\theta/2})^{-2}}{p^{D}_\theta(y_w\succ_x y_l) }  (- e^{d_\theta} - \nu e^{d_\theta/2})  \nabla_\theta d_\theta  \\
& = \frac{p^{D}_\theta(y_w\succ_x y_l)^2}{p^{D}_\theta(y_w\succ_x y_l)}(e^{-d_\theta} + \nu e^{-d_\theta/2}) \nabla_\theta d_\theta \\ 
& =
p^{D}_\theta(y_w\succ_x y_l)(e^{-d_\theta} + \nu e^{-d_\theta/2}) \nabla_\theta d_\theta  \\
 & = \underbrace{\frac{  e^{-d_\theta} + \nu e^{-d_\theta/2} }{  1+e^{-d_\theta}+ 2\nu e^{-d_\theta/2}}}_{  \Delta_{win}^{D}(d_\theta) }\nabla_\theta d_\theta \\
&= \Delta_{win}^{D}(d_\theta) \nabla_\theta \log \frac{\pi_\theta(y_w | x)}{\pi_\theta(y_l | x)}
\end{align*}
\begin{align*}
\nabla_\theta  & \log p^{D}_\theta(y_w\sim_x y_l)   = \nabla_\theta \log \big( 2 \nu e^{-d_\theta/2} p^{D}_\theta(y_w\succ_x y_l) \big)   \\ &= \nabla_\theta  \Bigl[ \log p^{D}_\theta(y_w\succ_x y_l) -d_\theta/2  \Bigr] \\
& = \Bigl[\frac{  e^{-d_\theta} + \nu e^{-d_\theta/2} }{  1+e^{-d_\theta}+ 2\nu e^{-d_\theta/2}}  -\frac{1}{2}  \Bigr] \nabla_\theta d_\theta \\
&= \underbrace{\Bigl[  \Delta_{win}^{D}(d_\theta) - \frac{1}{2} \bigr]}_{ \Delta_{tie}^{D}(d_\theta)} \Bigl[   \nabla_\theta \log \pi_\theta(y_w | x) - \nabla_\theta \log\pi_\theta(y_w | x)     \Bigr]   \\
& = \Delta_{tie}^{D}(d_\theta) \nabla_\theta \log \frac{\pi_\theta(y_w | x)}{\pi_\theta(y_l | x)}
\end{align*}
For illustration, we plot $\Delta_{win}$ and $\Delta_{tie}$ as a function of reward margin $d_\theta$ in Figure \ref{fig:graident scale factor analytical}. 

The quantities $\nabla_\theta \mathcal L_D(\pi_\theta; \pi_{ref})$ and $\nabla_\theta \mathcal L_{RK}(\pi_\theta; \pi_{ref})$ follow by substituting the above results into the gradient of Eq~\ref{eq:extended_obj_with_tie}
\begin{align}
    \nabla_\theta & \mathcal L(\pi_\theta; \pi_{ref}) =  \notag \\
    &- \nabla_\theta\mathbb{E}_{x,y_w,y_l,t=0} \log p_\theta(y_w \succ_x y_l) \notag \\
    &- \nabla_\theta\mathbb{E}_{x,y_w,y_l,t=1} \log p_\theta(y_w \sim_x y_l) 
    \label{eq:grad_extended_obj_with_tie}
\end{align}

\newpage 

\begin{figure}[!t]
    \begin{subfigure}{0.49\textwidth}
        \includegraphics[width=1.0\textwidth]{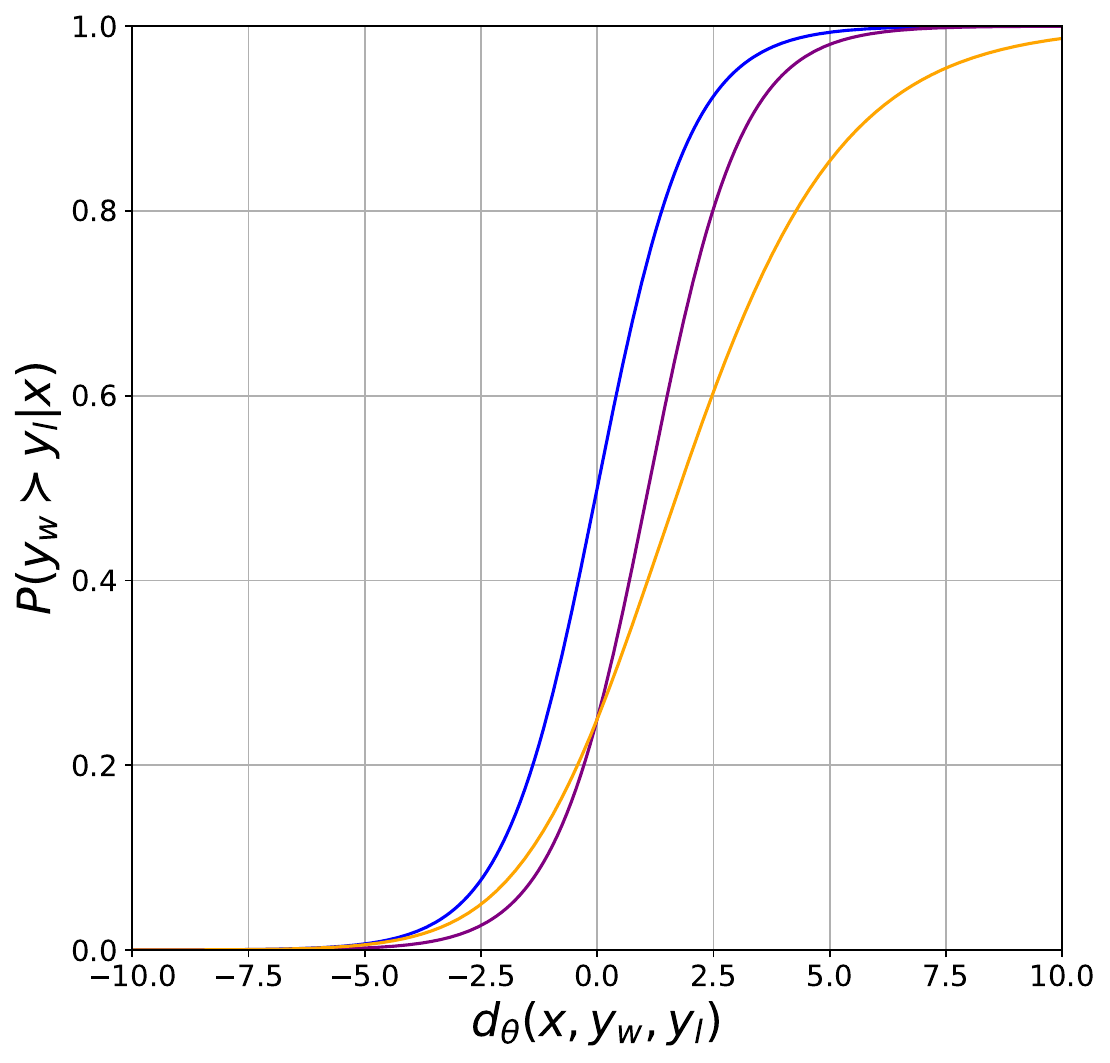}
    \end{subfigure}
    \hfill
    \begin{subfigure}{0.49\textwidth}
        \includegraphics[width=1.0\linewidth]{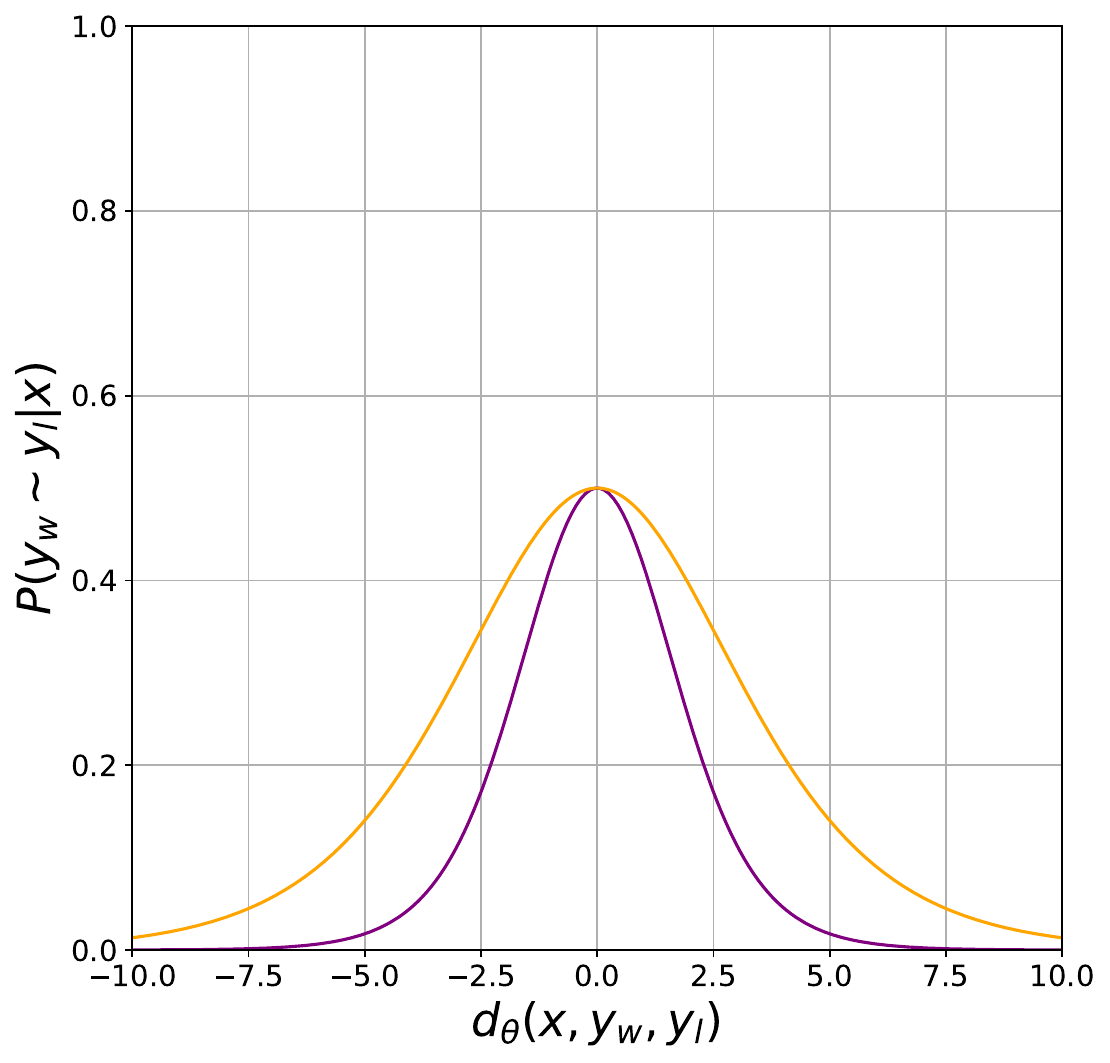}
    \end{subfigure}
    \caption{The clear preference probabilities $P(y_w\succ y_l|x)$ (left) and tie probabilities $P(y_w\sim y_l | x)$ (right)  as a function of reward margins $d_\theta(x, y_w, y_l)$ for Bradley-Terry (as used in DPO)  (\textcolor{customblue}{blue}), Rao-Kupper (\textcolor{custompurple}{purple}) (as used in DPO-RK),  and Davidson (\textcolor{orange}{orange}) (as used in DPO-D). $\alpha_{RK}=\log3$ and $\nu_D=1$ are used in producing these plots. }
    \label{fig:probability analytical}
\end{figure}
\begin{figure}[!ht]
    \begin{subfigure}{0.49\textwidth}
        \includegraphics[width=1.0\textwidth]{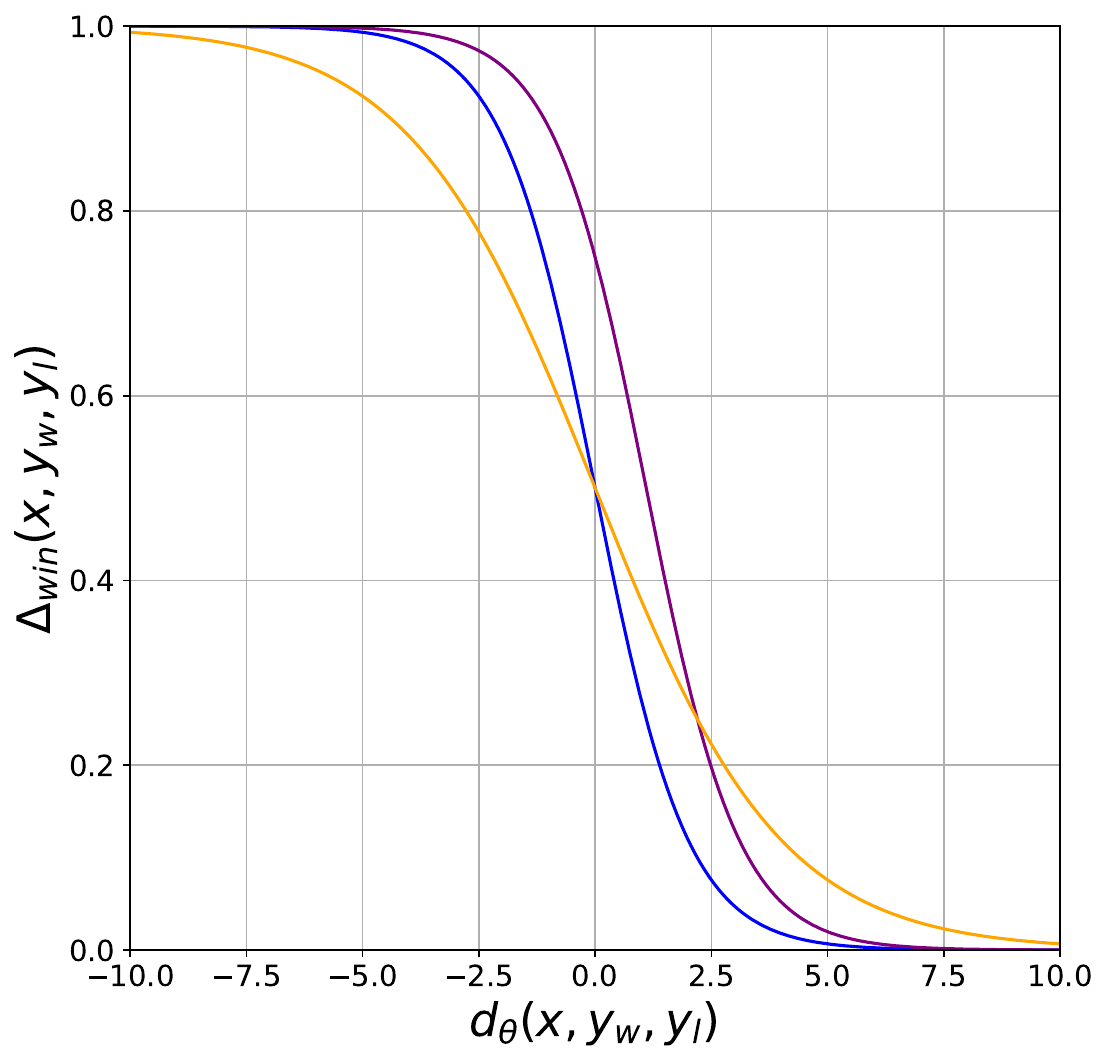}
    \end{subfigure}
    \hfill
    \begin{subfigure}{0.49\textwidth}
        \includegraphics[width=1.0\linewidth]{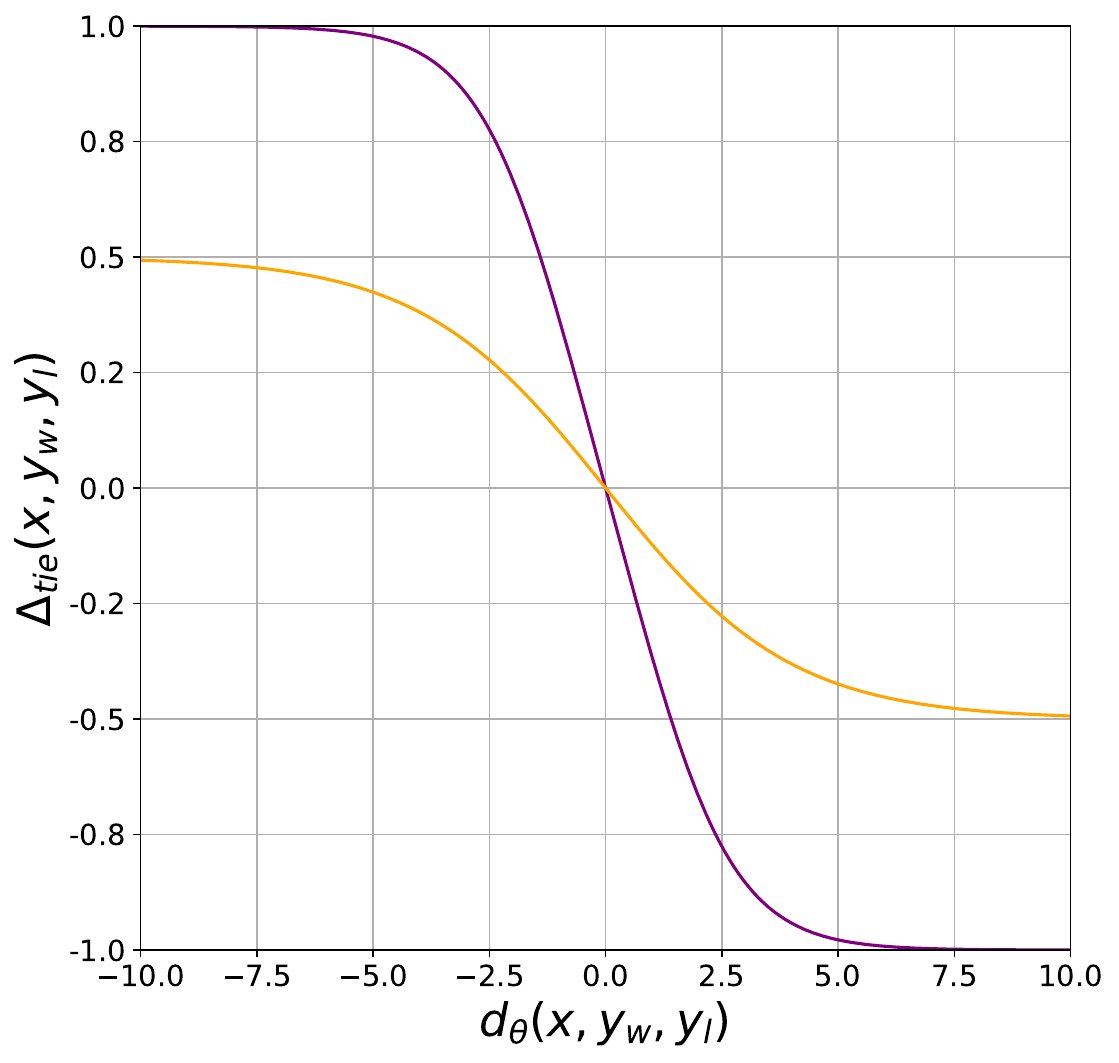}
    \end{subfigure}
    \caption{The gradient scale factors for DPO (\textcolor{customblue}{blue}) and  DPO-RK (\textcolor{custompurple}{purple}) and DPO-D (\textcolor{orange}{orange}) as a function of reward margins $d_\theta(x, y_w, y_l)$ on clear preference pairs (up) and tied pairs (down).$\alpha_{RK}=\log3$ and $\nu_D=1$ are used in producing these plots. }
    \label{fig:graident scale factor analytical}
\end{figure}

\section{Using DPO-RK and DPO-D model as Preference Classifier}
\label{app:Using DPO-RK and DPO-D model as Preference Classifier}

\subsection{Rao-Kupper and Davidson Classifiers}
\label{sssec:classifiers}
The DPO-RK and DPO-D variants yield probability distributions  $p_\theta(y_w \succ_x y_l)$ and $p_\theta(y_w \sim_x y_l)$ in terms of the policy $\pi_\theta$ and the reference model $\pi_{ref}$. 
We can use these distributions as classifiers to label a pair $(x, y_1, y_2)$ as either a win ($y_1 \succ_x y_2$ or $y_2 \succ_x y_1$) or a tie ($y_1 \sim_x y_2$), whichever has the highest probability under either the Rao-Kupper or the Davidson model (Eqs.~\ref{eq:rkwinpprob},~\ref{eq:rktiepprob}, or ~\ref{eq:dwinpprob},~\ref{eq:dtiepprob}). We will evaluate  classification performance on held-out data not used in training to see if policies produced by our DPO variants learn to distinguish wins from ties.

\subsection{Preference Pair Classification Accuracy}
\label{sec:adding ties - Preference Pair Classification Accuracy}
We assess the performance of the Rao Kupper and Davidson classifiers introduced in Sec.\ref{sssec:classifiers} in terms of their ability to label preference pairs as either clear preferences or ties. 
% Classification performance should reflect the ability of the policy reward (Eq) in distinguishing between clear preferences and ties. 
Ideally, classification performance will improve: (1) as tied pairs are added to the clear preference data sets (CP vs CP+TP); and (2) with margins generated from models produced by DPO variants that emphasize the distinction between  tied pairs and clear preference pairs (DPO-D(CP+TP), DPO-RK(CP+TP)).     

We assess classifier performance on the held-out set 
created by collecting CPs and TPs from the WMT18 ZH-EN test set as was done for WMT20 ZH-EN (Appendix \ref{app:experimental details and full results}); this yields pairs with gold labels as either clear preference pairs or tied pairs. Classification and assessment proceeds as follows:  we generate reward margins for the WMT18 ZH-EN pairs using DPO(CP), DPO(CP+TP), DPO-RK(CP+TP), DPO-D(CP+TP) models; we use these reward margins to label the unseen pairs using the Davidson and Rao-Kupper classification rules (Sec.~\ref{sssec:classifiers}); and finally compute the classification accuracy relative to the gold  labels.  

Results are shown in Table~\ref{tab:RK-based and D-based accuracy for DPO, DPO-RK, DPO-D}.  
We find that smaller beta in training consistently leads to better overall RK-classification accuracy (+10\% overall Acc. from $\beta=1.0$ to $\beta=0.1$), suggesting heavy regularization with respect to the reference model impedes preference ranking. 
% This agrees with \citet{PL-paper-chen} who find that attenuating the log-likelihood ratios from the reference model in DPO training leads to better ranking models.
Classifiers based on reward margins generated from DPO(CP) models perform well in identifying clear preference pairs ($\text{Acc.}> 85\%$) but poorly in identifying tied pairs ($\text{Acc.}< 35\%$).  This imbalance is likely explained by the DPO(CP) model never having seen tied pairs in training.  Adding TPs to the DPO datasets (DPO(CP+TP)) significantly improves the classification accuracy of tied pairs ($+30\%$) with more balanced classification accuracies for CPs and TPs.   The best overall classification accuracies ($\approx73\%$) are obtained with reward margins generated by models trained to match its classifier.   Across all beta values, DPO-RK(CP+TP) and DPO-D(CP+TP) achieve better overall accuracy and more-balanced CP accuracy and TP accuracy under their respective decision rules. 

\begin{table*}[!ht]
\centering
\begin{tabular}{l|ccc} \toprule
Model & $\beta=0.1$ & $\beta=0.5$ & $\beta=1.0$ \\ \midrule
&\multicolumn{3}{c}{\underline{Rao-Kupper Classifier}} \\
DPO(CP) & 60.1\% (\textbf{87.1\%}, 33.1\%) & 52.8\% (87.3\%, 18.3\%) & 50.1\% (86.9\%, 13.3\%) \\
DPO(CP+TP) & 67.0\% (72.0\%, 62.1\%) & 57.5\% (69.3\%, 45.7\%) & 51.5\% (71.2\%, 31.9\%) \\ 
DPO-RK(CP+TP) & \textbf{73.1}\% (74.5\%, \textbf{71.7\%}) & 64.2\% (73.2\%, 55.3\%) & 58.5\% (73.4\%, 43.5\%) \\  
\midrule
&\multicolumn{3}{c}{\underline{Davidson Classifer}} \\ 
DPO(CP) & 65.3\% (\textbf{84.4\%}, 46.3\%) & 57.4\% (83.7\%, 31.0\%) & 53.6\% (84.6\%, 22.6\%) \\
DPO(CP+TP) & 71.0\% (59.1\%, \textbf{82.8\%}) & 62.1\% (58.3\%, 65.8\%) & 57.2\% (62.3\%, 52.2\%) \\
DPO-D(CP+TP) & \textbf{73.8\%} (79.6\%, 67.9\%) & 66.8\% (75.9\%, 57.8\%) & 62.7\% (75.2\%, 50.3\%) \\ \bottomrule
\end{tabular}
\caption{
Preference pair classification accuracies (Overall Acc. (CP Acc., TP Acc.)) for Rao-Kupper and Davidson classification rules based on reward margins computed using DPO(CP), DPO(CP+TP), DPO-RK(CP+TP), and DPO-D(CP+TP) models as evaluated on the WMT18 ZH-EN test set. 
}
\label{tab:RK-based and D-based accuracy for DPO, DPO-RK, DPO-D}
\end{table*}

\subsection{Empirical Reward Margin Distributions}
\label{sec:Empirical Reward Margin Distributions}
\label{app:Reward Margin Distributions on Held-out Clear Preference Pairs}
We now look at the reward margins on held-out pairs to determine how the DPO objective generalizes to unseen data.  
Ideally, a post-DPO model should assign reward margins that  are large for clear preference pairs but close to zero for tied pairs.   We assess this on the same held-out data as in the previous section (Sec.~\ref{app:Using DPO-RK and DPO-D model as Preference Classifier}).  

\begin{table*}[!ht]
\centering
\begin{tabular}{l|ccc|ccc} \toprule
Model & $\beta=0.1$ & $\beta=0.5$ & $\beta=1.0$ & $\beta=0.1$ & $\beta=0.5$ & $\beta=1.0$ \\ \midrule
&\multicolumn{3}{c}{\underline{Clear Preference Pairs}} &\multicolumn{3}{|c}{\underline{Tied Pairs} } \\
DPO(CP) & \meanpmstd{8.2}{143.4} & \meanpmstd{9.5}{174.9} & \meanpmstd{10.0}{123.4} & \meanpmstd{0.7}{174.6} &  \meanpmstd{0.6}{88.3} & \meanpmstd{0.4}{61.8} \\
DPO(CP+TP) & \meanpmstd{2.4}{10.6} & \meanpmstd{2.3}{10.1} & \meanpmstd{2.5}{11.0} &  \meanpmstd{0.4}{23.5} & \meanpmstd{0.3}{10.2} & \meanpmstd{0.2}{7.1} \\ 
DPO-RK(CP+TP) & \meanpmstd{2.9}{18.2} & \meanpmstd{2.8}{10.9} & \meanpmstd{3.0}{11.2} & \meanpmstd{0.0}{1.8} & \meanpmstd{0.0}{2.1} & \meanpmstd{0.0}{2.9} \\  
DPO-D(CP+TP) & \meanpmstd{4.6}{34.1} & \meanpmstd{4.8}{36.8} & \meanpmstd{4.9}{40.2} & \meanpmstd{0.0}{3.9} & \meanpmstd{0.1}{5.3} & \meanpmstd{0.0}{6.0} \\  
 \bottomrule
\end{tabular}
\caption{Reward margin statistics (mean $\pm$ std) for held-out Clear Preference Pairs and Tied Pairs collected from WMT18 ZH-EN.}
\label{tab:reward_margin_mean_and_variance_combined}
\end{table*}

% Table \ref{tab:reward_margin_mean_and_variance_combined} shows the mean values of the reward margins calculated on the held-out CPs and TPs.  We find that the mean values are as expected: close to zero for tied pairs and farther from zero for clear preference pairs (Rows 1 vs. 2 for DPO(DP), and  Rows 3 vs. 4 for DPO(CP+TP)). 

% \begin{table}[!ht]
% \centering
% \begin{tabular}{ll|ccc|ccc} \toprule
% \# & Model & $\beta=0.1$ & $\beta=0.5$ & $\beta=1.0$ & $\beta=0.1$ & $\beta=0.5$ & $\beta=1.0$ \\ \midrule
% &&\multicolumn{3}{c}{Clear Preference Pairs} &\multicolumn{3}{|c}{Tied Pairs} \\
% 1 &DPO(CP) & 8.2/143.4 & 9.5/174.9 &10.0/123.4 & 0.7/174.6 &  0.6/88.3 & 0.4/61.8 \\
% 2 &DPO(CP+TP) &   2.4/10.6 & 2.3/10.1 & 2.5/11.0  &  0.4/23.5 & 0.3/10.2 & 0.2/7.1 \\ 
% 3 &DPO-RK(CP+TP) & 2.9/18.2 & 2.8/10.9 & 3.0/11.2 & 0.0/1.8  & 0.0/2.1 & 0.0/2.9\\  
% 4 &DPO-D(CP+TP) & 4.6/34.1 & 4.8/36.8 & 4.9/40.2 & 0.0/3.9  & 0.1/5.3 & 0.0/6.0 \\  
%  \bottomrule
% \end{tabular}
% \caption{Reward margin statistics (mean/variance) for Clear Preference Pairs and Tied Pairs from WMT20 ZH-EN.}
% \label{tab:reward_margin_mean_and_variance_combined}
% \end{table}

In Table \ref{tab:reward_margin_mean_and_variance_combined}, reward margins of DPO(CP+TP), DPO-RK(CP+TP), and DPO-D(CP+TP) are similar and well-behaved, showing means close-to-zero on TPs ($<0.4$) and farther from zero for CPs ($>2.3$). Reward margin standard deviations are also similar and reasonably small. 
% ($\approx \pm 5.0$). 
% , with DPO-D having slightly higher standard deviation on CPs ($\approx \pm 6.0$). 
However the standard deviation for both tied pairs and clear preference pairs are much higher for DPO(CP) models ($\geq 11.1$ on CPs and $\geq7.9$ on TPs).

\begin{figure}[!ht]
    \centering
    \begin{subfigure}[b]{0.48\textwidth}
        \includegraphics[width=\textwidth]{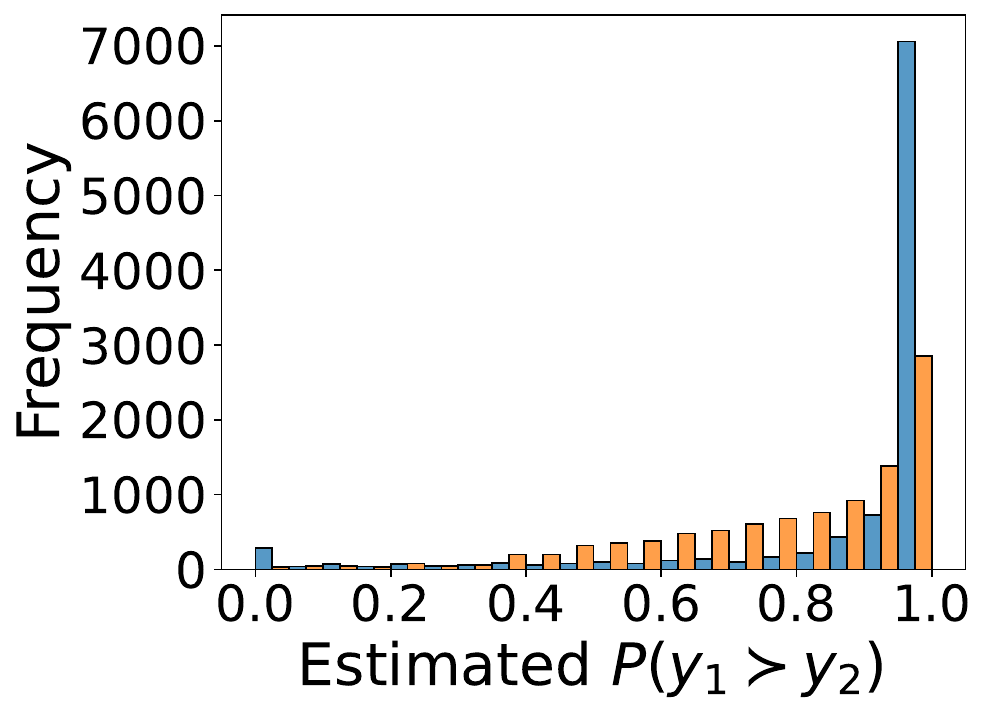}
        \caption{Held-out Clear Preferences Pairs (CP)}
        \label{fig:histogram-pref-prob-dpo}
    \end{subfigure}
    \hfill
    \begin{subfigure}[b]{0.48\textwidth}
        \includegraphics[width=\textwidth]{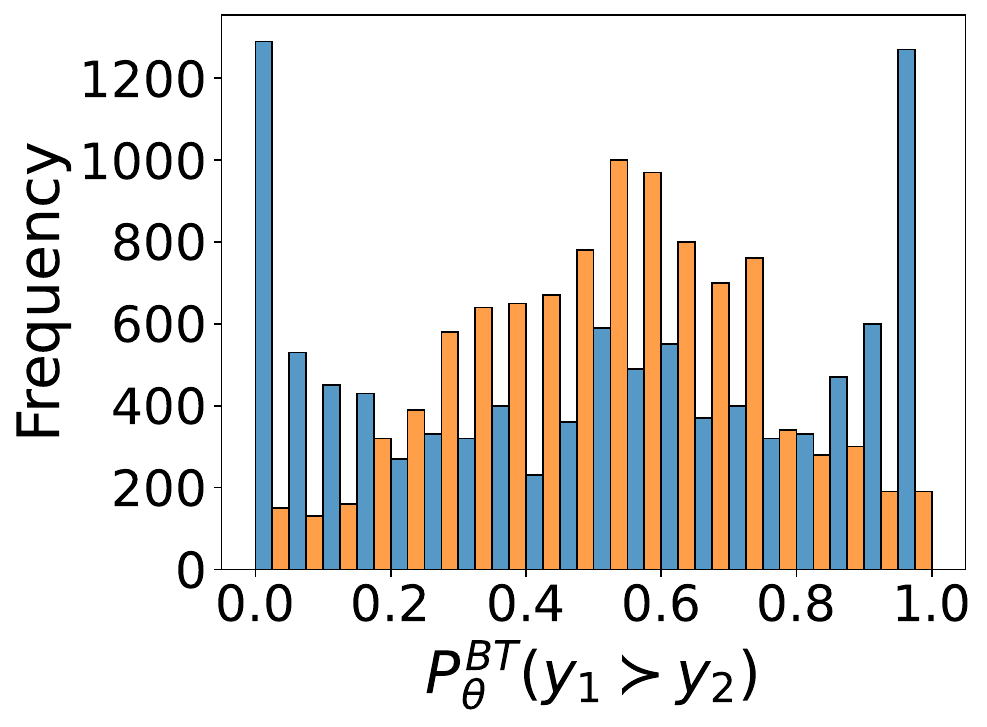}
        \caption{Held-out Tied Pairs (TP)}
        \label{fig:histogram-pref-prob-rk}
    \end{subfigure}
    \caption{Empirical distribution of preference probabilities under the Bradley-Terry model using the implicit reward function from the trained DPO policy on heldout CPs and TPs. DPO(CP) in \textcolor{customblue}{blue}, and DPO(CP+TP) in \textcolor{orange}{orange}.}
    \label{fig:DPO_PrefProbabilities_on_heldout_cps_and_tps}
\end{figure}

% \begin{figure}[!ht]
%  \begin{center} 
%     \includegraphics[width=0.3\textwidth]{figures/histogram_pref_prob_on_held_out_ties.pdf}
%      \caption{Empirical distribution of tied probabilities on tied pairs. DPO(CP) in \textcolor{customblue}{blue}, and DPO(CP+TP) in \textcolor{orange}{orange}. See Appendix \ref{app:Reward Margin Distributions on Held-out Clear Preference Pairs} for an analysis of CPs. 
%      %DPO-RK use the RK tie equation \ref{TODO} with $\alpha=\ln3$ and DPO-D with equation \ref{TODO} with $\lambda=1$. Maximal tie probability is $0.5$.
%      }
%     \label{fig:histogram of tie reward margins_1}
%  \end{center}
% \end{figure}

This can be explained by Figure~\ref{fig:DPO_PrefProbabilities_on_heldout_cps_and_tps} which shows that DPO(CP) models overwhelmingly assign preference probability values of either $\sim 1.0$ or $\sim 0.0$ to tied pairs, corresponding to very positive and very negative reward margins, respectively. This contributes to the high standard deviation and shows that for a tied pair $(y_1, y_2)$, DPO(CP) model exhibits a strong preference for either $y_1 \succ y_2$ or $y_2 \succ y_1$, even though these are tied pairs by construction ($y_1 \sim y_2$). In contrast, DPO(CP+TP) yields well-behaved estimated preference probability distribution more centered around 0.5 for tied pairs.

The DPO(CP) model correctly assigns high preference probability to most of the held-out CPs. This is consistent with its high classification accuracy on clear preference pairs in Table \ref{tab:RK-based and D-based accuracy for DPO, DPO-RK, DPO-D}. Similar to the estimated preference probability on held-out TPs, the DPO(CP) model tends to give confident, clear preference judgment with $>0.8$ probability in either direction. In comparison, the DPO(CP+TP) model is more conservative in making preference judgments, showing a less-sharp preference probability distribution over the held-out CP pairs. These results suggest that incorporating ties in DPO training leads to preference probability distributions that more evenly spread on both CPs and TPs as opposed to one concentrated on the two ends.

\begin{figure}[!ht]
    \centering
    \begin{subfigure}{0.49\textwidth}
        \includegraphics[width=1.0\textwidth]{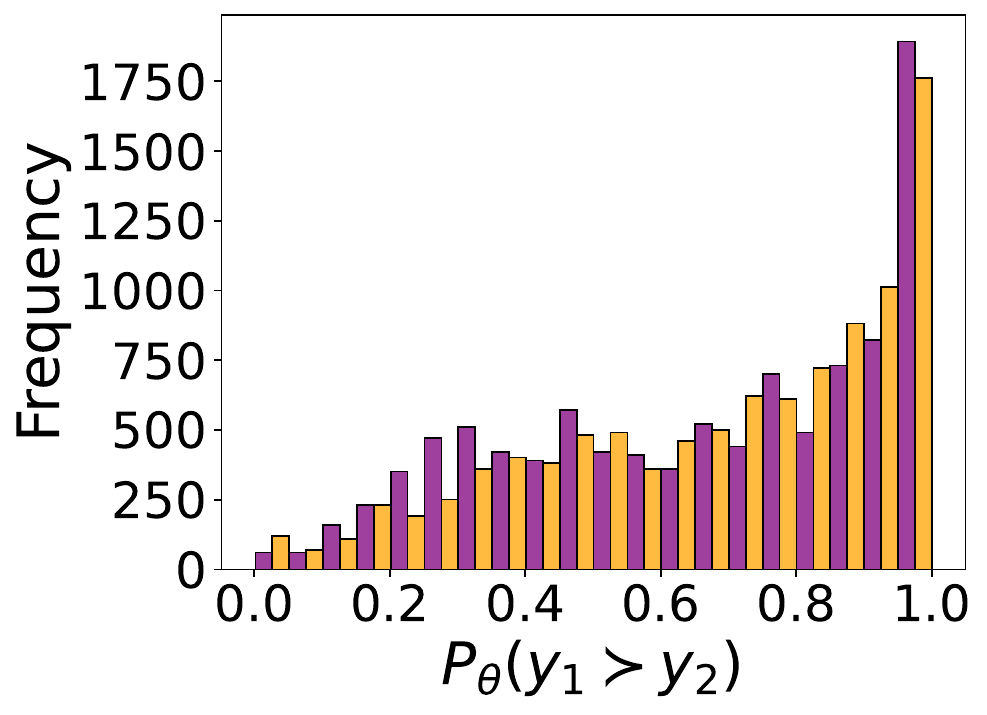}
        \caption{Preference probability under the models on held-out clear preference pairs.}
        % \label{fig:WMT21 - DPO(CP) vs DPO(CP+TP) KL-Performance frontier}
    \end{subfigure}
    \hfill
    \begin{subfigure}{0.49\textwidth}
        \centering
        \includegraphics[width=1.0\textwidth]{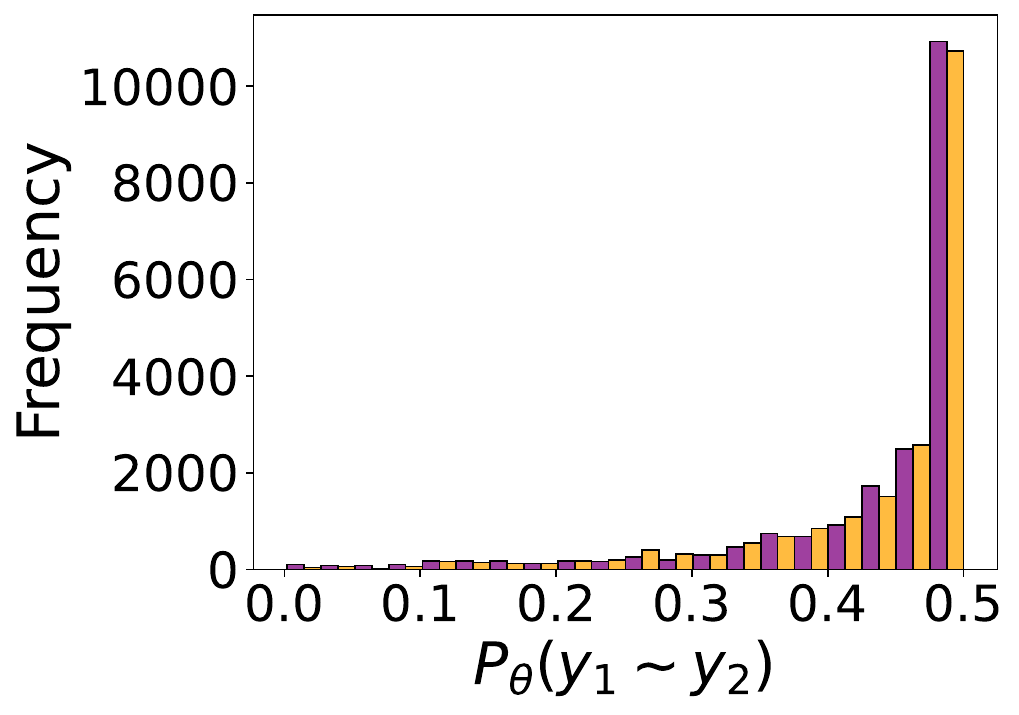}
        \caption{Tie probability under the models on held-out tied pairs.}
        % \label{fig:IWSLT17 - DPO(CP) vs DPO(CP+TP) KL-Performance frontier}
    \end{subfigure}
    \caption{DPO-D (\textcolor{orange}{orange}) and DPO-RK (\textcolor{custompurple}{purple}) preference/tie probability on held-out sets under the Davidson and Rao-Kupper models, respectively.}
    \label{fig:rkd-model-preference-tie-probability-histogram}
\end{figure}

For completeness, we also show the clear preference/tie probability distributions produced by models trained with DPO-RK(CP+TP) and DPO-D(CP+TP) on held-out clear preference pairs and tied pairs. Figure \ref{fig:rkd-model-preference-tie-probability-histogram} show that these distributions are well-behaved in that most of the probability mass are allocated to $P_\theta(y_1\succ y_2)>0.5$ on held-out clear preference pairs and to $P_\theta(y_1\sim y_2)\approx 0.5$ on held-out tied pairs. We note that under our hyper-parameter setting for the Rao-Kupper and Davidson models, the maximal tie probability is 0.5.

All models in this analysis are trained with $\beta=0.1$.

\newpage

\section{Training Dynamics and Convergence Behaviour}
\label{app:convergence behaviour}

We analyse how the inclusion of tied pairs affects the training dynamics of DPO, DPO-RK, and DPO-D. We study the BLOOMZ-mt-7b datasets with $\beta=0.7$ for WMT21 ZH-EN as these systems show both strong regularization effects and task performance degradation when tied pairs are added. Figure \ref{fig:WMT21 training statistics} shows the evolution of reward margins and gradient scale factors (Eqs.\ref{eq:dpo_grad} - \ref{eq:dgrad_tie}). 
% For DPO(CP+TP), we show these statistics on CPs and TPs separately for comparison. 

\begin{figure}[!htbp]
    \centering
    \begin{subfigure}[b]{0.49\textwidth}
        \includegraphics[width=1.0\textwidth]{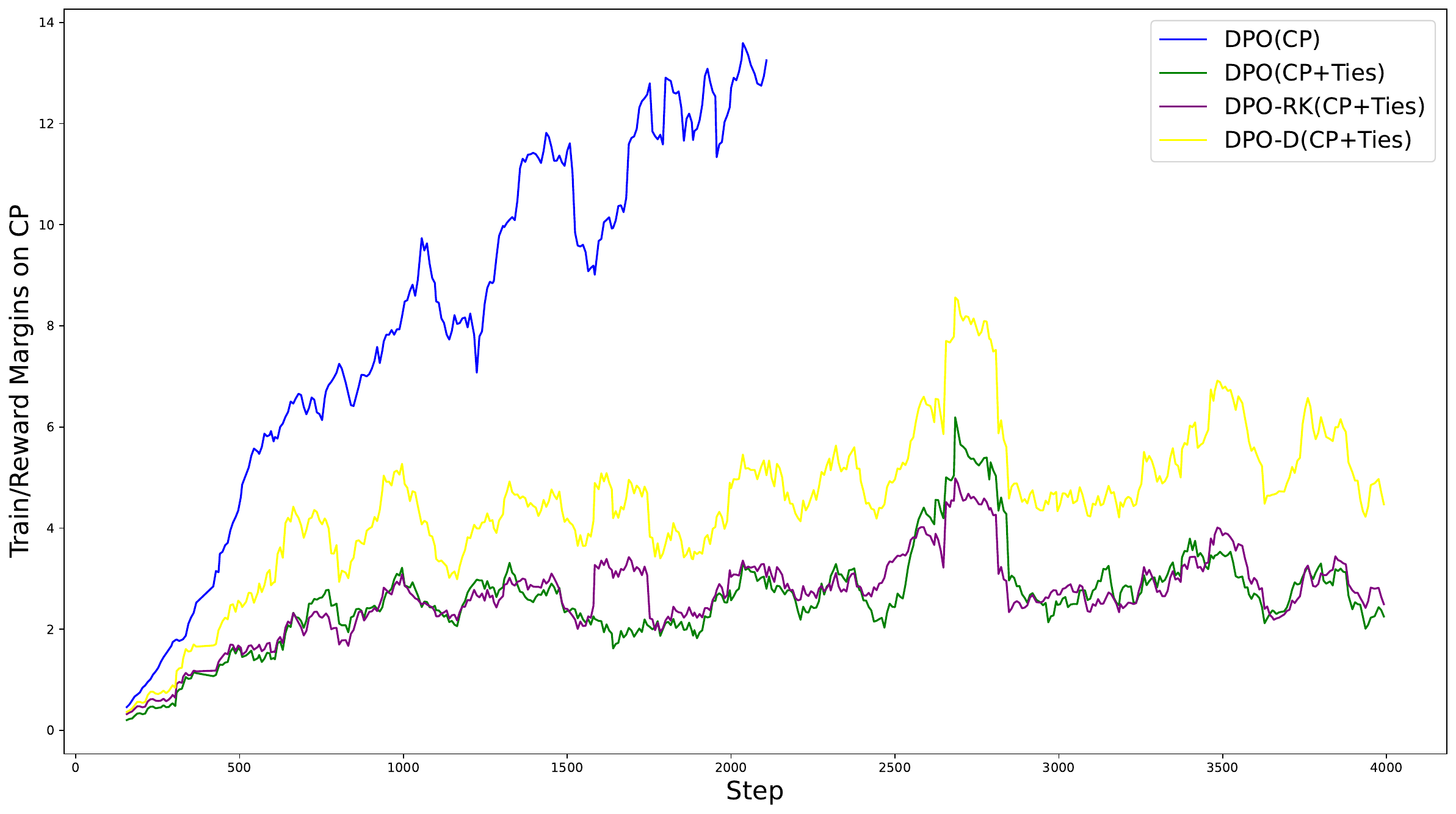}
        \caption{Reward margins on CP}
        \label{fig:all-train-reward-margins-on-cp}
    \end{subfigure}
    \hfill
    \begin{subfigure}[b]{0.49\textwidth}
        \centering
        \includegraphics[width=1.0\textwidth]{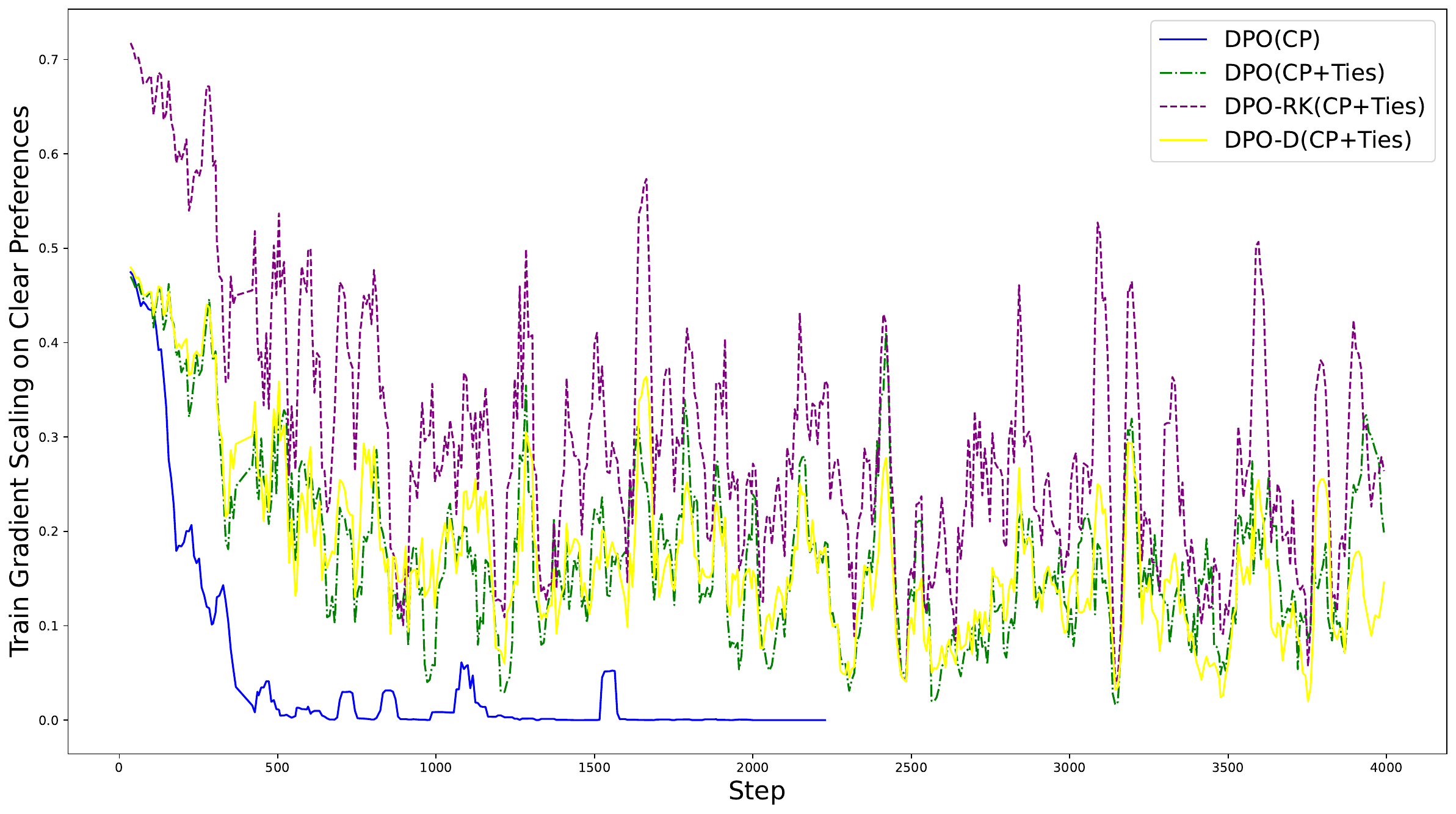}
        \caption{Gradient scaling on CP}
        \label{fig:all-train-gradient-scaling-on-cp}
    \end{subfigure}
    \begin{subfigure}[b]{0.49\textwidth}
        \centering
        \includegraphics[width=1.0\textwidth]{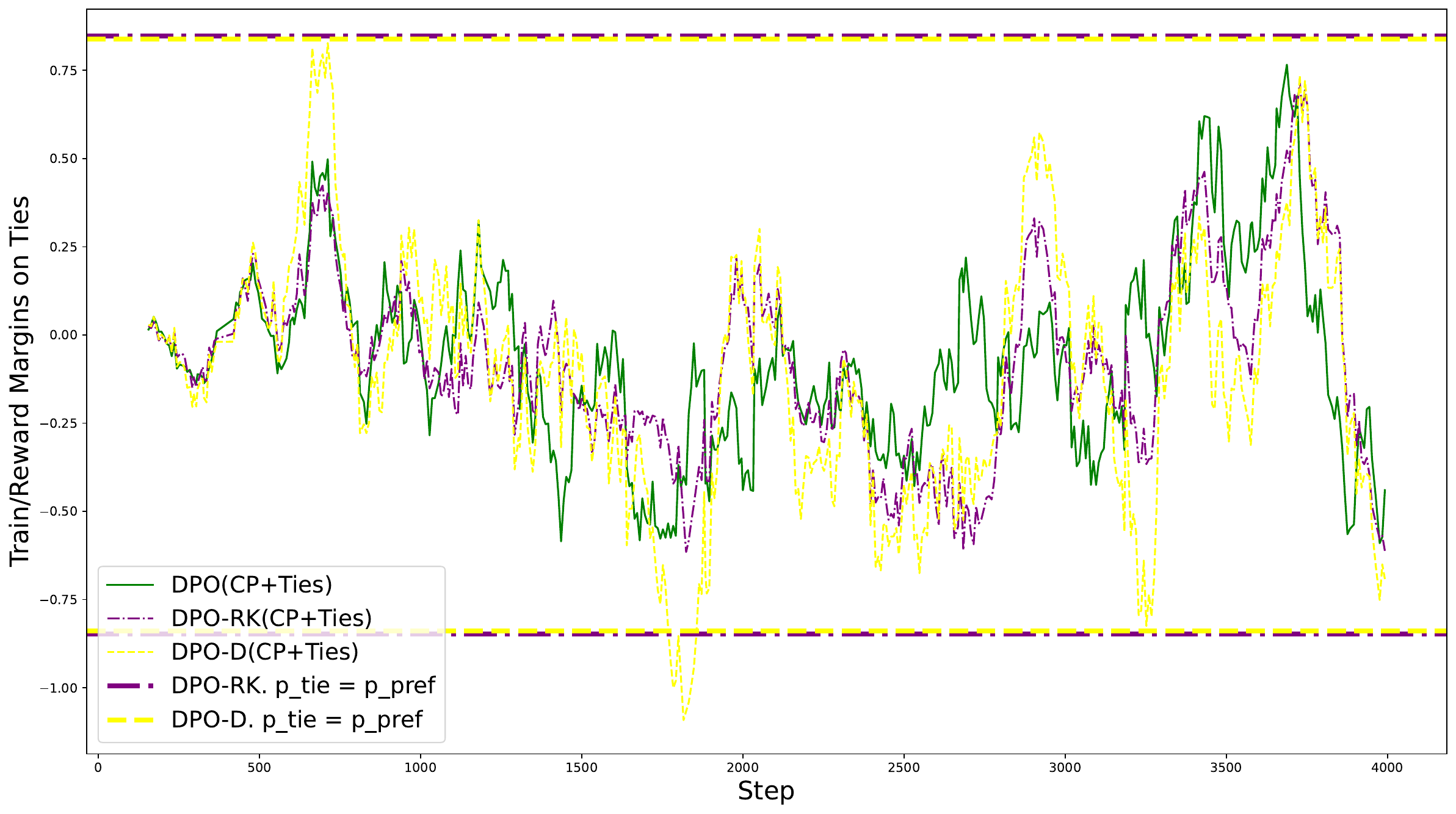}
        \caption{Reward margins on TP}
        \label{fig:all-train-reward-margins-on-ties}
    \end{subfigure}
    \begin{subfigure}[b]{0.49\textwidth}
        \centering
        \includegraphics[width=1.0\textwidth]{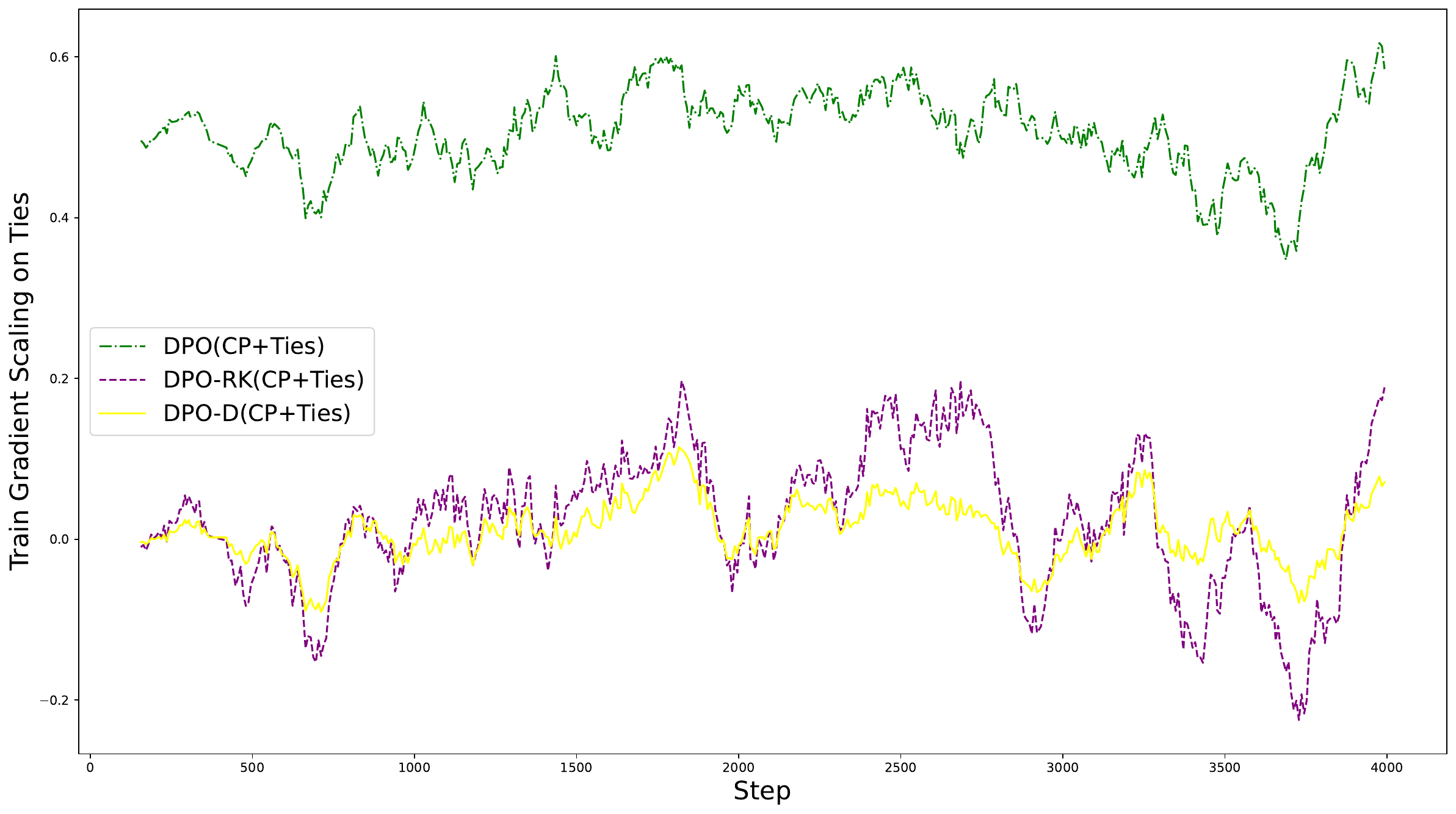}
        \caption{Gradient scaling on TP}
        \label{fig:all-train-reward-margins-on-ties}
    \end{subfigure}
    
    \caption{DPO(CP) (\textcolor{customblue}{blue}), DPO(CP+TP) (\textcolor{customgreen}{green}), DPO-RK(CP+TP) (\textcolor{purple}{purple}), and DPO-D(CP+TP) (\textcolor{yellow}{yellow}) training statistics on WMT21 ZH-EN.}
    \label{fig:WMT21 training statistics}
\end{figure}

DPO(CP) is well behaved: the reward margins on the CP data  increase over the epoch;   the DPO losses on the CP dataset decrease over the epoch;  
and the DPO gradient scale factor shows that learning slows and stabilizes after  the $500^{th}$ batch.   

Adding tied pairs to the DPO dataset alters this behaviour for both tied pairs and clear preference pairs.   DPO(CP+TP) does yield some gains in reward margins for clear preference pairs, but these are well below that of DPO(CP).  By contrast, DPO(CP+TP) fails almost entirely to find any improvement in the reward margins for its tied pair data.  While this is less than ideal from a modelling perspective, we note that it provides empirical support for the observation in the previous section that the reward margins on tied pairs should ideally remain close to zero.  Similar behaviour is observed in the DPO loss.  Decreases in loss over clear preference pairs are offset by loss increases on the tied pairs.   This is reflected in the gradient scale factors.  The gradient scale factors remain high as DPO(CP+TP) searches for a better policy. 

DPO-RK(CP+TP) and DPO-D(CP+TP) exhibit largely the same dynamics as DPO(CP+TP) except that the gradient scale factors on TPs have mean close to 0 instead of 0.5 and fluctuate between positive and negative values. This showcases that DPO-RK and DPO-D work to drive reward margin on TPs close to zero by possibly reversing optimization direction, a mechanism not present in the original DPO. 

%\newpage
\section{Experimental Details and Full Results}
\label{app:experimental details and full results}

% \subsection{Detailed Experimental Setup}
% \label{app:Detailed Experimental Setup}

We provide additional details of our experiments on Neural Machine Translation and Summarization with respect to the SFT models, the training configurations, and the decoding procedures. All experiments are run with the random seed set to 0. Codes are available at \url{https://github.com/EriChen0615/DPO-RKD}.

\subsection{Neural Machine Translation}
\label{ss:nmtedetails}
We largely follow \citet{bruce-dpo-mbr} in our experimental setup for NMT where the preference dataset is obtained via sampling and BLEURT-based ranking as explained in Sec.\ref{sec:adding ties}.

\paragraph{CPs and TPs in NMT.} We use DPO to improve translation quality similar to that done in \citet{bruce-dpo-mbr}.  
We apply DPO with BLOOMZ-mt-7b~\citep{Bloomz-model-paper} as the baseline model.  Translation quality is measured with BLEURT~\citep{BLEURT-paper} and COMET~\citep{comet-paper} on the WMT21 ZH-EN and IWSLT17 FR-EN  translation test sets.
We note that the WMT-23 metric overview paper reports high correlations (0.776 and 0.779, Table 1, \citet{WMT23_Metric_paper}) between BLEURT and COMET and human judgment. 
To construct a DPO preference dataset for the WMT21 ZH-EN test set,   we use BLOOMZ-mt-7b to generate 32 translations (via sampling) for each source sentence in the WMT20 ZH-EN test set. For each source sentence, the translations are ranked by their BLEURT scores computed with respect to the reference translations.  The highest and lowest scoring translations form the Clear Preference Pairs; for each source sentence, these are the two translations with the greatest difference in BLEURT score.  By contrast, we take the Tied Pairs as the two non-identical translations with the minimum absolute BLEURT difference; the translation with higher BLEURT is labeled as the winner of each Tied Pair. This yields ca. 16K CPs and TPs for use in DPO.   The same procedure is applied to the IWSLT17 validation set, yielding ca. 800 CPs and TPs for use as DPO preference datasets. We validate the selected TPs using GPT-4 in Appendix \ref{app:Using GPT-4 to validate ties on NMT}.

\paragraph{SFT Models} On WMT-21 ZH-EN, we performed supervised fine-tuning on the BLOOMZ-mt-7b~\cite{Bloomz-model-paper} using previous WMT test sets to obtain the SFT model from which we train with DPO/DPO-RK/DPO-D. The clear preference pairs and tied pairs are generated by sampling from this SFT model. On IWSLT-17 FR-EN, we use the pretrained BLOOMZ-mt-7b model directly in sampling clear preferences and tied pairs and in DPO fine-tuning, as we find further SFT leads to repetitive generation. 

\paragraph{Training Details} We use the RMSProp optimizer with the learning rate set to $5e^{-7}$ and the number of warm-up steps set to 150. All NMT experiments are run on two Nvidia A100-80G GPUs with an effective batch size of 4. We used FP32 for training the policy. The log-probabilities from the reference model are pre-computed with FP32 precision. Each training run  takes $\approx 2$ hours on WMT20 ZH-EN CP+TP data and $\approx 1$ hour on IWSLT17 FR-EN data.

\paragraph{Decoding} Following \citet{bruce-dpo-mbr}, we use beam search with a beam size $=4$ to decode all models.

\paragraph{Held-out Clear Preference Pairs and Tied Pairs} We curate held-out sets by generating translations by sampling on the WMT18 ZH-EN test set. Clear Preference Pairs and Tied Pairs are identified using their rankings under BLEURT exactly as done for WMT21 ZH-EN. This gives 3980 CPs and 3980 TPs for held-out evaluation in Appendix \ref{app:Using DPO-RK and DPO-D model as Preference Classifier}.

\subsection{Experiments with Stronger Translation Systems}
\label{app:ALMA_experiments}

We conducted additional experiments, replacing the BLOOMZ-mt-7B model with the ALMA-7B-LoRA model reported in \cite{ALMA-R_paper}, which has better translation performance.

We repeated the procedure described in Section~3.1 to collect 3000 Clear Preferences (CPs) and 3000 Tied Pairs (TPs) for ZH-EN translation. This was done by sampling candidate translations using ALMA-7B-LoRA on the ZH-EN training set of \cite{ALMA-R_paper}, followed by pair selection based on scores from automatic metrics. Following \cite{ALMA-R_paper}, we used the average score of KIWI-XXL and XCOMET to rank responses. We evaluated our models on the ZH-EN test set of \cite{ALMA-R_paper}, which consists of 2009 parallel sentences selected from FLORES-200. We report COMET, KIWI-XXL, and XCOMET as evaluation metrics. We also report KL divergence with respect to the reference model.

The table below shows the best-performing system evaluated on ZH-EN translation obtained from a $\beta$ sweep in $[0.1, 0.3, 0.5, 0.7, 0.9]$ under each configuration. For reference, we provide the performance of the base ALMA-7B-LoRA model as reported in \cite{ALMA-R_paper}. Additionally, ALMA-7B-R represents the best 7B model reported in \cite{ALMA-R_paper}, trained with Contrastive Preference Optimization (CPO) on a larger dataset than the one we used. While it is not directly comparable, it serves as a strong translation baseline.

\begin{table*}[!ht]
    \centering
    \begin{tabular}{lccccc}
        \hline
        \textbf{Model} & \textbf{KL} & \textbf{COMET} & \textbf{KIWI-XXL} & \textbf{XCOMET} & \textbf{Mean(X+K)} \\
        \hline
        ALMA-7B-LoRA \cite{ALMA-R_paper} & - & 79.8 & 73.7 & 83.9 & 78.8 \\
        ALMA-7B-R \cite{ALMA-R_paper}    & - & 81.0 & 75.7 & 90.7 & 83.2 \\
        \hline
        \textbf{Our Systems} & & & & & \\
        DPO(CP)              & 101.64 & 80.5 & 75.2 & 91.1 & 83.2 \\
        DPO(CP+TP)           & 60.37  & 80.2 & 75.0 & 90.8 & 82.9 \\
        DPO-RK(CP+TP)        & 75.57  & 80.2 & 75.0 & 90.8 & 82.9 \\
        DPO-D(CP+TP)         & 30.92  & 80.5 & 75.9 & 90.8 & 83.4 \\
        \hline
    \end{tabular}
    \caption{Performance metrics for ZH-EN translation systems. Results include KL divergence, COMET, KIWI-XXL, and XCOMET scores, along with the mean of XCOMET and KIWI-XXL. ALMA-7B-LoRA and ALMA-7B-R results are from \cite{ALMA-R_paper}.}
    \label{tab:stronger-translation-systems-app}
\end{table*}

We achieve the best performance under the metrics with DPO-D(CP+TP), which also has the lowest KL divergence, on top of a very strong baseline. Consistent with all our other results, including Tied Pairs in all training procedures yields a strong regularization effect.

\subsection{Summarization}

\paragraph{CPs and TPs in Summarization.} 
We follow \citet{ODPO-Direct-Preference-Optimization-with-an-offset-Amini} in DPO fine-tuning of Pythia-2.8B~\citep{Pythia-model-paper} on the TL;DR dataset~\citep{TLDR-dataset-OpenAI} with evaluation via win-rate against human-written summaries. Previous works use GPT-4 to compute the win-rate~\citep{DPO-Rafael, ODPO-DirectPreferenceOptimizationWithAnOffset}. We find that the judgments of PairRM~\citep{PairRM-paper} agree well with those of GPT-4 (Appendix \ref{app:PairRM as a proxy of GPT-4}) and opt to use PairRM win-rate as a cost-effective automatic metric. 
% use win-rate as judged by PairRM~\citep{PairRM-paper} as a cost-effective automatic metric and find that it agrees well with GPT-4 win-rate on TL;DR (Appendix \ref{app:PairRM as a proxy of GPT-4}).
% We use PairRM win-rate as an automatic metric and find that it agrees well with GPT-4 win-rate on TL;DR (Appendix \ref{app:PairRM as a proxy of GPT-4}).
In the TL;DR task, each prompt is associated with a collection of paired summaries, with a winner and a loser identified for each pair.  There is no immediately obvious way to distinguish TPs from CPs in the collection and so we use DPO itself to select TPs.   We first apply DPO with $\beta=0.1$ on the full TL;DR training dataset. Using the reward model formed by this model and the reference model, we compute the reward margins of all pairs of summaries in the training split. 
For each prompt, the pair with minimal reward margin is treated as a TP, with all other pairs kept as CPs, yielding ca. 14k (15.3\%) TPs. 
See Appendix~\ref{app:Verifying tied pair selection strategy} for a study of this selection strategy.

We follow \citet{ODPO-Direct-Preference-Optimization-with-an-offset-Amini} in experimental setups. The preference dataset is obtained via sampling and ranking with a DPO model without requiring an external reward model as explained in Sec.\ref{sec:adding ties}. 

\paragraph{SFT Model} We follow \citet{ODPO-Direct-Preference-Optimization-with-an-offset-Amini} to supervise-finetune a Pythia-2.8B model~\cite{Pythia-model-paper} on the chosen responses in TL;DR train split for one epoch to obtain the initial checkpoint for preference learning. We use the summarization prompt provided in Appendix D.2 by \citet{DPO-Rafael}.

\paragraph{Training Details} We use the RMSProp optimizer with the learning rate set to $5e^{-7}$ and the number of warm-up steps set to 150. All summarization experiments are run on two Nvidia A100-40G GPUs with an effective batch size of 64. We used FP32 for the policy and FP16 for the reference model. Each training run  takes $\approx 7$ hours on TL;DR CP+TP data.

\paragraph{Decoding} We use greedy decoding for all models as we find it performs on-par or better than temperature sampling (Appendix \ref{app:PairRM as a proxy of GPT-4}). 

\subsection{Tabulated KL-Performance Results on NMT and Summarization}
\label{app:Tabulated KL-Performance Results on NMT and Summarization}

We tabulate the KL-Performance results shown in Figure \ref{fig:DPO(CP) vs DPO(CP+TP) KL-Performance frontier} and Figure \ref{fig:rkdavidon-pareto}.

\subsubsection{Neural Machine Translation}
\label{sssec:nmtfullresults}

\begin{table*}[!ht]
\centering
\begin{tabular}{@{}lrrrrrr@{}} \toprule
System & beta & KL Divergence & BLEU & Length Ratio & COMET & BLEURT \\ \midrule
Bloomz-mt-7b1-SFT & - & 0 & 17.6 &  & 77.9 & 61.6 \\
% DPO~\citep{bruce-dpo-mbr} & 0.7 &  &  &  & 80.8 & 66 \\ 
\midrule 
DPO(CP) & 0.1 & 174.13 & 7.23 & 3.01 & 70.2 & 67.7 \\
DPO(CP) & 0.2 & 68.12 & 20.8 & 1.10 & 80.8 & 66.2 \\
DPO(CP) & 0.3 & 62.85 & 20.7 & 1.13 & 80.6 & 66.4 \\
DPO(CP) & 0.4 & 56.02 & 21.4 & 1.09 & 80.7 & 66.4 \\
DPO(CP) & 0.5 & 50.99 & 21.2 & 1.11 & 80.8 & 66.5 \\
DPO(CP) & 0.6 & 47.97 & 21.5 & 1.09 & 80.9 & 66.5 \\
DPO(CP) & 0.7 & 44.08 & 21.5 & 1.11 & 81.0 & 66.7 \\
DPO(CP) & 0.8 & 41.88 & 21.3 & 1.14 & 80.8 & 66.7 \\
DPO(CP) & 0.9 & 41.24 & 21.5 & 1.14 & 80.8 & 66.8 \\
DPO(CP) & 1.9 & 33.69 & 22.3 & 1.09 & 81.2 & 67.0 \\
DPO(CP) & 1.2 & 32.01 & 22.4 & 1.09 & 81.3 & 67.1 \\
DPO(CP) & 1.5 & 29.58 & 21.7 & 1.13 & 81.1 & 67.1 \\
DPO(CP) & 1.55 & 29.01 & 21.9 & 1.13 & 81.1 & 67.1 \\ \midrule
DPO(CP+TP) & 0.1 & 51.29 & 20.3 & 1.16 & 80.0 & 66.0 \\
DPO(CP+TP) & 0.2 & 36.37 & 18.8 & 1.30 & 80.1 & 66.6 \\
DPO(CP+TP) & 0.3 & 26.15 & 19.5 & 1.24 & 80.2 & 66.6 \\
DPO(CP+TP) & 0.4 & 18.21 & 20.6 & 1.20 & 80.4 & 66.6 \\
DPO(CP+TP) & 0.5 & 15.47 & 21.2 & 1.15 & 80.4 & 66.4 \\
DPO(CP+TP) & 0.6 & 14.74 & 21.9 & 1.10 & 80.6 & 66.4 \\
DPO(CP+TP) & 0.7 & 13.29 & 22.1 & 1.11 & 80.5 & 66.4 \\
DPO(CP+TP) & 0.8 & 12.57 & 22.2 & 1.10 & 80.5 & 66.2 \\
DPO(CP+TP) & 0.9 & 12.10 & 21.9 & 1.10 & 80.5 & 66.3 \\
DPO(CP+TP) & 1.0 & 11.43 & 22.0 & 1.11 & 80.5 & 66.2 \\ \midrule
DPO-RK(CP+TP) & 0.1 & 48.55 & 19.3 & 1.22 & 80.2 & 66.9 \\
DPO-RK(CP+TP) & 0.2 & 28.61 & 22.1 & 1.11 & 80.9 & 66.9 \\
DPO-RK(CP+TP) & 0.3 & 20.21 & 22.5 & 1.11 & 81.0 & 67.1 \\
DPO-RK(CP+TP) & 0.4 & 14.80 & 22.4 & 1.12 & 81.1 & 67.1 \\
DPO-RK(CP+TP) & 0.5 & 11.66 & 22.8 & 1.10 & 81.0 & 67.1 \\
DPO-RK(CP+TP) & 0.6 & 9.74 & 22.2 & 1.13 & 80.8 & 66.8 \\
DPO-RK(CP+TP) & 0.7 & 8.04 & 22.3 & 1.12 & 80.8 & 66.7 \\
DPO-RK(CP+TP) & 0.8 & 8.10 & 22.1 & 1.13 & 80.8 & 66.8 \\
DPO-RK(CP+TP) & 0.9 & 7.58 & 21.8 & 1.15 & 80.7 & 66.8 \\
DPO-RK(CP+TP) & 1.0 & 6.31 & 22.3 & 1.11 & 80.7 & 66.6 \\ \midrule
DPO-D(CP+TP) & 0.2 & 42.74 & 21.4 & 1.13 & 80.8 & 66.6 \\
DPO-D(CP+TP) & 0.3 & 38.56 & 21.2 & 1.15 & 80.2 & 66.5 \\
DPO-D(CP+TP) & 0.4 & 17.01 & 22.5 & 1.11 & 81.0 & 67.1 \\
DPO-D(CP+TP) & 0.5 & 20.20 & 22.7 & 1.10 & 81.1 & 67.1 \\
DPO-D(CP+TP) & 0.6 & 26.85 & 22.3 & 1.10 & 81.1 & 66.9 \\
DPO-D(CP+TP) & 0.7 & 14.97 & 22.6 & 1.11 & 81.1 & 67.1 \\
DPO-D(CP+TP) & 0.8 & 13.33 & 22.7 & 1.11 & 81.1 & 67.1 \\
DPO-D(CP+TP) & 1.0 & 10.05 & 22.3 & 1.12 & 80.9 & 67.0 \\ \bottomrule \\
\end{tabular}
\caption{KL-Performance evaluated on WMT-21 ZH-EN.}
\label{tab:kl-performance-table-wmt21-zhen}
\end{table*}

In addition to KL Divergence and BLEURT, we also provide COMET~\citep{comet-paper} scores, BLEU~\citep{bleu_score} scores and BLEU's Length Ratio.

We observe the ``reward hacking'' phenomenon identified by \citet{bruce-dpo-mbr} on both WMT21 ZH-EN and IWSLT17 FR-EN where systems achieve good BLEURT but have large length ratio (>1.5) and lower COMET than the pre-DPO system. These systems learn to generate long, repetitive translations which BLEURT fails to recognize as low-quality. \citet{bruce-dpo-mbr} find that using small beta values (e.g. 0.1) in DPO training results in reward hacking models. Our results are consistent with their findings and further suggest that large KL divergence from the reference model is a good indicator for reward hacking. On WMT21 ZH-EN, the only model that exhibits reward hacking is trained by DPO(CP) with beta=0.1 which also yields the highest KL divergence (174.13) among all models, greatly exceeding the second-highest KL divergence (68.12). On IWSLT17 FR-EN, Almost all models with KL Divergence $> 30$ (DPO(CP), $\beta=0.1$, DPO-RK(CP+TP), $\beta=0.1$ and DPO-D(CP+TP) $\beta=0.1, 0.5$) show reward hacking behaviours. 

Reward hacking on NMT can be resolved by increasing regularization with respect to the reference model. We find that training with larger beta values or incorporating ties in DPO-RK/DPO-D training can provide such regularization without performance degradation. 

% \usepackage{booktabs}
%\newpage

\begin{table*}[!ht]
\begin{tabular}{lrrrrrr} \toprule
System & beta & KL Divergence & BLEU & Length Ratio & COMET & BLEURT \\ \midrule
Bloomz-mt-7b1 & - &  & 17.6 &  & 85.4 & 74.8 \\
% DPO~\citep{bruce-dpo-mbr} & 0.7 &  &  &  & 86.8 & 76.5 \\ 
\midrule
DPO(CP) & 0.1 & 53.60 & 25.8 & 1.40 & 82.3 & 74.7 \\
DPO(CP) & 0.3 & 30.80 & 23.7 & 1.60 & 83.6 & 76.5 \\
DPO(CP) & 0.5 & 16.70 & 36.8 & 1.00 & 86.1 & 76.2 \\
DPO(CP) & 0.7 & 13.80 & 38.5 & 1.00 & 86.4 & 76.4 \\
DPO(CP) & 1.0 & 12.40 & 38.6 & 1.00 & 86.5 & 76.5 \\
DPO(CP) & 1.2 & 11.80 & 38.8 & 0.98 & 86.5 & 76.5 \\
DPO(CP) & 1.5 & 10.70 & 38.9 & 0.99 & 86.5 & 76.5 \\ \midrule
DPO(CP+TP) & 0.1 & 35.60 & 35.8 & 1.00 & 85.6 & 75.5 \\
DPO(CP+TP) & 0.3 & 25.80 & 35.7 & 1.10 & 85.4 & 75.9 \\
DPO(CP+TP) & 0.5 & 22.00 & 35.1 & 1.10 & 85.8 & 76.3 \\
DPO(CP+TP) & 0.7 & 17.00 & 38.7 & 1.00 & 86.3 & 76.3 \\
DPO(CP+TP) & 1.0 & 11.50 & 38.9 & 1.00 & 86.4 & 76.4 \\
DPO(CP+TP) & 1.2 & 8.50 & 39.1 & 0.98 & 86.5 & 76.4 \\
DPO(CP+TP) & 1.5 & 6.30 & 39.0 & 0.98 & 86.4 & 76.3 \\ \midrule
DPO-RK(CP+TP) & 0.1 & 46.70 & 23.0 & 1.60 & 78.7 & 76.3 \\
DPO-RK(CP+TP) & 0.2 & 19.51 & 35.9 & 1.05 & 85.9 & 76.4 \\
DPO-RK(CP+TP) & 0.3 & 15.50 & 36.1 & 1.10 & 86.1 & 76.5 \\
DPO-RK(CP+TP) & 0.5 & 13.30 & 31.4 & 1.20 & 85.7 & 76.6 \\
DPO-RK(CP+TP) & 0.7 & 10.90 & 31.3 & 1.20 & 85.8 & 76.5 \\
DPO-RK(CP+TP) & 0.8 & 10.90 & 29.9 & 1.28 & 85.6 & 76.5 \\
DPO-RK(CP+TP) & 0.9 & 11.60 & 27.2 & 1.40 & 85.3 & 76.4 \\
DPO-RK(CP+TP) & 1.0 & 11.60 & 26.1 & 1.50 & 85.1 & 76.3 \\
DPO-RK(CP+TP) & 1.2 & 11.80 & 24.4 & 1.57 & 84.8 & 76.3 \\ \midrule
DPO-D(CP+TP) & 0.1 & 48.60 & 25.3 & 1.41 & 82.6 & 76.3 \\
DPO-D(CP+TP) & 0.3 & 19.90 & 35.4 & 1.07 & 85.8 & 76.5 \\
DPO-D(CP+TP) & 0.5 & 51.90 & 8.4 & 4.35 & 75.1 & 76.1 \\
DPO-D(CP+TP) & 0.7 & 12.80 & 36.6 & 1.06 & 86.2 & 76.6 \\
DPO-D(CP+TP) & 1.0 & 10.30 & 37.8 & 1.03 & 86.3 & 76.6 \\
DPO-D(CP+TP) & 1.2 & 10.90 & 32.1 & 1.20 & 85.9 & 76.6 \\ \bottomrule
\end{tabular}
\caption{KL-Performance evaluated on IWSLT17 FR-EN}
\label{tab:KL-Performance Table IWSLT17}
\end{table*}

\subsubsection{Summarization}
Table \ref{tab:KL-Performance table for TL;DR} shows the KL-PairRM winrate on TL;DR summarization.

\begin{table*}[!ht]
\centering
\begin{tabular}{lrrr} \toprule
System & beta & KL Divergence & PairRM Winrate \\
Pythia-2.8B-SFT, Greedy & - & 0.00 & 37.5 \\ \midrule
DPO(CP) & 0.025 & 97.03 & 67.9 \\
DPO(CP) & 0.05 & 60.31 & 70.3 \\
DPO(CP) & 0.07 & 57.14 & 71.5 \\
DPO(CP) & 0.08 & 38.16 & 66.4 \\ 
DPO(CP) & 0.10 & 26.82 & 62.5 \\
DPO(CP) & 0.30 & 9.97 & 63.7 \\
DPO(CP) & 0.50 & 5.79 & 59.0 \\
DPO(CP) & 0.70 & 3.78 & 57.8 \\
\midrule
DPO(CP+TP) & 0.025 & 87.66 & 63.7 \\
DPO(CP+TP) & 0.03 & 119.60 & 66.8 \\ 
DPO(CP+TP) & 0.04 & 70.69 & 69.5 \\
DPO(CP+TP) & 0.05 & 35.39 & 63.3 \\
DPO(CP+TP) & 0.10 & 17.21 & 57.4 \\
DPO(CP+TP) & 0.30 & 4.50 & 58.6 \\
DPO(CP+TP) & 0.50 & 7.61 & 57.8 \\
DPO(CP+TP) & 0.70 & 2.91 & 55.9 \\
\midrule
DPO-RK(CP+TP) & 0.04 & 80.86 & 65.2 \\
DPO-RK(CP+TP) & 0.05 & 62.57 & 67.2 \\
DPO-RK(CP+TP) & 0.10 & 40.50 & 67.6 \\
DPO-RK(CP+TP) & 0.20 & 22.24 & 67.6 \\
DPO-RK(CP+TP) & 0.30 & 12.45 & 68.0 \\
DPO-RK(CP+TP) & 0.50 & 6.15 & 65.6 \\
DPO-RK(CP+TP) & 0.70 & 4.33 & 61.7 \\
\midrule
DPO-D(CP+TP) & 0.05 & 82.35 & 64.8 \\
DPO-D(CP+TP) & 0.10 & 54.06 & 71.5 \\
DPO-D(CP+TP) & 0.20 & 39.23 & 66.0 \\
DPO-D(CP+TP) & 0.30 & 22.46 & 68.8 \\
DPO-D(CP+TP) & 0.40 & 12.57 & 67.6 \\
DPO-D(CP+TP) & 0.50 & 9.92 & 67.2 \\
DPO-D(CP+TP) & 0.70 & 6.82 & 64.8 \\ \bottomrule
\end{tabular}
\caption{KL-PairRM winrate against 256 human-written summaries on TL;DR summarization}
\label{tab:KL-Performance table for TL;DR}
\end{table*}

\clearpage

\subsection{Standard Deviation on GSM8K Mathematical Reasoning}
\label{app:GSM8K-variance}

\begin{table}[!ht]
\centering
\begin{tabular}{lccc}
\toprule
$\beta$ & DPO (CP) & DPO-RK (CP+TP) & DPO-D (CP+TP) \\
\midrule
0.1 & 80.4\% $\pm$ 1.34\% & 82.9\% $\pm$ 0.90\% & 82.2\% $\pm$ 1.95\% \\
0.3 & 83.7\% $\pm$ 0.45\% & 84.3\% $\pm$ 0.18\% & 83.7\% $\pm$ 0.68\% \\
0.5 & 83.5\% $\pm$ 0.14\% & 84.1\% $\pm$ 0.08\% & 84.2\% $\pm$ 0.15\% \\
0.7 & 83.6\% $\pm$ 0.53\% & 83.4\% $\pm$ 0.19\% & 84.5\% $\pm$ 0.58\% \\
1.0 & 83.8\% $\pm$ 0.29\% & 84.0\% $\pm$ 0.08\% & 83.7\% $\pm$ 0.42\% \\
\bottomrule
\end{tabular}
\caption{Median accuracy ($\pm$ standard deviation) on GSM8K test set across five random seeds for each $\beta$ value. Tie-compatible variants (DPO-RK, DPO-D) consistently outperform the baseline DPO.}
\label{tab:beta_variance_results}
\end{table}

We run experiments for five random seeds for each variant at each $\beta$ value (75 training runs in total). In Table~\ref{tab:beta_variance_results}, we report the median accuracy and standard deviation. Consistent with our previous finding, tie-compatible variants achieve higher median accuracy at all $\beta$ values. Gains are generally greater than one standard deviation.

\section{Additional Supporting Experiments and Analyses}
\label{app:Additional Supporting Experiments and Analyses}

\subsection{Sensitivity Study on the Value of $\alpha$ for DPO-RK and DPO-D}
\label{app:Ablation Study on the Value of alpha for DPO-RK and DPO-D}

In our early experiments with a weaker translation model, we find that the choice of $\alpha$ does not have a strong impact on final task performance. In Table \ref{tab:alpha-ablation-study}, we vary the hyperparameter $\alpha$ and report the KL divergence, KIWI-22, BLEURT, and COMET scores with the BLOOMZ-mt-7B setup evaluated on WMT-21 ZH-EN.

\begin{table}[!ht]
    \centering
    \begin{tabular}{c c c c c}
        \hline
        \textbf{$\alpha$} & \textbf{KL} & \textbf{KIWI-22} & \textbf{BLEURT} & \textbf{COMET} \\
        \hline
        \multicolumn{5}{c}{\textbf{DPO-RK}} \\
        \hline
        0.5 & 11.23 & 78.9 & 66.9 & 81.0 \\
        \underline{ln3} & 14.80 & 79.0 & 67.1 & 81.1 \\
        2.0 & 21.84 & 79.1 & 67.1 & 81.1 \\
        3.0 & 30.41 & 79.0 & 66.9 & 80.9 \\
        5.0 & 76.66 & 69.9 & 68.9 & 74.6 \\
        \hline
        \multicolumn{5}{c}{\textbf{DPO-D}} \\
        \hline
        -2.5 & 11.56 & 79.0 & 67.0 & 81.0 \\
        -1.0 & 15.25 & 79.0 & 67.1 & 81.1 \\
        \underline{0.0} & 11.66 & 79.2 & 67.1 & 81.0 \\
        1.0 & 29.72 & 79.0 & 67.0 & 81.0 \\
        2.5 & 47.20 & 79.2 & 67.1 & 81.1 \\
        \hline
    \end{tabular}
    \caption{Effect of varying $\alpha$ on DPO-RK ($\beta=0.4$) and DPO-D ($\beta=0.5$) systems. For DPO-RK, $\alpha=\alpha_{RK}$ in Eq.\ref{eq:dpork_loss}; For DPO-D, $\alpha=\ln \nu_D$ in Eq.\ref{eq:dpod_loss}. Results are reported for KL divergence, KIWI-22, BLEURT, and COMET metrics. Default values of $\alpha$ are \underline{underlined}.}
    \label{tab:alpha-ablation-study}
\end{table}

We find that (1) The final task performance is stable over a range of $\alpha$ values (note that $\alpha < 0$ is illegal for DPO-RK) and (2) for large values of $\alpha$, there is relatively little regularization as measured by KL divergence.

As explained in Sec.\ref{sec:DPO-RK and DPO-D Updates}, $\alpha$ controls the width of the band in reward margins where there's little or no gradient contributions from the tied pairs whose difference in reward falls within the band. This explains the lack of regularization under large $\alpha$. Given large $\alpha$, the gradient on TPs approach zero regardless of the reward margin of the pair under the current model. As we show in the paper, tied pairs contribute to a regularization effect. This explains why large and hence little gradient contribution from TPs leads to relatively little regularization.

\subsection{Grid search of $\beta$ and $\alpha$ on WMT}

We further conduct a grid search over $\beta$ and $\alpha$ values on WMT to investigate their effects. The results are shown in Table~\ref{tab:wmt_grid_results}.

\begin{table*}[!ht]
\centering
\begin{tabular}{llrrr}
\toprule
Variant & $\beta$ & $\alpha$ & KL & COMET \\
\midrule
\multirow{5}{*}{DPO-RK (CP+TP)} 
 &  & 0.1 & 19.73 & 80.9 \\
 &      & 0.5 & 22.31 & 81.0 \\
 & 0.2     & $\ln 3$ & 28.61 & 80.9 \\
 &      & 2.0 & 37.05 & 81.1 \\
 &      & 3.0 & 49.00 & 81.0 \\[3pt]
\multirow{5}{*}{DPO-RK (CP+TP)} 
 &  & 0.1 & 9.54 & 80.9 \\
 &      & 0.5 & 14.80 & 81.0 \\
 & 0.4     & $\ln 3$ & 21.84 & 81.1 \\
 &      & 2.0 & 30.41 & 81.1 \\
 &      & 3.0 & 76.66 & 74.6 \\[3pt]
\multirow{5}{*}{DPO-RK (CP+TP)} 
 & & 0.1 & 6.32 & 80.7 \\
 &      & 0.5 & 5.91 & 80.6 \\
 & 0.8      & $\ln 3$ & 8.10 & 80.8 \\
 &      & 2.0 & 10.37 & 80.9 \\
 &      & 3.0 & 14.53 & 80.9 \\[6pt]
\midrule
\multirow{5}{*}{DPO-D (CP+TP)} 
 &  & $-2.5$ & 27.46 & 80.9 \\
 &     & $-1.0$ & 34.80 & 80.6 \\
 &  0.2     & 0 & 42.74 & 80.8 \\
 &      & 1.0 & 56.83 & 80.6 \\[3pt]
\multirow{5}{*}{DPO-D (CP+TP)} 
 &  & $-2.5$ & 14.56 & 81.0 \\
 &      & $-1.0$ & 18.71 & 81.0 \\
 & 0.4     & 0 & 17.01 & 81.0 \\
 &      & 1.0 & 33.15 & 81.0 \\
 &      & 2.5 & 56.18 & 80.7 \\[3pt]
\multirow{5}{*}{DPO-D (CP+TP)} 
 &  & $-2.5$ & 11.56 & 81.0 \\
 &      & $-1.0$ & 15.25 & 81.1 \\
 &  0.5    & 0 & 11.66 & 81.0 \\
 &      & 1.0 & 29.72 & 81.0 \\
 &      & 2.5 & 47.20 & 81.1 \\[3pt]
\multirow{5}{*}{DPO-D (CP+TP)} 
 &  & $-2.5$ & 8.16 & 80.7 \\
 &      & $-1.0$ & 9.10 & 80.7 \\
 & 0.8     & 0 & 13.33 & 81.1 \\
 &      & 1.0 & 19.54 & 81.0 \\
 &      & 2.5 & 30.48 & 80.8 \\
\bottomrule
\end{tabular}
\caption{Grid search results for $\beta$ and $\alpha$ on WMT showing the corresponding KL divergence and COMET scores.}
\label{tab:wmt_grid_results}
\end{table*}

Consistent with our findings in Appendix~\ref{app:Ablation Study on the Value of alpha for DPO-RK and DPO-D}, we find that task performance is stable over a range of $\beta$ and $\alpha$ values. These results add further empirical evidence supporting the role of $\beta$ and $\alpha$ as knobs for regularization strength. Controlled for a fixed value of $\beta$, increasing $\alpha$ and $\beta$ generally leads to higher KL divergence with respect to the reference model at the end of training. 

Based on these experiments, we recommend the following hyperparameter tuning scheme for DPO-RK and DPO-D:
\begin{enumerate}
    \item Run training across a range of $\beta$ values with the default $\alpha$ to identify the effective range of $\beta$.
    \item Adjust $\alpha$ (and $\beta$, if desired) at those values, using the final KL divergence as a guide. 
\end{enumerate}

As shown in experiments, there is typically a ``sweet range'' of KL where performance is optimal. This procedure allows practitioners to efficiently identify the best performance–KL trade-off for a given setup.

\subsection{GPT-4 Evaluation on Summarization}
\label{app:GPT-4 Evaluation on Summarization}

For TL;DR, we additionally report the win-rate as evaluated by GPT-4 of the best-performing systems in Table \ref{tab:gpt4-evaluation-summarization-app}.

\begin{table}[!ht]
    \centering
    \begin{tabular}{lcc}
        \hline
        \textbf{Model} & \textbf{PairRM} & \textbf{GPT-4} \\
        \hline
        Pythia-2.8-DPO(CP) & \textbf{71.5\%} & 62.1\% \\
        Pythia-2.8-DPO-D(CP+TP) & \textbf{71.5\% }& \textbf{64.1\%} \\
        Pythia-2.8-DPO(CP+TP) & 69.5\% & 62.1\% \\
        Pythia-2.8-DPO-RK(CP+TP) & 68.0\% & 57.8\% \\
        Pythia-2.8-SFT & 37.5\% & 36.7\% \\
        \hline
    \end{tabular}
    \caption{Win-rate comparison of the best-performing systems evaluated using PairRM and GPT-4.}
    \label{tab:gpt4-evaluation-summarization-app}
\end{table}

We note that the rank based on PairRM win-rate agrees with the rank produced by GPT-4 except that GPT-4 prefers Pythia-2.8-DPO-D(CP+TP).

\subsection{Varying the Percentage of Ties}
\label{app:varying the percentage of ties}

In Table \ref{tab:tp-proportion-regularization} we vary the proportion of TPs and evaluate on the WMT-21 ZH-EN dataset, keeping the CPs unchanged. We train systems with three $\beta$ values [0.2, 0.4, 0.6] and report the best-performing system in terms of BLEURT. We observe that, under the same $\beta$ value, including more TPs reduces the KL divergence with respect to the reference model at the end of training while maintaining the same performance. This provides further empirical evidence for our analysis in Sec.\ref{sec:Explaining the regularization effect of ties via ideal DPO policy theory}, where ties regularize training.

\begin{table}[!ht]
    \centering
    \begin{tabular}{lcccc}
        \hline
        \textbf{TP\%} & \textbf{KL} & \textbf{KIWI-22} & \textbf{BLEURT} & \textbf{COMET} \\ \hline
        \multicolumn{5}{c}{DPO-RK (CP+TP)} \\
        \hline
        25\% & 10.16 & 78.4 & 66.5 & 80.7 \\
        50\% & 6.05 & 78.2 & 66.2 & 80.4 \\
        75\% & 4.40 & 78.5 & 66.4 & 80.7 \\
        \hline
        \multicolumn{5}{c}{DPO-D (CP+TP)} \\
        \hline
        25\% & 17.26 & 78.5 & 66.4 & 80.8 \\
        50\% & 10.40 & 78.5 & 66.3 & 80.7 \\
        75\% & 6.63 & 78.5 & 66.5 & 80.8 \\
        \hline
    \end{tabular}
    \caption{Regularization effects of varying TPs in DPO-RK and DPO-D systems. KIWI-22~\cite{KIWI-22_paper} is a reference-less metric. A TP\% of 50\% means that we keep 50\% of all the TPs, resulting in a CPs:TPs ratio of 2:1.}
    \label{tab:tp-proportion-regularization}
\end{table}

\subsection{IPO Baselines and KIWI-22 Evaluation on WMT}
\label{app:IPO Baselines and KIWI-22 Evaluation on WMT}

We additionally evaluate our systems with KIWI-22~\citep{KIWI-22_paper}, a popular reference-less neural metric used by recent works in evaluating state-of-the-art (SoTA) machine translation systems~\citep{ALMA-R_paper}.

\begin{table*}[!ht]
    \centering
    \begin{tabular}{lccc}
        \hline
        \textbf{Model} & \textbf{Best BLEURT} & \textbf{Best COMET} & \textbf{Best KIWI-22} \\
        \hline
        Bloomz-mt-7b1-SFT              & 61.6 & 77.9 & 77.3 \\
        Bloomz-mt-7b1-DPO (CP)         & \textbf{67.1} & \textbf{81.3} & \textbf{82.0} \\
        Bloomz-mt-7b1-DPO (CP+TP)      & 66.6 & 80.6 & 78.5 \\
        Bloomz-mt-7b1-DPO-RK (CP+TP)   & \textbf{67.1} & 81.0 & 79.0 \\
        Bloomz-mt-7b1-DPO-D (CP+TP)    & \textbf{67.1} & 81.0 & 79.2 \\
        Bloomz-mt-7b1-IPO (CP)         & 66.6 & 80.6 & 78.6 \\
        Bloomz-mt-7b1-IPO (CP+TP)      & 66.3 & 80.5 & 78.4 \\
        \hline
    \end{tabular}
    \caption{Evaluation results under BLEURT, COMET, and KIWI-22 metrics for various models. DPO and its variants show clear gains over the SFT baseline across all three metrics.}
    \label{tab:ipo-kiwi-evaluation}
\end{table*}

We perform additional NMT experiments with IPO fine-tuning on CPs (IPO(CP+TP) and CPs+TPs (IPO(CP+TP)). We conduct IPO fine-tuning (using code from https://github.com/eric-mitchell/direct-preference-optimization) with the a range of beta values [0.025, 0.05, 0.1, 0.3, 0.5, 0.7, 1.0] and report the best performing system as evaluated on the WMT-21 ZH-EN dataset in Table \ref{tab:ipo-kiwi-evaluation}. We use identical training hyper-parameters as used in DPO, DPO-RK and DPO-D fine-tuning. We note that these metrics are strongly rated in the WMT-23 metric overview paper~\citep{WMT23_Metric_paper}.

We find that (1) both IPO(CP) and IPO(CP+TP) yield substantial gains compared to the baseline SFT model (+4.7 BLEURT. (2) Adding TPs to IPO causes a small degradation (<-0.3) across all metrics. (3)
IPO(CP) performs on par with DPO(CP+TP) but under-performs DPO(CP), DPO-RK(CP+TP), and DPO-D(CP+TP) on all metrics.

\subsection{PairRM as a Proxy Evaluator for GPT-4}
\label{app:PairRM as a proxy of GPT-4}
\begin{table}[!ht]
%\begin{wraptable}{r}{0.45\textwidth}
% \vspace{-1.4cm}
\centering
\begin{tabular}{@{}llll@{}}
\toprule
System & GPT-4 & PairRM &  \\ \midrule
DPO &  &  &  \\
~~T=1.0 & 23.4\% & 27.3\% &  \\
~~T=0.75 & 40.2\% & 40.6\% &  \\
~~T=0.5 & 52.3\% & 54.7\% &  \\
~~T=0.25 & 46.9\% & 51.6\% &  \\
~~T=0.0 & 50.4\% & 55.5\% &  \\
 &  &  &  \\
SFT &  &  &  \\
~~T=1.0 & 22.3\% & 23.0\% &  \\
~~T=0.5 & 37.5\% & 38.7\% &  \\
~~T=0.0 & 36.7\% & 39.8\% &  \\ \bottomrule
\end{tabular}
\caption{Win-rate of Pythia-2.8B model SFT/DPO on TL;DR train against 256 human-written summaries as judged by GPT4-0613 and PairRM.}
\label{tab:PairRM win-rate}
% \end{wraptable}
\end{table}
PairRM~\citep{PairRM-paper} is a strong reward model that has been shown to be effective in curating preference datasets for iterative DPO training~\citep{snorkel-iterative-dpo-with-pairrm}. In our experiments on TL;DR summarization, we use the PairRM reward model instead of GPT-4 for comparing generated summaries against human references. In this appendix, we show that win-rate as judged by PairRM is a good proxy for GPT4-0613~\citep{openai2024gpt4technicalreport} win-rate on the TL;DR dataset~\cite{TLDR-dataset-OpenAI}.

We generate summaries from SFT pythia-2.8B model by sampling at temperature $T=[0.0, 0.5, 1.0]$ and the DPO model ($\beta=0.1$) trained on TL;DR's full training set at temperature $T=[0.0, 0.25, 0.5, 0.75, 1.0]$. Their win-rates against the 256 human-written summaries in the TL;DR valid-2 split as judged by GPT-4 and PairRM are tabulated in Table \ref{tab:PairRM win-rate}. We find that the win-rates by GPT-4 and PairRM are similar and that system rankings are generally preserved. We opt to use PairRM as our evaluation metric which enables us to conduct experiments faster and at lower costs.

\subsection{Verifying a Tied Pair Selection Strategy for TL;DR}
\label{app:Verifying tied pair selection strategy}

% Tabulate DPO(CP), DPO(CP+TP), DPO(CP+rTP), DPO-RK(CP+TP), DPO-D(CP+TP) at $\beta=0.3$. Show that DPO-D and DPO-RK perform better -> shows that TPs are better modeled as ties. 
% \item We are confident that BLEURT, being a well-established specialized metric for WMT, is capable of assigning similar scores to similar-quality translations. 

% \begin{itemize}
% \item 
As explained in Sec.~\ref{sec:adding ties}, we use the reward model associated with the DPO model trained on TL;DR to identify summarizations that are similar in quality.  Note that we are performing unsupervised labelling of ties in the DPO training data, which is somewhat more forgiving than the classification task discussed in other sections which requires labelling ties in held-out data not seen in training.  We do however assume that the reward model should perform well on the data it was trained on.   

To investigate these assumptions, we swap the preferred and the dispreferred responses in all tied pairs to form ``reversed Tied Pairs'' (rTP). If the responses in TP are truly similar in quality (i.e., it is acceptable to reverse the preference direction), training with DPO(CP+TP) and DPO(CP+rTP) should yield similar performing models. Furthermore, the DPO-RK and DPO-D learning procedures which explicitly model tied pairs should yield better performing model.  We conduct  experiments on TL;DR. Table \ref{tab:reversed ties results} shows that the performance relation $\text{DPO-D(CP+TP)} \sim \text{DPO-RK(CP+TP)} \succ \text{DPO(CP+TP)} \sim \text{DPO(CP+rTP)}$ indeed holds for TL;DR, which suggests that our Tied Pair selection strategy is reasonable.

\begin{table}[!htb]
\centering
    % \begin{minipage}{.5\linewidth}
    %   \centering
    %     \begin{tabular}{@{}lr@{}} \toprule
    %     \multicolumn{2}{c}{WMT21 ZH-EN}  \\ \midrule & BLEURT \\
    %     DPO(CP+ TP) & 66.4 \\
    %     DPO(CP+rTP) & 66.3 \\ 
    %     % DPO(CP) &  66.7\\
    %     DPO-RK(CP+TP) &  66.7 \\
    %     DPO-D(CP+TP) & 67.1 \\
    %     \bottomrule
    %     \end{tabular}
    % \end{minipage}%
    % \begin{minipage}{.5\linewidth}
    %   \centering
       \begin{tabular}{@{}lc@{}} \toprule
        % \multicolumn{2}{c}{TL;DR}  \\ \midrule
        System & PairRM \\
        DPO(CP+ TP)  &  58.6\% \\
        DPO(CP+rTP)  & 60.9\% \\
        % DPO(CP), & 63.7 \\
        DPO-RK(CP+TP)  & 68.0\% \\
        DPO-D(CP+TP) & 68.8\% \\
        \bottomrule
        \end{tabular}
    % \end{minipage} 
    \caption{Win-rates of Pythia-2.8B model DPO on TL;DR train against 256 human-written summaries as judged by  PairRM.
    Systems were trained on CP+TP or CP+rTP data with DPO, DPO-RK, or DPO-D at fixed $\beta=0.3$. For DPO-RK and DPO-D learning, rTP is equivalent to TP as there is no preference direction for ties.}
    \label{tab:reversed ties results}
\end{table}

\subsection{Using GPT-4 to validate ties on NMT}
\label{app:Using GPT-4 to validate ties on NMT}

We conduct the following experiments with GPT-4 to evaluate whether the tied pairs selected by BLEURT are reliable.

We randomly select 128 source sentences on our training set sampled from WMT20 ZH-EN and collect their corresponding CPs and TPs as determined by BLERUT score's ranking. Then, for each source sentence, we present the two pairs of translations to GPT-4o-2024-08-06 and prompt it to decide which pair of translations is a tied pair.

We find that on the 128 randomly selected examples from WMT20 ZH-EN, the GPT-4o model agrees 82.03\% of the time with BLEURT's assignment of CP and TP.

An example of our prompt and GPT-4's response is below. In this case, GPT-4 chooses the same pair (Pair A) as TP as BLEURT. 

%We also note that the WMT23 metric reports high correlation (0.776, Table 1 in [3]) between BLEURT score and human judgment on translation.

The prompt for GPT-4 and example responses are provided below.

\begin{lstlisting}[caption={Prompt and Response for using GPT-4 to validate tie selection on NMT}, label={lst:gpt4-prompt}]
Pair A:
    Better Translation: Beijing Youth Daily reporters learned that there were many important amendments in the first published " drafts".
    Worse Translation: Beijing Youth Daily reporters learned that the "First Draft" has many important modifications.
Pair B:
    Better Translation: The Beijing Youth Daily reporter learned that there were many important changes in the Draft for the first time published.
    Worse Translation:  Beijing Youth Daily reporter learned that the first version of the "Comments" had been revised many times.

You should provide a short rationale and then decide which pair is a tied pair. You should follow the format "DECISION: [[A]]" or "DECISION: [[B]]" to give your decision.

Response:

To determine which pair represents a tied pair, let's evaluate the differences between the translations in each pair.

In Pair A:
- The "better" translation is: "Beijing Youth Daily reporters learned that there were many important amendments in the first published 'drafts'."
- The "worse" translation is: "Beijing Youth Daily reporters learned that the 'First Draft' has many important modifications."

Both translations in Pair A convey similar meanings, describing modifications or amendments to a draft. The primary differences are stylistic, such as the placement and quotation of "draft".

In Pair B:
- The "better" translation is: "The Beijing Youth Daily reporter learned that there were many important changes in the Draft for the first time published."
- The "worse" translation is: "Beijing Youth Daily reporter learned that the first version of the 'Comments' had been revised many times."

In Pair B, the two translations have significant differences in meaning. The "better" translation accurately captures the notion of "important changes" and "first published," whereas the "worse" translation introduces "first version of the 'Comments'" and a focus on revisions rather than initial publication, which diverges more from the source sentence.  

With this analysis, Pair A has translations that are more similar in meaning and are thus likely to be the tied pair, while Pair B shows a clearer distinction in translation quality.

DECISION: [[A]]
\end{lstlisting}

\newpage

\clearpage

\newpage
\section{Simulating the Perfect DPO-Davidson Policy}
In Section~\ref{ss:tpvkl} we make use of the relationship derived by \citet[Appendix A.2]{PL-paper-chen} which specifies the optimal DPO policy to  
% \begin{equation}
% \frac{\pi^*(y_1|x)}{\pi^*(y_2|x)}=\frac{\pi_{ref}(y_1|x)}{\pi_{ref}(y_2|x)}\Big(\frac{\mathbb{P}(y_1 \succ_x y_2)}{1-\mathbb{P}(y_1 \succ_x y_2)}\Big)^{1/\beta}
% \label{eq:dpo_opt_d}    
% \end{equation}
minimize the binary classification loss  
\begin{align*}
    \min_\pi&  \mathbb{P}(y_1 \succ_x y_2) \log \pi (y_1 \succ_x y_2)] \\ &+ (1-\mathbb{P}(y_1 \succ_x y_2)) \log (1- \pi (y_1 \succ_x y_2))
\end{align*}

where $\mathbb{P}(y_1 \succ_x y_2)$ is the human ground truth preference distribution.  

We extend the analysis of \citet{PL-paper-chen} to include the Davidson model, noting that the binary maximum likelihood objective becomes ternary.  We assume we have the ground-truth human preference distributions $\mathbb{P}(y_1 \succ_x y_2)$, $\mathbb{P}(y_2 \succ_x y_1)$, and $\mathbb{P}(y_1 \sim_x y_2)$ needed to define the objective. The resulting     
Theorem \ref{theorem:simulating perfect DPO-D} can be viewed as a generalization of Theorem 3 of \citet{PL-paper-chen} that allows for the observations of ties.  Where ties are not allowed (i.e. $\nu_{D}=0$), the Davidson model simplifies to the Bradley-Terry model and Theorem 3 of \citet{PL-paper-chen} 
is recovered as a special case of Theorem \ref{theorem:simulating perfect DPO-D}.

% \subsection{Simulating Perfect DPO-D Policy}
\label{app:Simulating Ideal DPO-Davidon Model}
%In this appendix, we show how to simulate the policy that optimally optimizes the DPO-D objective assuming access to perfect comparison data (i.e., true proportions of human preference and tie judgments). The derivation is similar to Theorem 3.1 of \citet{PL-paper-chen} which simulates the perfect DPO/RLHF model.

\begin{theorem}[Simulating Perfect DPO-D Policy]
Assume we are given an aggregated comparison datapoint $(x, y_1, y_2)$ and   human ground-truth preference probabilities  $\mathbb{P}(y_1 \succ_x y_2)$,  $\mathbb{P}(y_1 \succ_x y_2)$, and $\mathbb{P}(y_1 \sim_x y_2)$ which  obey the Davidson model with hyper-parameter $\nu_{D}$. Let the reference model be $\pi_{ref}$. It follows that the perfect DPO-Davidson policy $\pi^*$ on this aggregated comparison datapoint satisfies 
\begin{equation}
    \frac{\pi^*(y_1|x)}{\pi^*(y_2|x)}=\frac{\pi_{ref}(y_1|x)}{\pi_{ref}(y_2|x)}\Big(\frac{\mathbb{P}(y_1 \succ_x y_2)}{\mathbb{P}(y_2 \succ_x y_1)}\Big)^{1/\beta}
    \label{eq:perfect DPO-D policy simulation}
\end{equation}
or equivalently
\begin{equation}
    \frac{\pi^*(y_1|x)}{\pi^*(y_2|x)}=\frac{\pi_{ref}(y_1|x)}{\pi_{ref}(y_2|x)}\Big(2\nu_D\frac{\mathbb{P}(y_1 \succ_x y_2)}{\mathbb{P}(y_1 \sim_x y_2)}\Big)^{2/\beta}
    \label{eq:perfect DPO-D policy simulation with tie probabilities}
\end{equation}

\label{theorem:simulating perfect DPO-D}
\end{theorem}

\begin{proof}
The DPO-D policy objective optimizes the following three-way classification loss:
\begin{align*}
\min_\pi & \mathbb{P}(y_1 \succ_x y_2)\log \pi(y_1 \succ_x y_2) \\ &+ \mathbb{P}(y_2 \succ_x y_1)\log\pi(y_2 \succ_x y_1) \\ &+ \mathbb{P}(y_1 \sim_x y_2)\log\pi(y_1 \sim_x y_2)
\end{align*}
Let $\theta^*$ denotes a set of parameters such that $\pi_{\theta^*}$ is an optimal policy for the above loss, then $\pi_{\theta^*}$ satisfies:
\begin{align*}
    \pi_{\theta^*}(y_1 \succ_x y_2) &= \mathbb{P}(y_1 \succ_x y_2) \\
    \pi_{\theta^*}(y_2 \succ_x y_1) &= \mathbb{P}(y_2 \succ_x y_1) \\
    \pi_{\theta^*}(y_1 \sim_x y_2) &= \mathbb{P}(y_1 \sim_x y_2)
\end{align*}

Expressing the policy probability $\pi_{\theta^*}(y_w \succ_x y_l)$ and $\pi_{\theta^*}(y_l \succ_x y_w)$ in terms of the reward margins $d_{\theta^*}(x,y_w,y_l)$:
\begin{align*}
\mathbb{P}(y_1 \succ_x y_2) &= \frac{ 1 }{ 1 + e^{-d_{\theta^*}(x,y_w,y_l)} + 2\nu_D e^{-d_{\theta^*}(x,y_w,y_l)/2}} \\
    \mathbb{P}(y_2 \succ_x y_1) &= \frac{ e^{-d_{\theta^*}(x,y_1,y_2)} }{ 1 + e^{-d_{\theta^*}(x,y_1,y_2)} + 2\nu_D e^{-d_{\theta^*}(x,y_1,y_2)/2}} 
\end{align*}

Rearranging, we have
\begin{align*}
    \frac{\mathbb{P}(y_2 \succ_x y_1)}{\mathbb{P}(y_1 \succ_x y_2)} &= \exp\big(-d_{\theta^*}(x,y_1,y_2) \big) \\ &= \exp\Big(\beta \log \frac{\pi_{\theta^*}(y_2|x)}{\pi_{ref}(y_2|x)} - \beta \log \frac{\pi_{\theta^*}(y_1|x)}{\pi_{ref}(y_1|x)}\Big)
\end{align*}
Taking logarithms on both side and divide by $\beta$.
\begin{equation*}
    \frac{1}{\beta}\log\frac{\mathbb{P}(y_1 \succ_x y_2)}{\mathbb{P}(y_2 \succ_x y_1)} = \log\frac{\pi_{\theta^*}(y_2|x)\pi_{ref}(y_1|x)}{\pi_{ref}(y_2|x)\pi_{\theta^*}(y_1|x)}
\end{equation*}
Exponentiating both sides gives

\begin{equation*}
    \frac{\pi_{\theta^*}(y_2|x)}{\pi_{\theta^*}(y_1|x)} = \frac{\pi_{ref}(y_2|x)}{\pi_{ref}(y_1|x)}\Big(\frac{\mathbb{P}(y_2 \succ_x y_1)}{\mathbb{P}(y_1 \succ_x y_2)}\Big)^{1/\beta}
\end{equation*}

Taking the inverse yields Eq \ref{eq:perfect DPO-D policy simulation}. 

To see the equivalence between Eq \ref{eq:perfect DPO-D policy simulation} and Eq \ref{eq:perfect DPO-D policy simulation with tie probabilities}, note that the ground-truth preference and tie probabilities which obey the Davidson model satisfy the following relation:

\begin{align*}
    \mathbb{P}(y_1\sim_x y_2) = 2\nu_D \sqrt{\mathbb{P}(y_1\succ_x y_2)\mathbb{P}(y_2 \succ_x y_1)}
\end{align*}

% \begin{align*}
%     p_\theta^D(y_1 \sim_x y_2) &= 2\nu_D e^{-d_\theta(x, y_1, y_2)/2} p_\theta^D(y_1 \succ_x y_2) \\ &= 2\nu_D \sqrt{\frac{p_\theta^D(y_2\succ_x y_1)}{p_\theta^D(y_1\succ_x y_2)}}p_\theta^D(y_1 \succ_x y_2) \\
%     &= 2\nu_D \sqrt{(p_\theta^D(y_1\succ_x y_2)p_\theta^D(y_2\succ_x y_1))}
% \end{align*}

Rearranging Eq \ref{eq:perfect DPO-D policy simulation}:
    
\begin{align*}
    \frac{\pi^*(y_1|x)}{\pi^*(y_2|x)}&=\frac{\pi_{ref}(y_1|x)}{\pi_{ref}(y_2|x)}\Big(\frac{\mathbb{P}(y_1 \succ_x y_2)}{\mathbb{P}(y_2 \succ_x y_1)}\Big)^{1/\beta} \\ &= \frac{\pi_{ref}(y_1|x)}{\pi_{ref}(y_2|x)}\Big(\sqrt{\frac{\mathbb{P}(y_1 \succ_x y_2)}{\mathbb{P}(y_2 \succ_x y_1)}}\Big)^{2/\beta} \\
    &= \frac{\pi_{ref}(y_1|x)}{\pi_{ref}(y_2|x)}\Big(\frac{\mathbb{P}(y_1 \succ_x y_2)}{\sqrt{\mathbb{P}(y_1 \succ_x y_2)\mathbb{P}(y_2 \succ_x y_1)}}\Big)^{2/\beta} \\
    &= \frac{\pi_{ref}(y_1|x)}{\pi_{ref}(y_2|x)}\Big(2\nu_D\frac{\mathbb{P}(y_1 \succ_x y_2)}{\mathbb{P}(y_1 \sim_x y_2)}\Big)^{2/\beta}
    \label{eq:perfect DPO-D policy simulation}
\end{align*}

which is Eq \ref{eq:perfect DPO-D policy simulation with tie probabilities}. 
\end{proof}

\section{Extended Related Work}
\label{sec:extended related work}
\paragraph{Variants of Direct Preference Optimization} A range of variants of Direct Preference Optimization have been proposed based on problem-specific or theoretical motivations. \citet{DisentanglingLengthDPO} tackle excessive response length by introducing explicit length normalization in the DPO objective. SimPO~\citep{SimPO} modifies the DPO objective to remove the need for a reference model and to include length normalization. KTO~\citep{KTOpaper} is motivated by Kahneman and Tversky’s prospect theory and learns from non-paired preference data. ODPO~\citep{ODPO-Direct-Preference-Optimization-with-an-offset-Amini} incorporates preference strength in the objective by introducing an offset parameter. In deriving ODPO, the offset parameter of \citet[Theorem 3]{ODPO-DirectPreferenceOptimizationWithAnOffset})  plays a role similar to 
 the sensitivity threshold of \citet{rao1967ties}. We note that the ODPO objective with a fixed offset agrees with our proposed DPO-RK objective restricted to clear preference data, but does not extend to ties. 

\paragraph{Frameworks for Pair-wise Preference Optimization} Several works propose theoretical frameworks for understanding general Preference Optimization from which DPO can be obtained as a special case. \citet{IPO-AGeneralTheoreticalParadigmtoUnderstandLearningfromHumanPreferences} introduces the $\Psi$PO formalism which allows alternative expression of the reward in terms of the model's predicted probability. IPO is derived when the identity mapping is used, and DPO arises under a log-ratio mapping. \citet{DensityEstimationPerspectiveForPreferenceOptimization} formulate learning from pair-wise preference as learning the implicit preference generating distribution of the annotators. In this formalism, DPO is a well-specified model for the implicit preference distribution assuming that the human preference generative process follows the Bradley-Terry model. Our work can be viewed as assuming an annotator preference generating distribution that allows for the outcome of a tie (i.e. the Rao-Kupper or the Davidson model).
\citet{GPO-GeneralizedPreferenceOptimization-AUnifiedApproachtoOfflineAlignment} propose a generalized approach to deriving offline preference optimization losses through binary classification. In this work, we consider the ternary classification with the possibility of declaring a tie. In Appendix \ref{app:Simulating Ideal DPO-Davidon Model}, we show that the `perfect' DPO-D policy can be simulated starting from the ternary classification loss.

% In the sense of \citet{DensityEstimationPerspectiveForPreferenceOptimization}, our work can be understood as assuming an annotator preference generating distribution that allows for the outcome of a tie (i.e., the Rao-Kupper or the Davidson distribution).

% In this work, we consider a ternary classification with the possibility of declaring a tie. In Appendix \ref{app:Simulating Ideal DPO-Davidon Model}, we show that the perfect DPO-D policy can be simulated starting from the ternary classification loss.
\paragraph{Pair-wise Comparison Models}  \citet{TheManyRoutesToTheUbiquitousbradleyterrymodel} review the  Bradley-Terry model, including its relation to the logistic distribution~\citep{BradleyTerryandLogisiticDistribution} and  Luce choice axiom~\cite{LuceChoice}. 
The Rao-Kupper~\citep{rao1967ties} and the Davidson model~\citep{David-TheMethodOfPairedComparisons} are two notable extensions to Bradley-Terry (Sec.~\ref{ssec:bte}). 
%The Rao-Kupper model is derived with the use of a threshold parameter. The Davidson model is derived by relating the no preference probability with the preference probabilities such that the Luce’s choice axiom is respected. 
%Pair-wise comparison models were heavily studied by statisticians between 1950s-80s.  
See the review by \citet{David-TheMethodOfPairedComparisons}  and  bibliography by \citet{A-Bibliography-on-the-Method-of-Paired-Comparisons}. Modeling ties remains an active research topic in fields such as sport team ranking~\citep{ImprovePageRankSport},  medical treatments~\citep{TieMedicalRecord}, and chatbots~\citep{RankingChatbots}.

% \begin{itemize}
%     \item \textbf{DPO and Variants}
%     \begin{itemize}
%         \item Loss variants: IPO. Non-reference model variants: SimPO, etc.
%         \item Incorporate Preference Strength in the Training Objective (PRO, ODPO)
%         \item Length Normalization (Disentangling Length, SimPO)
%     \end{itemize}
%     \item \textbf{Bradley-Terry Extensions}
%     \begin{itemize}
%         \item To incorporate ties
%         \item To incorporate order effect
%         \item To incorporate ranking (general Plackett-Luce model)
%     \end{itemize}
% \end{itemize}
\newpage 
 
\section{Qualitative Cases from DPO, DPO-RK, and DPO-D Systems}
\label{app:Qualitative Cases from DPO, DPO-RK, and DPO-D Systems}

We collect the responses on WMT-21 ZH-EN from the best performing model under each configuration. Below, we present two cases where systems trained with DPO-RK(CP+TP) and DPO-D(CP+TP) perform better than DPO(CP) and one case where they perform worse than DPO(CP) as judged by KIWI-22. 

The KIWI-22 score of each translation is shown in parentheses.

\subsection*{Case 1: ID = 10}

\noindent\textbf{Source:} \begin{CJK*}{UTF8}{gbsn}
小康不小康，关键看老乡。
\end{CJK*}

\noindent\textbf{Reference:} Hometowns are the key to moderate prosperity.

\begin{itemize}
    \item \textbf{DPO(CP):} It is not important whether you are rich or poor, as long as you have a good relationship with your neighbors. (71.6)
    \item \textbf{DPO(CP+TP):} Wealthy or poor, it all depends on our neighbors. (72.0)
    \item \textbf{DPO-RK(CP+TP):} Wealthy or poor, it all depends on our neighbors. (72.0)
    \item \textbf{DPO-D(CP+TP):} Whether you are rich or poor depends on your neighbors. (77.9)
\end{itemize}

In this case, the DPO(CP) system produces a long translation with incorrect meaning. The other three systems produce correct, literal translations. The translation from DPO-D(CP+TP) is more fluent. However, all systems fail to fully capture the meaning of \begin{CJK*}{UTF8}{gbsn}
"小康"
\end{CJK*} (i.e., "moderate prosperity") as done in the reference.

\subsection*{Case 2: ID = 269}

\noindent\textbf{Source:} \begin{CJK*}{UTF8}{gbsn} 这一航空运营商在声明中说，旗下一架客机7月15日从华盛顿州西雅图飞往得克萨斯州奥斯汀途中，一台发动机突然熄火，客机最终在奥斯汀机场成功紧急降落。\end{CJK*}

\noindent\textbf{Reference:} The airline released a statement claiming that as its aircraft was flying from Seattle, Washington to Austin, Texas, one of the engines suddenly stalled, and the aircraft finally made an emergency landing in Austin.

\begin{itemize}
    \item \textbf{DPO(CP):} This airline operator said in a statement that one engine on one of its aircraft failed midflight on July 15 from Seattle, Washington to Austin, Texas, and the aircraft successfully made an emergency landing at the Austin airport. (84.5)
    \item \textbf{DPO(CP+TP):} This airline operator said in a statement that one of its aircraft was on its way from Seattle, Washington to Austin, Texas when one of its engines suddenly went out of order, and the aircraft successfully made an emergency landing at the airport in Austin. (81.4)
    \item \textbf{DPO-RK(CP+TP):} This airline operator said in a statement that one engine on one of its aircraft failed during the flight from Seattle, Washington to Austin, Texas on July 15, and the aircraft successfully made an emergency landing at the Austin airport. (85.4)
    \item \textbf{DPO-D(CP+TP):} This airline operator said in a statement that one engine on one of its aircraft failed during the flight from Seattle, Washington to Austin, Texas on July 15, and the aircraft successfully made an emergency landing at the Austin airport. (85.4)
\end{itemize}

The translation produced by the DPO(CP+TP) system fails to capture the date of the event ("July 15"). Translations from DPO-RK(CP+TP) and DPO-D(CP+TP) are identical and differ from that from DPO(CP) only in rephrasing "midflight" with "during the flight."

\subsection*{Case 3: ID = 91}

\noindent\textbf{Source:} 
\begin{CJK*}{UTF8}{gbsn} 运动少年热血竞技秀《运动吧少年》在7月11日正式登陆湖南卫视。\end{CJK*}

\noindent\textbf{Reference:} The inspiring competitive show for youth sports, We Are The One, was officially broadcast on Hunan TV on July 11.

\begin{itemize}
    \item \textbf{DPO(CP):} The youth sports action show “Be Active, Be Sporty” will officially debut on Hunan Satellite TV on July 11. (76.8)
    \item \textbf{DPO(CP+TP):} The youth sports competition show "Go, Boy!" officially landed on the Changsha Satellite TV Station on July 11. (75.6)
    \item \textbf{DPO-RK(CP+TP):} The youth sports competition show "Go Play Boys" will be officially broadcast on the Hunan Satellite TV on July 11. (72.1)
    \item \textbf{DPO-D(CP+TP):} The youth sports competition show "Let's Go Boys" will officially debut on the Hunan Satellite TV on July 11. (74.2)
\end{itemize}

In this case, all systems produced a more literal translation of the name of the TV show compared to the reference. Although the KIWI-22 scores for DPO-RK(CP+TP) and DPO-D(CP+TP) are lower, the translation qualities are in fact similar.

\section{Broader Impacts}
\label{app:Broader Impacts}
This paper presents work whose goal is to advance the field of Machine Learning. There are many
potential societal consequences of our work, none of which we feel must be specifically highlighted
here.

\section*{NeurIPS Paper Checklist}

\begin{enumerate}

\item {\bf Claims}
    \item[] Question: Do the main claims made in the abstract and introduction accurately reflect the paper's contributions and scope?
    \item[] Answer: \answerYes{} % Replace by \answerYes{}, \answerNo{}, or \answerNA{}.
    \item[] Justification: See abstract, introduction, and conclusion.

\item {\bf Limitations}
    \item[] Question: Does the paper discuss the limitations of the work performed by the authors?
    \item[] Answer: \answerYes{} % Replace by \answerYes{}, \answerNo{}, or \answerNA{}.
    \item[] Justification: See Appendix \ref{app:limitation}.

\item {\bf Theory assumptions and proofs}
    \item[] Question: For each theoretical result, does the paper provide the full set of assumptions and a complete (and correct) proof?
    \item[] Answer: \answerYes{} % Replace by \answerYes{}, \answerNo{}, or \answerNA{}.
    \item[] Justification: Complete derivation for DPO-RK and DPo-D are provided in Appendix \ref{app:Mathematical Derivations}. Complete proof of ideal DPO-D policy theory is provided in Appendix \ref{app:Simulating Ideal DPO-Davidon Model}.

    \item {\bf Experimental result reproducibility}
    \item[] Question: Does the paper fully disclose all the information needed to reproduce the main experimental results of the paper to the extent that it affects the main claims and/or conclusions of the paper (regardless of whether the code and data are provided or not)?
    \item[] Answer: \answerYes{} % Replace by \answerYes{}, \answerNo{}, or \answerNA{}.
    \item[] Justification: Hyper-parameters of our experiments are provided in \ref{app:experimental details and full results}, where model versions, datasets, training hyper-parameters and computing platforms are specified. We will publish code on publication.

\item {\bf Open access to data and code}
    \item[] Question: Does the paper provide open access to the data and code, with sufficient instructions to faithfully reproduce the main experimental results, as described in supplemental material?
    \item[] Answer: \answerYes{} % Replace by \answerYes{}, \answerNo{}, or \answerNA{}.
    \item[] Justification: Code and scripts will be provided upon publication.

\item {\bf Experimental setting/details}
    \item[] Question: Does the paper specify all the training and test details (e.g., data splits, hyperparameters, how they were chosen, type of optimizer, etc.) necessary to understand the results?
    \item[] Answer: \answerYes{} % Replace by \answerYes{}, \answerNo{}, or \answerNA{}.
    \item[] Justification: See Appendix \ref{app:experimental details and full results}.

\item {\bf Experiment statistical significance}
    \item[] Question: Does the paper report error bars suitably and correctly defined or other appropriate information about the statistical significance of the experiments?
    \item[] Answer: \answerNo{} % Replace by \answerYes{}, \answerNo{}, or \answerNA{}.
    \item[] Justification: Our main results reports a grid search in hyper-parameter and we report all available results. We note a large increase in computational cost is required to conduct multiple training runs at each hyper-parameter setting. We note that we have reported results on 5 different experimental setups (Appendix \ref{app:experimental details and full results} and Sec.\ref{sec:Performance Improvement through DPO-RK and DPO-D}) with at least 15 training runs per setup (3 DPO variants times minimal 5 beta values), and our findings are consistent across these experiments. 

\item {\bf Experiments compute resources}
    \item[] Question: For each experiment, does the paper provide sufficient information on the computer resources (type of compute workers, memory, time of execution) needed to reproduce the experiments?
    \item[] Answer: \answerYes{} % Replace by \answerYes{}, \answerNo{}, or \answerNA{}.
    \item[] Justification: See Appendix \ref{app:experimental details and full results}.
    
\item {\bf Code of ethics}
    \item[] Question: Does the research conducted in the paper conform, in every respect, with the NeurIPS Code of Ethics \url{https://neurips.cc/public/EthicsGuidelines}?
    \item[] Answer: \answerYes{} % Replace by \answerYes{}, \answerNo{}, or \answerNA{}.
    \item[] Justification: Our research has been conducted with strict adherence to the NeurIPS Code of
Ethics.

\item {\bf Broader impacts}
    \item[] Question: Does the paper discuss both potential positive societal impacts and negative societal impacts of the work performed?
    \item[] Answer: \answerYes{} % Replace by \answerYes{}, \answerNo{}, or \answerNA{}.
    \item[] Justification: See Appendix \ref{app:Broader Impacts}
    
\item {\bf Safeguards}
    \item[] Question: Does the paper describe safeguards that have been put in place for responsible release of data or models that have a high risk for misuse (e.g., pretrained language models, image generators, or scraped datasets)?
    \item[] Answer: \answerNA {} % Replace by \answerYes{}, \answerNo{}, or \answerNA{}.
    \item[] Justification: This paper has no such risks.

\item {\bf Licenses for existing assets}
    \item[] Question: Are the creators or original owners of assets (e.g., code, data, models), used in the paper, properly credited and are the license and terms of use explicitly mentioned and properly respected?
    \item[] Answer: \answerYes{} % Replace by \answerYes{}, \answerNo{}, or \answerNA{}.
    \item[] Justification: We have used open-sourced datasets and models that are widely-used in the research community and well-known at the time of writing. They have been cited where appropriate.

\item {\bf New assets}
    \item[] Question: Are new assets introduced in the paper well documented and is the documentation provided alongside the assets?
    \item[] Answer: \answerNA{} % Replace by \answerYes{}, \answerNo{}, or \answerNA{}.
    \item[] Justification: This paper does not release new assets.

\item {\bf Crowdsourcing and research with human subjects}
    \item[] Question: For crowdsourcing experiments and research with human subjects, does the paper include the full text of instructions given to participants and screenshots, if applicable, as well as details about compensation (if any)? 
    \item[] Answer: \answerNA{} % Replace by \answerYes{}, \answerNo{}, or \answerNA{}.
    \item[] Justification: This paper does not involve crowdsourcing nor research with human subjects.

\item {\bf Institutional review board (IRB) approvals or equivalent for research with human subjects}
    \item[] Question: Does the paper describe potential risks incurred by study participants, whether such risks were disclosed to the subjects, and whether Institutional Review Board (IRB) approvals (or an equivalent approval/review based on the requirements of your country or institution) were obtained?
    \item[] Answer: \answerNA{} % Replace by \answerYes{}, \answerNo{}, or \answerNA{}.
    \item[] Justification: this paper does not involve crowdsourcing nor research with human subjects.

\item {\bf Declaration of LLM usage}
    \item[] Question: Does the paper describe the usage of LLMs if it is an important, original, or non-standard component of the core methods in this research? Note that if the LLM is used only for writing, editing, or formatting purposes and does not impact the core methodology, scientific rigorousness, or originality of the research, declaration is not required.
    %this research? 
    \item[] Answer: \answerNA{} % Replace by \answerYes{}, \answerNo{}, or \answerNA{}.
    \item[] Justification: We do not use LLMs in developing the core method in this research.

\end{enumerate}

\end{document}